\documentclass[journal,onecolumn]{IEEEtran}

\usepackage[utf8]{inputenc}
\usepackage{amsmath, amssymb, amsfonts, dsfont, amsthm}
\newtheorem{theorem}{Theorem}[section]
\newtheorem{remark}[theorem]{Remark}
\newtheorem{definition}[theorem]{Definition}

\newtheorem{lemma}[theorem]{Lemma}
\newtheorem{corollary}[theorem]{Corollary}

\usepackage{kantlipsum}
\usepackage{tikz}
\usetikzlibrary{arrows}
\usepackage{wrapfig}
\usepackage{paralist}
\usepackage{verbatim}
\usepackage{fullpage}
\usepackage{relsize}
\usepackage{hyperref}
\usepackage{graphicx}
\usepackage{epstopdf}
\usepackage{stmaryrd}
\usepackage{mathtools}
\usepackage[]{algorithm2e}
\allowdisplaybreaks[1]
\usepackage{sistyle}
\newcommand\bSI[1]{{\small[\SI{}{#1}]}}

\usepackage{pgfplots}
\usepgfplotslibrary{groupplots}
\usetikzlibrary{arrows.meta}
\usepackage{subcaption}

\makeatletter
\newlength\unitwdth
\newlength\numwdth
\newlength\tdima
\newcommand\SIdescr[2]{
    \setlength\tdima{\linewidth}
    \addtolength\tdima{\@totalleftmargin}
    \addtolength\tdima{-\dimen\@curtab}
    \addtolength\tdima{-\unitwdth}
    \addtolength\tdima{-\numwdth}
    \parbox[t]{\tdima}{
        #1
        \leaders\hbox{$\m@th\mkern \@dotsep mu\hbox{\tiny.}\mkern \@dotsep mu$}
        \hfill
        \ifhmode\strut\fi
        \makebox[0pt][l]{
            \makebox[\unitwdth][l]{}
            \makebox[\numwdth][r]{#2}}}}
\makeatother

\newcommand{\Z}{\mathbb{Z}}
\newcommand{\N}{\mathbb{N}}

\newcommand{\R}{\mathbb{R}}

\newcommand{\cNN}{\mathcal{N}}

\newcommand{\supp}{\mathrm{supp}}

\renewcommand{\epsilon}{\varepsilon}
\newcommand{\eps}{\epsilon}

\newcommand{\M}{\mathcal{M}}
\renewcommand{\L}{\mathcal{L}}

\DeclarePairedDelimiter{\paren}{(}{)}
\DeclarePairedDelimiter{\bracket}{\{ }{\} }

\usepackage{amsmath,amsfonts,bm}

\def\eqref#1{equation~\ref{#1}}

\def\1{\bm{1}}

\def\eps{{\epsilon}}

\DeclareMathAlphabet{\mathsfit}{\encodingdefault}{\sfdefault}{m}{sl}
\SetMathAlphabet{\mathsfit}{bold}{\encodingdefault}{\sfdefault}{bx}{n}

\DeclareMathOperator*{\argmax}{arg\,max}

\usepackage{hyperref}
\usepackage{url}

\usepackage{times}

\DeclareRobustCommand{\rprod}{{\mathpalette\irprod\relax}}
\newcommand{\irprod}[2]{\raisebox{\depth}{$#1\prod$}}

\DeclareMathSymbol{\Phi}{\mathalpha}{operators}{8}
\DeclareMathSymbol{\Psi}{\mathalpha}{operators}{9}

\pgfplotsset{compat=1.16}

\begin{document}
\title{High-Dimensional Distribution Generation Through Deep Neural Networks}
\author{Dmytro Perekrestenko, Léandre Eberhard, and Helmut Bölcskei

\thanks{D. Perekrestenko is with Google, Switzerland (email: dperekrestenko@google.com)}

\thanks{L. Eberhard is with Upstart Network Inc., USA (e-mail: leandre.eberhard@upstart.com)}

\thanks{H. B\"olcskei is with Chair for Mathematical Information Science, ETH Zurich, Switzerland (e-mail: hboelcskei@ethz.ch)}

\thanks{This paper builds on and significantly extends previous work by a subset of the authors reported in \cite{icml2020}.}

\thanks{The work of D.~Perekrestenko and L.~Eberhard was carried out while they were at ETH Zurich.}
}
\maketitle
\thispagestyle{plain}

\begin{abstract}
We show that every $d$-dimensional probability distribution of bounded support can be generated through deep ReLU networks out of a $1$-dimensional uniform input distribution. What is more, this is possible without incurring a cost---in terms of approximation error measured in Wasserstein-distance---relative to generating the $d$-dimensional target distribution from $d$ independent random variables. This is enabled by a vast generalization of the space-filling approach discovered in \cite{Bailey2019}. The construction we propose elicits the importance of network depth in driving the Wasserstein distance between the target distribution and its neural network approximation to zero. Finally, we find that, for histogram target distributions, the number of bits needed to encode the corresponding generative network equals the fundamental limit for encoding probability distributions as dictated by quantization theory.
\end{abstract}

\section{Introduction}

Deep neural networks have been employed very successfully as generative 
models for complex natural data such as images \cite{radford2015unsupervised,karras2018stylebased} and natural language \cite{bowman2015generating,xu2018dpgan}.
Specifically, the idea is to train deep networks so that they realize complex high-dimensional probability distributions by transforming samples taken from simple low-dimensional distributions such as uniform or Gaussian \cite{Welling2014:VAE,Goodfellow2014:Gan,Arjovsky17:Wgan}. 

Generative networks with output dimension higher than the input dimension occur, for instance, in language modelling where deep networks are used to predict the next word in a text sequence. Here, 
the input layer size is determined by the dimension of the word embedding (typically $\sim 100$) and the output layer, representing a vector of probabilities for each of the words in the vocabulary, is of the size of the vocabulary (typically $\sim 100k$). Another example where the dimension of the output distribution is mandated to be higher than that of the input distribution is given by explicit \cite{Welling2014:VAE,Tolstikhin2018Wassauto} and implicit \cite{Goodfellow2014:Gan,Arjovsky17:Wgan} density generative networks.

Notwithstanding the practical success of deep generative networks, a profound theoretical understanding of their representational capabilities is still lacking. First results along these lines appear in 
\cite{lee2017ability}, where it was shown that generative networks can  approximate distributions arising from the composition of Barron functions \cite{barron1993}. It remains unclear, however, which distributions can be obtained in such a manner. More recently, it was established \cite{lu2020universal} that for every given target distribution (of finite third moment) and source distribution, both defined on $\R^d$, there exists a ReLU network whose gradient pushes forward the source distribution to an arbitrarily accurate---in terms of Wasserstein distance---approximation of the target distribution.
The aspect of dimensionality increase was addressed in \cite{Bailey2019}, where it is shown 
that a uniform univariate source distribution can be transformed, by a ReLU network, into a uniform target distribution of arbitrary dimension through a space-filling approach. Besides, \cite{Bailey2019} also demonstrates how a univariate Gaussian target distribution can be obtained from a univariate uniform source distribution and vice versa. In a general context, the problem of optimal transport \cite{villani2008optimal} between source and target distributions on spaces of different dimensions was studied in \cite{mccann2019optimal}.

The approximation of distributions through generative networks is inherently related to function approximation and hence to the expressivity of neural networks. A classical result along those lines is the universal approximation theorem \cite{Cybenko1989,HORNIK1989359}, which states that single-hidden-layer neural networks with sigmoidal activation function can approximate continuous functions on compact subsets of $\mathbb{R}^n$ arbitrarily well. More recent developments in this area are concerned with the influence of network depth on attainable approximation quality \cite{telgarsky2016benefits,daubechies2019nonlinear,eldan2015power}. A theory establishing the fundamental limits of deep neural network expressivity is provided in \cite{deep-approx-18,deep-it-2019}.

The aim of the present paper is to develop a universal approximation result for generative neural networks. Specifically, we show
that every target distribution supported on a bounded subset of $\R^d$ can be approximated arbitrarily well in terms of Wasserstein distance by pushing forward a $1$-dimensional uniform source distribution through a ReLU network. The result is constructive in the sense of explicitly identifying the corresponding generative network. Concretely, we proceed in two steps. Given a target distribution, we find the histogram distribution that best approximates it---for a given histogram resolution---in Wasserstein distance. This histogram distribution is then realized by a ReLU network driven by a uniform univariate input distribution. The construction of this ReLU network exploits a space-filling property, vastly generalizing the one discovered in \cite{Bailey2019,icml2020}. The main conceptual insight of the present paper is that generating arbitrary $d$-dimensional target distributions from a $1$-dimensional uniform distribution through a deep ReLU network does not come at a cost---in terms of approximation error measured in Wasserstein distance---relative to generating the target distribution from $d$ independent random variables through, e.g., (for arbitrary $d$) the normalizing flows method \cite{rezende15} and (for $d=2$) the Box-Muller method \cite{box1958}. We emphasize that the generating network has to be deep, in fact its depth has to go to infinity to obtain the same Wasserstein-distance error as a construction from $d$ independent random variables would yield. 
Finally, we find that, for histogram target distributions, the number of bits needed to encode the corresponding generative network equals the fundamental limit for encoding probability distributions as dictated by quantization theory \cite{10.5555/555805}.

\section{Definitions and notation} 

We start by introducing general notation.

$\log$ stands for the logarithm to base $2$. For $n_1,n_2 \in \N$, the set of integers in the range $[n_1,n_2]$ is designated as $[n_1\!:\!n_2]$. For $\mathbf{x} = (x_1,x_2, \dots, x_d) \in \mathbb{R}^d$, we denote the vector obtained by retaining the first $t, \ t\leq d$, entries by $\mathbf{x}_{[1:t]} := (x_1,x_2, \dots, x_t) \in \mathbb{R}^t$.
$U(\Delta)$ stands for the uniform distribution on the interval $\Delta$; when $\Delta=[0,1]$, we simply write $U$. Given a probability density function (pdf) $p$, we denote its push-forward under the function $f$ as $f \# p$. $\delta_{x}$ refers to the Dirac delta distribution. $\mathfrak{B}^d$ stands for the Borel $\sigma$-algebra on $\R^d$, i.e., the smallest $\sigma$-algebra on $\R^d$ that contains all open subsets of $\R^d$. For a vector $\mathbf{b}\in\R^d$, we let $\|\mathbf{b}\|_{\infty}:=\max_{i}|b_i|$, similarly we write $\|\mathbf{A}\|_{\infty}:=\max_{i,j}|\mathbf{A}_{i,j}|$ for the matrix $\mathbf{A}\in\R^{m\times n}$. The Cartesian product of the intervals $\mathcal{I}_{i},\,i \in [1\!:\!d]$, is denoted by $\bigtimes_{i=1}^{d} \mathcal{I}_i = \mathcal{I}_1\bigtimes \mathcal{I}_2 \bigtimes \cdots \bigtimes \mathcal{I}_d$.
Finally, $\chi_{\mathcal I}$ stands for the indicator function on the set $\mathcal I$.

We proceed to define ReLU neural networks.
\begin{definition}\label{def:NN}
Let $L\in \N$ and $N_0, N_1, \ldots, N_{L}\in \N$. A \textnormal{ReLU neural network} $\Phi$ is a map $\Phi: \R^{N_0} \to \R^{N_L}$ given by
\begin{equation*}
\Phi = \begin{cases}
\begin{array}{lc} W_1, \ & L=1\\
W_2\circ\rho\circ W_1, \ & L=2\\
W_L\circ\rho \circ W_{L-1} \circ \rho \circ \dots \circ \rho \circ W_{1}, \ & L\ge3 ,
\end{array} \end{cases}
\end{equation*}
where, for $\ell\in [1\!:\!L]$, $W_{\ell}\colon \R^{N_{\ell-1}} \to \R^{N_\ell},W_\ell(\mathbf{x}):=\mathbf{A}_\ell \mathbf{x} + \mathbf{b}_\ell$
are the associated affine transformations with (weight) matrices $\mathbf{A}_{\ell}\in \mathbb{R}^{N_{\ell}\times N_{\ell-1}}$ and (bias) vectors $\mathbf{b}_\ell\in \mathbb{R}^{N_\ell}$, and the ReLU activation function $\rho\colon\R\to\R,\ \rho(x) := \max(x,0)$ acts component-wise, i.e., $\rho(x_1,\dots,x_N):=(\rho(x_1),\dots,\rho(x_N))$.
We denote by $\cNN_{d,d'}$ the set of all ReLU networks with input dimension $N_0=d$ and output dimension $N_L=d'$. Moreover, we define the following quantities related to the notion of size of the network $\Phi$:
\begin{itemize}
    \item the \emph{connectivity} $\M(\Phi)$ is the total number of nonzero weights, i.e., entries in the matrices $\mathbf{A}_\ell$, $\ell\in [1\!:\!L]$, and the vectors $\mathbf{b}_\ell$, $\ell\in [1\!:\!L]$
    \item \emph{depth} $\L(\Phi):=L$
    \item \emph{width} $\mathcal{W}(\Phi):=\max_{\ell=0,\dots,L} N_\ell$
\end{itemize}
\end{definition}

The distance between probability measures will be quantified through Wasserstein distance defined as follows.

\begin{definition}\label{def:wasserstein-distance}
Let $\mu$ and $\nu$ be probability measures on $(\R^d, \mathfrak{B}^d)$.
A coupling between $\mu$ and $\nu$ is defined as a probability measure $\pi$ on $(\R^d \times \R^d, \mathfrak{B}^{2d})$ such that $\pi(A_1 \times \R^d) = \mu(A_1)$ and $\pi(\R^d \times A_2) = \nu(A_2)$, for all $A_1,A_2 \in \mathfrak{B}^d$.
Let $\scriptstyle \rprod \displaystyle (\mu,\nu)$ be the set of all couplings between $\mu$ and $\nu$.
The Wasserstein distance between $\mu$ and $\nu$ is defined as \[ W(\mu,\nu):= \inf_{\pi \in \scriptscriptstyle \rprod \scriptstyle (\mu,\nu)} \int_{\R^d \times \R^d} \|\mathbf{x}-\mathbf{y}\| d \pi(\mathbf{x},\mathbf{y}),  \]
where $\|\cdot\|$ denotes Euclidean norm.
\end{definition}

We will frequently use the concept of histogram distributions formalized as follows.

\begin{definition}
\label{gen_1_dim_dist}
A random variable $X$ is said to have a general histogram distribution of resolution $n$ on $[0,1]$, denoted as $X \sim \mathcal{G}[0,1]^1_n$, if for some 
$t_0,t_1,\dots,t_n \in \R$ such that $0=t_0 \leq t_1 \leq \dots \leq t_n=1$ and if $t_i=t_{i+1}$, then $t_{i+2}\neq t_i$, $\forall i \in [0\!:\!(n-2)]$, its pdf is given by
\begin{equation}
\begin{aligned}
p(x) &= \sum_{k=0}^{n-1} w_{k} \kappa_{[t_k,t_{k+1}]}(x), \,\, \mbox{with} \,\, \sum_{k=0}^{n-1} w_{k} d(t_{k},t_{k+1}) = 1, \label{eq:def-gen-hist}
\end{aligned}
\end{equation}
and $w_{k}>0$ for all $k \in [0\!:\!(n-1)]$. Here,
\[\kappa_{[t_k,t_{k+1}]}(x) = \begin{cases}
    \chi_{[t_k,t_{k+1}]}(x), & \mbox{ if } t_k<t_{k+1} \\
    \delta_{x-t_k}, & \mbox{ if } t_k=t_{k+1}\end{cases},\]
and
\begin{equation}
    \label{d_function}
d(t_{k},t_{k+1}) = \begin{cases}
    t_{k+1} - t_{k}, & \mbox{ if } t_k<t_{k+1} \\
    1, & \mbox{ if } t_k=t_{k+1}\end{cases}.
\end{equation}
\end{definition}

General histogram distributions allow for bins of arbitrary size and for point singularities, see the right-hand plot in Figure \ref{fig:gen_dist}. We will, however, mostly be
concerned with histogram distributions of uniform tile size, defined as follows.
\begin{definition}
\label{d_dim_dist}
A random vector $\mathbf{X} = (X_1, X_2, \dots, X_d)^\top$ is said to have a histogram distribution of resolution $n$ on the $d$-dimensional unit cube, denoted as $\mathbf{X} \sim \mathcal{E}[0,1]_n^d$, if its pdf is given by
\[
\begin{aligned}
p_{\mathbf{X}}(\mathbf{x}) = \sum_{\mathbf{k}} w_{\mathbf{k}} \chi_{c_{\mathbf{k}}}(\mathbf{x}), \,\, \mbox{with} \,\, \sum_{\mathbf{k}} w_{\mathbf{k}}  = n^d,
\end{aligned}
\]
and $w_{\mathbf{k}}>0$ for all vectors $\mathbf{k} = (i_1, i_2, \dots, i_d) \in [0\!:\!(n-1)]^d$, referred to as index vectors, and 
$c_{\mathbf{k}} = [i_1/n, (i_1+1)/n] \times [i_2/n, (i_2+1)/n] \times \dots \times [i_d/n, (i_d+1)/n]$.
\end{definition}

Example histogram distributions in the $1$- and $2$-dimensional case, respectively, are depicted in Figs.~\ref{fig:hist_1d} and \ref{fig:hist}. Throughout the paper, to indicate that the random vector $\mathbf{X}$ with distribution $p_{\mathbf{X}}(\mathbf{x})$ satisfies $\mathbf{X} \sim \mathcal{E}[0,1]_n^d$, we shall frequently also write $p_{\mathbf{X}}(\mathbf{x}) \in \mathcal{E}[0,1]_n^d$.

\begin{remark}
 For ease of exposition, in Definitions \ref{gen_1_dim_dist} and \ref{d_dim_dist},
 we let the intervals $[t_k,t_{k+1}]$ and the cubes $c_{\mathbf{k}}$, respectively, be closed, thus allowing the breakpoints to belong to different intervals/cubes. While this comes without loss of generality, for concreteness, it is understood that
 the value of the pdf at a breakpoint is given by the average across the intervals/cubes containing the breakpoint.
\end{remark}

\begin{figure}
\centering
\begin{minipage}{.6\textwidth}
  \begin{flushleft}
  \includegraphics[width=0.9\linewidth]{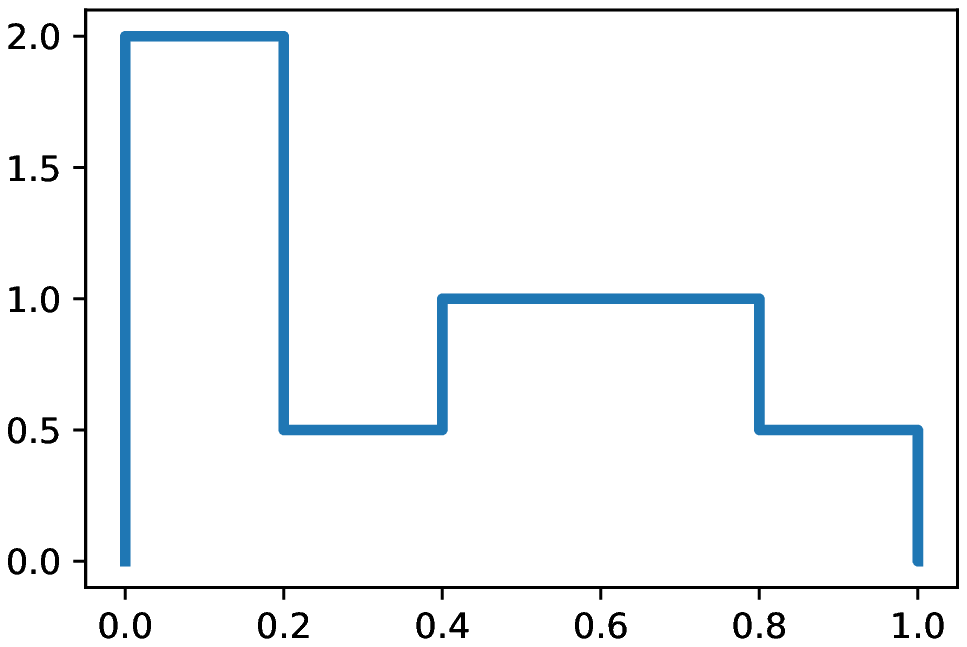}
  \captionof{figure}{Histogram distribution $\mathcal{E}[0,1]_5^1$}
  \label{fig:hist_1d}
  \end{flushleft}
\end{minipage}%
\begin{minipage}{.4\textwidth}
  \begin{flushright}
  \includegraphics[width=\linewidth]{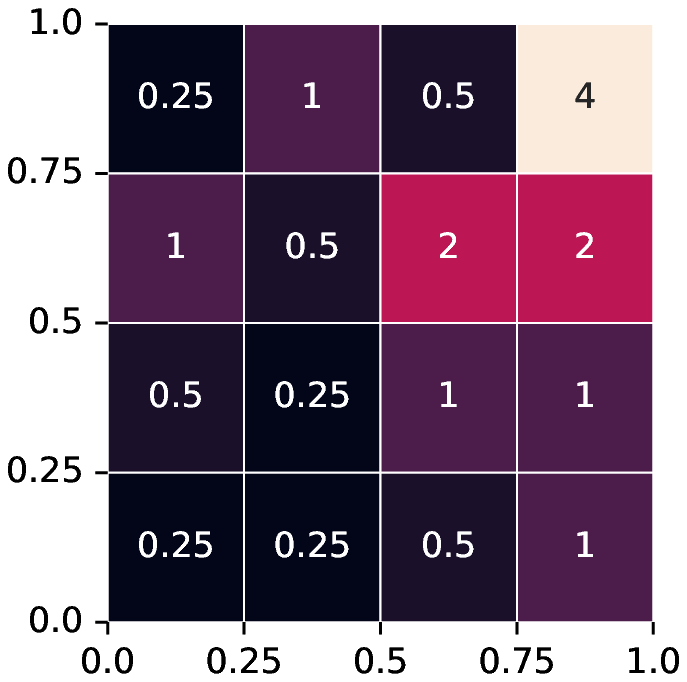}
  \captionof{figure}{Histogram distribution $\mathcal{E}[0,1]_4^2$}
  \label{fig:hist}
  \end{flushright}
\end{minipage}
\end{figure}

\section{Sawtooth functions} \label{sec:sawtooth-functions}

As mentioned above, our universal generative network construction is based on a new space-filling property of ReLU networks, vastly generalizing the one discovered in \cite{Bailey2019,icml2020}. At the heart of this
idea are
higher-order sawtooth functions obtained as follows. 
Consider the sawtooth function $g: \R \rightarrow [0,1]$,
\begin{equation*}
g(x) = \begin{cases} 
2x, &\mbox{if } x \in [0,1/2), \vspace{0.1cm}\\
2(1-x), &\mbox{if } x \in [1/2,1], \\
0, & \mbox{else},
\end{cases}
\end{equation*}
let $g_{1}(x)=g(x)$, and define the sawtooth function of order $s$ as the $s$-fold composition of $g$ with itself according to
\begin{equation*}
\label{g_s_def}
g_s := \underbrace{g \circ g \circ \dots \circ g}_{s},
\hspace{0.3cm} s \geq 2.
\end{equation*}
Figure \ref{fig:saw_def} depicts the sawtooth functions of orders $1,2,$ and $3$. Next, we note that $g$ can be realized by a $2$-layer ReLU network $\Phi_g \in \cNN_{1,1}$ of connectivity $\mathcal{M}(\Phi_g) = 8$ and depth $\mathcal{L}(\Phi_g) = 2$ according to 
$\Phi_g = W_2 \circ \rho \circ W_1=g$
with
\[W_1(x) =  \hspace{-0.1cm} 
 \begin{pmatrix}
  2\\
  4\\
  2
 \end{pmatrix}  
 \hspace{-0.1cm}x \
 -  
 \begin{pmatrix}
  0 \\
  2 \\
  2
 \end{pmatrix} \hspace{-0.1cm},
 \hspace{0.15cm}
 W_2(\mathbf{x}) = \begin{pmatrix}
  1\, & -1\, & 1
 \end{pmatrix}  \hspace{-0.1cm} \begin{pmatrix}
 x_1\\
 x_2\\
 x_3
 \end{pmatrix}\hspace{-0.1cm}.
\]
The $s$th-order sawtooth function $g_s$ can hence be realized by a ReLU network $\Phi^s_g \in \cNN_{1,1}$ of connectivity $\mathcal{M}(\Phi^s_g) = 11s - 3$ and depth $\mathcal{L}(\Phi^s_g) = s+1$ according to $\Phi^s_g = W_2 \circ \rho \circ \underbrace{W_g \circ \rho \circ \cdots \circ W_g \circ \, \rho}_{s-1} \circ\, W_1=g_s$ with
\[W_g(\mathbf{x}) = \begin{pmatrix}
  2\,\, & -2\,\, & 2\\
  4\,\, & -4\,\, & 4\\
  2\,\, & -2\,\, & 2
 \end{pmatrix} \begin{pmatrix}
 x_1\\
 x_2\\
 x_3
 \end{pmatrix}
 -
 \begin{pmatrix}
  0 \\
  2 \\
  2
 \end{pmatrix}\hspace{-0.1cm}.
\]

We close this section with an important technical ingredient of the generalized space-filling idea presented in Section~\ref{sec:distribution-dimensionality-increase}.

\begin{lemma}
\label{p_saw}
Let $f(x)$ be a continuous function on $[0,1]$, with $f(0)=0$. Then, for all $s \in \N$,
\begin{equation}
f(g_s(x)) = \sum^{2^{s-1}-1}_{k=0}  f\big(g(2^{s-1}x-k)\big), \label{eq:fg-decomposition}
\end{equation}
and for all $k \in [0\!:\!(2^{s-1}-1)]$,
\begin{equation}
\supp\big(f\big(g(2^{s-1}x-k)\big)\big) = \Bigg(\frac{k}{2^{s-1}}, \frac{k+1}{2^{s-1}}\Bigg). \label{eq:fg-support}
\end{equation}
\end{lemma}
\begin{proof}
We first note that the sawtooth functions $g_{s}(x)$ satisfy \cite{telgarsky2016benefits}
\[g_s(x) = \sum^{2^{s-1}-1}_{k=0} g(2^{s-1}x-k),\]
with $g(2^{s-1}x-k)$ supported on $\left(\frac{k}{2^{s-1}}, \frac{k+1}{2^{s-1}}\right)$. As $f(0)=0$, the support of $f(g(2^{s-1}x-k))$ coincides with that of $g(2^{s-1}x-k)$, which in turn yields (\ref{eq:fg-support}). To establish (\ref{eq:fg-decomposition}), we note that the supports of
$g(2^{s-1}x-k)$ are pairwise disjoint across $k$ and hence
\[
\hspace{1.15cm}f(g_s(x)) = f\Bigg(\sum^{2^{s-1}-1}_{k=0} g(2^{s-1}x-k)\Bigg) = \sum^{2^{s-1}-1}_{k=0}  f\big(g(2^{s-1}x-k)\big).\hspace{1.15cm}\qedhere
\]
\end{proof}

\begin{figure}
    \begin{center}
    \captionsetup{justification=centering}
    \hspace{-0.5cm} \includegraphics[width=\textwidth]{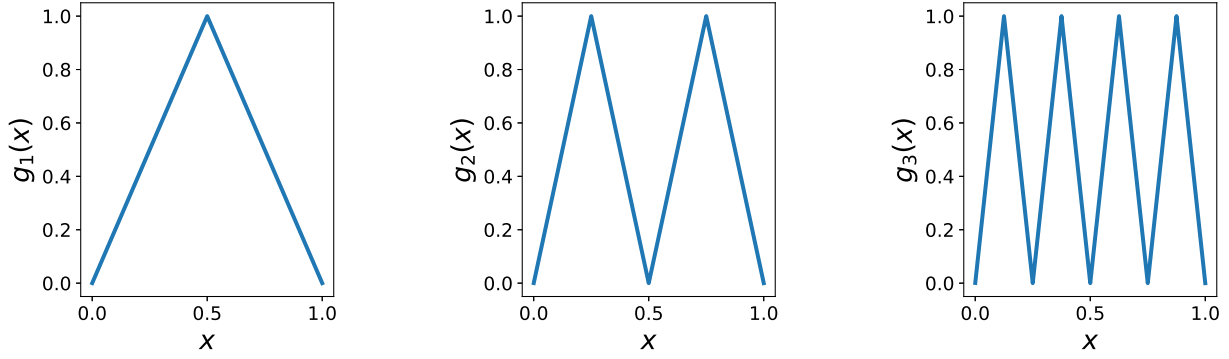}
     \caption{Sawtooth functions}
     \label{fig:saw_def}
     \end{center}
\end{figure}
\section{ReLU networks generate histogram distributions}

This section establishes a systematic connection between ReLU networks and histogram distributions. Specifically, we show that the pushforward of a uniform distribution under a piecewise linear function results in a histogram distribution. We also identify, for a given histogram distribution, the piecewise linear function generating it under pushforward of a uniform distribution. Combined with the insight that ReLU networks realize piecewise linear functions, the desired connection is established.

We start with an auxiliary result.

\begin{lemma}
\label{one-piece}
Let $a,b \in \R, a<b, \Delta = [a,b]$, and let $h(x)=mx+s$, for $x \in \R$, with $m \in \R, s \in \R$. Then, $Q = h\#U(\Delta)$ is uniformly distributed on $[ma+s,mb+s]$, for $m>0$, and on $[mb+s,ma+s]$, for $m<0$. For $m=0$, the pdf of $Q$ is given by $\delta_{\cdot-s}$.
\end{lemma}
\begin{proof}
We start with the case $m \in \R \setminus \{0\}$. The pdf of the pushforward of a random variable with pdf $p(x)$ under the function $f(x)$ is given by
\begin{equation*}
    \label{push-forward_formula}
    q(y)= p(f^{-1}(y)) \left|\frac{d}{dy}f^{-1}(y)\right|.
\end{equation*}
Particularized to $f^{-1}(y) = h^{-1}(y) = \frac{y-s}{m}$ and $p(x) = \frac{1}{b-a} \chi_\Delta(x)$, this yields
\[ q(y) = \begin{cases} \frac{1}{m(b-a)}, & \text{if } y \in [ma+s, mb+s] \\
0, & \text{otherwise},
\end{cases}\]
for $m>0$, and 
\[ q(y) = \begin{cases} \frac{1}{m(a-b)}, & \text{if } y \in [mb+s, ma+s] \\
0, & \text{otherwise},
\end{cases}\]
for $m<0$. Finally, if $m=0$, then the entire interval $[a,b]$ is mapped to the point $y=s$ and the corresponding pdf is given by $q(y)=\delta_{y-s}$.
\end{proof}

We next show that the pushforward of a uniform distribution under a piecewise linear function results in a (general) histogram distribution.

\begin{theorem}
\label{pwl2histogram}
For every piecewise linear continuous function $f:\R \rightarrow \R$, such that $f(x) \in [0,1], \forall x \in [0,1]$, and $f(0)=0, f(1)=1$, there exists an $n$ so that $f\# U \sim \mathcal{G}[0,1]^1_n$.
\end{theorem}
\begin{proof}
We split the domain of $f$ into $t\in \N$ pairwise disjoint intervals $I_i=[a_i, b_i], i \in [0\!:\!(t-1)]$, each of which $f$ is linear on, 
specifically $f(x)=m_i x + s_i,\,x \,\in\,I_i$.
Using the law of total probability and the chain rule, the pdf of $q = f \#U$ can accordingly be represented as
\begin{equation}
    q(y) = \sum_{j=0}^{t-1} q(y|u \in I_j) \mathbb{P}(u \in I_j). \label{eq:rep-y-density}
\end{equation}
As $U$ is uniform, it is also uniform conditional on being in a given interval $I_j$.
By Lemma \ref{one-piece} it therefore follows that $q(y|u \in I_j)$
can be written as 
\begin{equation*}
    q(y|u\in I_j) = \begin{cases}
         \frac{\chi_{R_j}(y)}{|R_j|}, &\mbox{ if } m_j \neq 0,\\[1mm]
         \delta_{y-s_j}, &\mbox{ if } m_j=0,
    \end{cases}
\end{equation*} 
where $R_j=[m_j a_j+s_j,m_j b_j+s_j]$ if $m_j>0$, and $R_j=[m_j b_j+s_j,m_j a_j+s_j]$ if $m_j<0$.
Noting that by continuity of $f$ and the boundary conditions $f(0)=0,f(1)=1$, we have $\bigcup_j R_j = [0,1]$, it follows that
$q(y)$ in (\ref{eq:rep-y-density}) corresponds to a general histogram distribution according to (\ref{eq:def-gen-hist}) with $n=t$ and
\begin{equation*}
    w_j=\begin{cases}
        \frac{\mathbb{P}(u \in I_j)}{|R_j|}, &\mbox{ if } m_j \neq 0,\\[1mm]
        \mathbb{P}(u \in I_j), &\mbox{ if } m_j =0.
    \end{cases} \label{eq:w-cases} \hspace{1.15cm}\qedhere
\end{equation*}
\end{proof}

We will also need the converse to the result just established, in particular a constructive version thereof explicitly identifying the 
piecewise linear function that leads to a given general histogram distribution under pushforward of a uniform distribution on the interval $[0,1]$.

\begin{theorem}
\label{1d_theorem}
Let $p(x)$ be the pdf of $X \sim \mathcal{G}[0,1]^1_n$ with weights $w_k$, $k \in [0\!:\!(n-1)]$, and breakpoints $0=t_0\leq t_1 \leq \dots \leq t_n=1$, and set $b_0=0$, $b_i = \sum_{k=0}^{i-1} w_k\, d(t_{k}, t_{k+1})$, $i \in [1\!:\!n]$, where 
\[
d(t_{k},t_{k+1}) = \begin{cases}
    t_{k+1} - t_{k}, & \mbox{ if }\,\, t_k<t_{k+1} \\
    1, & \mbox{ if } \,\, t_k=t_{k+1}\end{cases}.
\]
Further, define $a_i$, $i \in [0\!:\!(n-1)]$, as follows:
If $t_0=t_{1}$, then $a_0 = 0$ and $a_1 = \frac{1}{w_1}$. If $t_0\neq t_{1}$, then $a_0 = \frac{1}{w_0}$.
For $k\in[1:(n-2)]$, if $t_k=t_{k+1}$, then $a_k = -\frac{1}{w_{k-1}}$ and $a_{k+1} = \frac{1}{w_{k+1}}$, and, if $t_{k-1}\neq t_k\neq t_{k+1}$, then $a_k = \frac{1}{w_{k}} - \frac{1}{w_{k-1}}$. Finally, if $t_{n-1}\neq 1$, then $a_{n-1} = \frac{1}{w_{n-1}} - \frac{1}{w_{n-2}}$.
Then,
\begin{equation}
\label{eq:f-piecewise-linear}
    f(x) = \sum_{i=0}^{n-1} a_i \rho (x-b_i)
\end{equation}
is the piecewise linear function satisfying $f\#U = p$.
\end{theorem}
\begin{proof}
Let $I_i:=[b_i, b_{i+1}]$, $i \in [0\!:\!(n-1)]$. 
Then, $\bigcup_{i\in[0:(n-1)]} I_i = [0,1]$ and, for all $i \in [0\!:\!(n-1)]$, the function $f(x)$ in (\ref{eq:f-piecewise-linear}) is linear on $I_i$ with slope 
given by
\begin{equation*}
    \sum_{j=0}^{i} a_j= \begin{cases}
        1/w_i, &\mbox{ if } t_i\neq t_{i+1}\\
        0, &\mbox{ if } t_i = t_{i+1}
    \end{cases}.
\end{equation*}
Next, note that under $f(x)$ the interval $I_{i}$ is mapped to the interval $[t_i,t_{i+1}]$ if $t_i\neq t_{i+1}$ and to the singleton $\{t_i\}$ if $t_i = t_{i+1}$.
The proof is finalized upon using $\mathbb{P}(u \in I_i)=b_{i+1}-b_{i}$ to conclude that (cf. the proof of Theorem~\ref{pwl2histogram})
$\mathbb{P}(u \in I_i)/|(1/w_i)(b_{i+1}-b_i)|=w_i$, for $t_i\neq t_{i+1}$, and $\mathbb{P}(u \in I_i)=b_{i+1}-b_{i}=w_i \, d(t_i,t_{i+1})=w_i$ in the case
$t_i=t_{i+1}$.
\end{proof}

An example of a piecewise linear function and the corresponding general histogram distribution according to Theorems~\ref{pwl2histogram} and \ref{1d_theorem} is provided in Figure \ref{fig:gen_dist}.
Theorems~\ref{pwl2histogram} and \ref{1d_theorem} are of independent interest as they allow to conclude that ReLU networks, by virtue of always realizing piecewise linear functions, produce general histogram distributions when pushing forward uniform distributions. In the remainder of the paper, we shall, however, work with histogram distributions $\mathcal{E}[0,1]^d_n$ only, in particular for $d=1$, in which case Theorem~\ref{1d_theorem} takes on a simpler form spelled out next.
\begin{corollary}
\label{1d_theorem_hist}
Let $p(x)$ be the pdf of $X \sim \mathcal{E}[0,1]^1_n$ with weights $w_k$, $k \in [0\!:\!n]$, and let $a_0 = \frac{1}{w_{0}}, \ a_i = \frac{1}{w_{i}} - \frac{1}{w_{i-1}}$, $i \in [1\!:\!(n-1)]$, $b_0=0$, $b_i = \frac{1}{n} \sum_{k=0}^{i-1} w_k$, $i \in [1\!:\!n]$. Then, $p=f\#U$ with the piecewise linear function
\[f(x) = \sum_{i=0}^{n-1} a_i \rho(x-b_i).\]
\end{corollary}

We shall often need the explicit form of $f(x)$ in Corollary~\ref{1d_theorem_hist} on its intervals of linearity $[b_{\ell}, b_{\ell+1}]$. A direct calculation
reveals that
\[f(x) = \frac{x}{w_{\ell}}-\frac{\sum_{i=0}^{\ell-1}w_i}{nw_{\ell}} + \frac{\ell}{n}, \quad x \in [b_{\ell},b_{\ell+1}].\]

\begin{figure}
    \begin{center}
    \captionsetup{justification=centering}
    \hspace{-0.5cm} \includegraphics[width=0.7\textwidth]{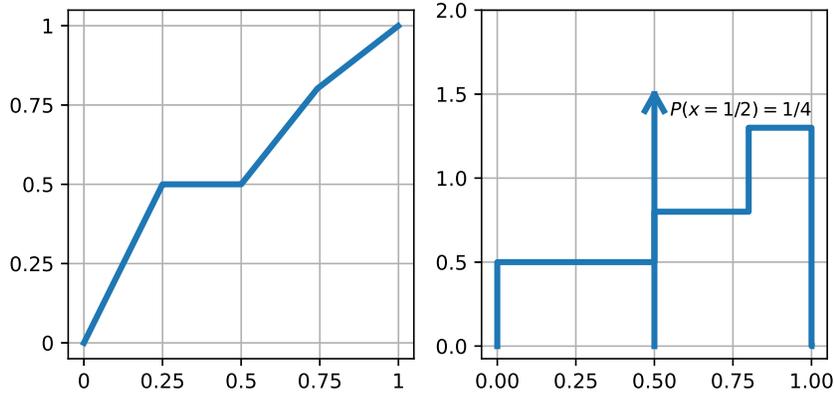}
     \caption{A piecewise linear function $f$ (left) and the corresponding general histogram distribution $f\# U$ (right).}
     \label{fig:gen_dist}
     \end{center}
\end{figure}

\section{Increasing distribution dimensionality} \label{sec:distribution-dimensionality-increase}

This section is devoted to the aspect of distribution dimensionality increase through deep ReLU networks.
Specifically, for a given random vector $\mathbf{X} \sim \mathcal{E}[0,1]_n^d$ with (histogram) distribution $p_{\mathbf{X}}(\mathbf{x})$ of resolution $n$, we construct a piecewise linear map $M: [0,1] \rightarrow [0,1]^d$ such that the pushforward $M \# U$ approximates $p_{\mathbf{X}}(\mathbf{x})$ arbitrarily well.
The main ingredient of our construction is a vast generalization of the space-filling property of sawtooth functions discovered in \cite{icml2020}. Informally speaking, the novel space-filling property we describe allows to completely fill the $d$-dimensional target space according to a prescribed target histogram distribution by transporting probability mass from the $1$-dimensional uniform distribution $U$ to $d$-dimensional space.

We first develop some intuition behind this construction. Specifically, we consider the $2$-dimensional case and visualize the idea of
approximating a $2$-dimensional target distribution through pushforward of $U$ by the sawtooth functions $g_s(x)$ (depicted in Figure \ref{fig:saw_def}) as painting the curve $g_s(x)$ with probability mass taken from $U$.
The geodesic distance traveled by the brush distributing probability mass along $g_s(x)$ goes to infinity according to $2^s$ as $s\rightarrow \infty$. This follows by noting that the number of teeth is $2^s$ and for $s \rightarrow \infty$, the length of the individual teeth (in fact their halves) approaches $1$. Therefore, as $s \rightarrow \infty$ the square $[0,1]^2$ will be filled with paint completely. Moreover, as $x$ traverses from $0$ to $1$, the speed at which probability mass is allocated to the marginal dimensions, i.e., along the $x_1$-and $x_2$-axes, is constant. To see this, simply note that along the $x_2$-axis
the speed of the brush is given by the derivative of $g_s(x)$, which by virtue of $g_s(x)$ consisting of piecewise linear segments, is constant. Likewise, as the inverse of a linear function is again a linear function, the brush moves with constant speed along the $x_1$-axis as well. This guarantees that the resulting $2$-dimensional probability distribution along with its marginals and conditional distributions are all uniform. The rate at which the joint distribution approaches a $2$-dimensional uniform distribution can be quantified in terms of Wasserstein distance by defining the transport map $M:x\rightarrow(x,g_s(x))$ and noting that $W(M \#U, U([0,1]^2) ) \leq \frac{\sqrt{2}}{2^s}$ \cite{Bailey2019}. What is noteworthy here is that the map $M$ takes probability mass from $\mathbb{R}$ to $\mathbb{R}^2$ in a space-filling fashion, i.e., we get a dimensionality increase as $s \rightarrow \infty$.

By adjusting the ``paint plan'', this idea can now be generalized to $2$-dimensional histogram target distributions that are constant with respect to one of the dimensions, here, for concreteness, the $x_1$-dimension. Specifically, we replace $g_s(x)$ in the construction above by $f(g_s(x))$, where the piecewise linear function $f(x)$ determines the paint plan resulting in the desired weights (across the $x_2$-axis) according to Corollary~\ref{1d_theorem_hist}. We refer to Figure~\ref{fig:map} for an illustration of the idea. While the outer function $f(x)$ determines how much time the paint brush spends in a given interval along the $x_2$-axis, the inner function $g_s(x)$ takes care of filling the unit square as $s \rightarrow \infty$. The larger the slope of $f(x)$ on a given interval along the $x_2$-axis, the less time the brush spends in that interval and the smaller the amount of probability mass allocated to the interval. Concretely, by Corollary~\ref{1d_theorem_hist} the amount of probability mass deposited in a given interval is inversely proportional to the slope of $f(x)$ on that interval.

Finally, consider the $2$-dimensional histogram distribution $p_{X_1,X_2}(x_1, x_2) \in \mathcal{E}[0,1]_n^2$ and note that it can be represented as follows
\begin{equation}
\begin{aligned}
\label{chain_rule}
    p_{X_1,X_2}(x_1,x_2) & =  \sum_{k_1,k_2} w_{k_1,k_2} \chi_{c_{k_1,k_2}}(x_1,x_2) =  \sum_{k_1,k_2} w_{k_1}\,w_{k_2|k_1} \chi_{c_{k_1}}(x_1) \chi_{c_{k_2}}(x_2)\\
    & = \sum_{k_1} w_{k_1} \chi_{c_{k_1}}(x_1) \sum_{k_2} w_{k_2|k_1} \chi_{c_{k_2}}(x_2)\\
    & = \sum_{i=0}^{n-1} p_{X_1}\big(x_1 \in [i/n,(i+1)/n]\big) \, p_{X_2|X_1}\big(x_2| x_1 \in [i/n,(i+1)/n]\big), 
\end{aligned} 
\end{equation}
where $w_{k_1}=\frac{1}{n}\sum_{k_2}w_{k_1,k_2}$, $w_{k_2|k_1}=w_{k_1,k_2}/w_{k_1}$, $p_{X_1}(x_1 \in [i/n,(i+1)/n])=w_{i}\chi_{c_{i}}(x_1)$ denotes the restriction of the marginal histogram distribution $p_{X_1}$ (see Lemma~\ref{marginals_pw} below) to the interval $[i/n,(i+1)/n]$,
and $p_{X_2|X_1}\big(x_2| x_1 \in [i/n,(i+1)/n]\big) = \sum_{k_2} w_{k_2|i} \chi_{c_{k_2}}(x_2)$, for each $i\in[0\!:\!(n-1)]$, can be viewed as a $2$-dimensional histogram distribution that is constant with respect to $x_1$, and which we assume to be generated by $f_{X_2}^{(i)}(g_s(x))$ according to the procedure described in the previous paragraph.
 
Now, in order to ``paint'' the general $2$-dimensional histogram distribution in (\ref{chain_rule}), we have to ``squeeze'' the space-filling curves $f_{X_2}^{(i)}(g_s(x))$ into the respective boxes $[i/n,(i+1)/n] \times [0,1]$. This is effected by exploiting that $g_s(x)$ is compactly supported on $[0,1]$ for all $s \in \mathbb{N}$, which allows us to realize the desired localization according to $f_{X_2}^{(i)}(g_s(nx-i))$. The resulting localized space-filling curves are then stitched together by adding them up according to $\sum_{i=0}^{n-1}f_{X_2}^{(i)}(g_s(nx-i))$. Denoting the piecewise linear function that generates the marginal histogram distribution $p_{X_{1}}$ according to Corollary~\ref{1d_theorem_hist} as\footnote{The choice of the superscript $\mathbf{z}_1$ in $f_{X_1}^{\mathbf{z}_1}$ will become clear in Definition~\ref{Z-functions} below.}  $f_{X_1}^{\mathbf{z}_1}$, i.e., $f_{X_1}^{\mathbf{z}_1}\#U = p_{X_1}$, the transport map $M: x \rightarrow \Big(f_{X_1}^{\mathbf{z}_1}(x),\sum_{i=0}^{n-1}f_{X_2}^{(i)}(g_s(nf_{X_1}^{\mathbf{z}_1}(x)-i))\Big)$ when applied to $U$ generates $p_{X_1,X_2}$ in (\ref{chain_rule}) asymptotically in $s$. To see this, we first note that the second component of $M$ pushes
forward, by $\sum_{i=0}^{n-1}f_{X_2}^{(i)}(g_s(nx-i))$, the random variable $f_{X_1}^{\mathbf{z}_1}\#U$ resulting from the pushforward of $U$ by the first component of $M$. This allows us to read off the conditional distributions $p_{X_2|X_1}$. Specifically, thanks to the
individual components in $\sum_{i=0}^{n-1}f_{X_2}^{(i)}(g_s(nx-i))$ being disjointly suppported on $[i/n,(i+1)/n]$, we can conclude that the distribution of the $2$-dimensional random variable $M\#U$ satisfies $p_{X_2|X_1}\big(x_2| x_1 \in [i/n,(i+1)/n]\big) = \sum_{k_2} w_{k_2|i} \chi_{c_{k_2}}(x_2), i \in [0\!:\!(n-1)],$ as desired. Next, noting that the distribution $p_{X_1}$ of $f_{X_1}^{\mathbf{z}_1}\#U$ has components $p_{X_1}\big(x_1 \in [i/n,(i+1)/n]\big)=w_{i}\chi_{c_{i}}(x_1)$ supported disjointly on $[i/n,(i+1)/n]$, it follows from $p_{X_1,X_2}(x_1,x_2)=p_{X_1}(x_1)\,p_{X_2|X_1}(x_2|x_1)$ that the distribution of $M\#U$ is given by
\begin{equation*}
    p_{X_1,X_2}(x_1,x_2) = \sum_{i=0}^{n-1} p_{X_1}\big(x_1 \in [i/n,(i+1)/n]\big) \,  p_{X_2|X_1}\big(x_2| x_1 \in [i/n,(i+1)/n]\big),
\end{equation*}
which is (\ref{chain_rule}). An example illustrating the overall construction can be found in Figure \ref{fig:map_comb}. 

\begin{figure}
  \includegraphics[width=\linewidth]{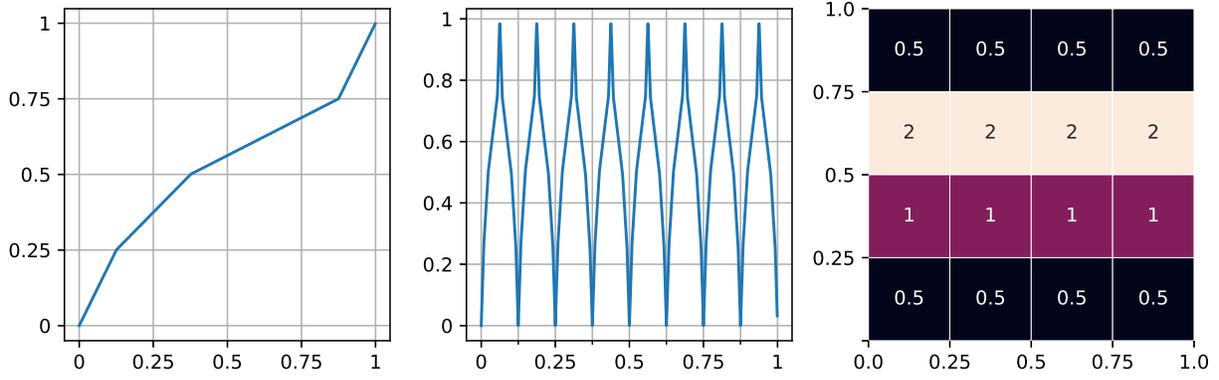}
  \caption{Generating a histogram distribution via the transport map $x \rightarrow (x,f(g_s(x)))$. Left---the function $f(x)$, center---$f(g_4(x))$,
  right---a heatmap of the resulting histogram distribution.}
  \label{fig:map}
\end{figure}

\begin{figure}
  \includegraphics[width=\linewidth]{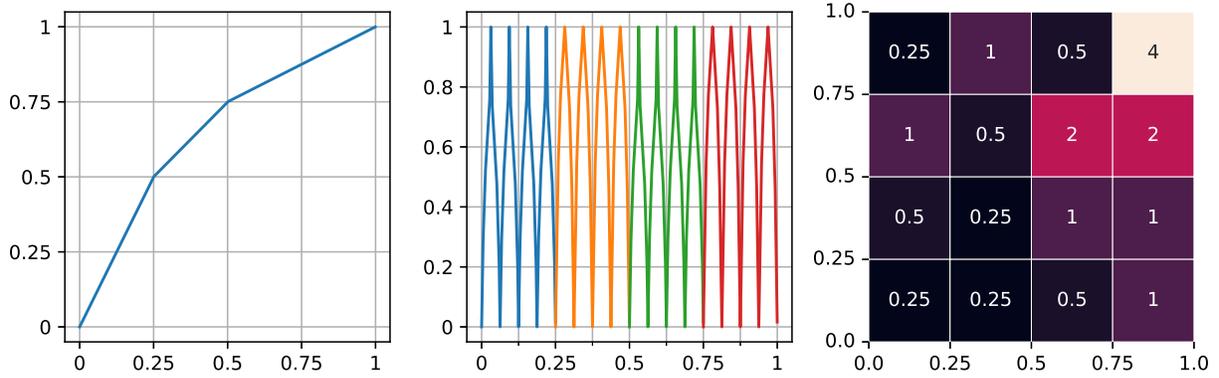}
  \caption{Generating a $2$-D histogram distribution via the transport map $\, x \, \rightarrow (f_{X_1}^{\mathbf{z}_1}(x), \sum_{i=0}^{3} f_{X_2}^{(i)} (g_3 (4f_{X_1}^{\mathbf{z}_1}(x)-i))).$ Left---the function $f_{X_2}^{(1)}=f_{X_2}^{(3)}=f_{X_1}^{\mathbf{z}_1}$, center---$\sum_{i=0}^{3} f_{X_2}^{(i)} (g_3 (4f_{X_1}^{\mathbf{z}_1}(x)-i))$, right---a heatmap of the resulting histogram distribution.
  We took $f_{X_2}^{(0)}=f_{X_2}^{(2)}$ to be given by the function depicted on the left in Figure \ref{fig:map}.}
  \label{fig:map_comb}
\end{figure}

We are now ready to formalize the idea just described and generalize it to target distributions of arbitrary dimension. To this end, we start with a technical lemma stating that all marginal and conditional distributions of a $d$-dimensional histogram distribution are themselves histogram distributions, a result that
was already used implicitly in the description of our main idea in the $2$-dimensional case above.
\begin{lemma}
\label{marginals_pw}
Let $p_{\mathbf{X}}(\mathbf{x})$ be the pdf of the random vector $\mathbf{X} \sim \mathcal{E}[0,1]_n^d$. Then, for all $t \in [1\!:\!(d-1)]$, its
marginal distributions satisfy $p_{X_1,\dots,X_t}(\mathbf{x}_{[1:t]}) \in \mathcal{E}[0,1]_n^t$. Moreover, for all $t \in [1\!:\!(d-1)]$ and $\mathbf{z} = (z_1, z_2, \dots, z_t) \in [0\!:\!(n-1)]^t$, defining
$c_{\mathbf{z}} = [z_1/n, (z_1+1)/n] \times [z_2/n, (z_2+1)/n] \times \dots \times [z_t/n, (z_t+1)/n]$,
the conditional distributions $p_{X_{t+1}|X_1,\dots,X_t}(x_{t+1}|\mathbf{x}_{[1:t]} \in c_{\mathbf{z}})$ are independent of the specific value of $\mathbf{x}_{[1:t]} \in c_{\mathbf{z}}$ and obey $p_{X_{t+1}|X_1,\dots,X_t}(x_{t+1}|\mathbf{x}_{[1:t]} \in c_{\mathbf{z}}) \in \mathcal{E}[0,1]_n^1$.
\end{lemma}

\begin{proof}
The proof of the first statement follows by noting that, for all $t \in [1\!:\!(d-1)]$,
\begin{equation}
\label{eq:marginal-distribution}
\begin{aligned}p_{X_1,\dots,X_t}(\mathbf{x}_{[1:t]}) &= \int_{[0,1]^{d-t}} \sum_{\mathbf{k}} w_{\mathbf{k}} \chi_{c_{\mathbf{k}}}(\mathbf{x}) dx_{t+1} dx_{t+2} \dots dx_d \\
&= \sum_{\mathbf{z}} w_{\mathbf{z}} \chi_{c_{\mathbf{z}}}(\mathbf{x}_{[1:t]}),  
\end{aligned}
\end{equation}
where $\mathbf{z} = \mathbf{k}_{[1:t]} = (z_1, z_2, \dots, z_t)$ with $\mathbf{k}$ according to Definition~\ref{d_dim_dist}, and $w_{\mathbf{z}} = (1/n)^{d-t}\,\sum_{i_{t+1},\dots, i_d} w_{\mathbf{k}}\,>\,0$. With $\sum_{\mathbf{k}}w_{\mathbf{k}}=n^{d}$ from Definition~\ref{d_dim_dist}, we get
$\sum_{\mathbf{z}}w_{\mathbf{z}}=n^{t-d}\sum_{\mathbf{k}}w_{\mathbf{k}}=n^{t}$, which establishes that (\ref{eq:marginal-distribution}) constitutes a valid histogram distribution in $\mathcal{E}[0,1]_n^t$.

To prove the second statement, we first note that for all $t \in [1\!:\!(d-1)]$,
\[p_{X_{t+1}|X_1,\dots,X_t}(x_{t+1}|\mathbf{x}_{[1:t]}) = \frac{p_{X_1,\dots,X_{t+1}}(\mathbf{x}_{[1:(t+1)]})}{p_{X_1,\dots,X_t}(\mathbf{x}_{[1:t]})} = \frac{\sum_{(\mathbf{z},z_{t+1})} w_{(\mathbf{z},z_{t+1})} \chi_{c_{(\mathbf{z},z_{t+1})}}(\mathbf{x}_{[1:(t+1)]})}{\sum_{\mathbf{z}} w_{\mathbf{z}} \chi_{c_{\mathbf{z}}}(\mathbf{x}_{[1:t]})}. \]
Next, for a given $\mathbf{z}' \in [1\!:\!(n-1)]^t$, we have
\begin{align*}
 p_{X_{t+1}|X_1,\dots,X_t}(x_{t+1}|\mathbf{x}_{[1:t]} \in c_{\mathbf{z}'}) &= \frac{\sum_{z_{t+1}} w_{(\mathbf{z}',z_{t+1})} \chi_{c_{(\mathbf{z}',z_{t+1})}}(\mathbf{x}_{[1:(t+1)]})}{w_{\mathbf{z}'}} \\ &= \sum_{z_{t+1}} \frac{ w_{(\mathbf{z}',z_{t+1})}}{w_{\mathbf{z}'}} \chi_{c_{z_{t+1}}}(x_{t+1}),   
\end{align*}
which allows us to conclude that, for all $t \in [1\!:\!(d-1)]$ and $\mathbf{z} = (z_1, z_2, \dots, z_t) \in [0\!:\!(n-1)]^t$, the conditional distribution $p_{X_{t+1}|X_1,\dots,X_t}(x_{t+1}|\mathbf{x}_{[1:t]} \in c_{\mathbf{z}})$ is independent of the specific value of $\mathbf{x}_{[1:t]} \in c_{\mathbf{z}}$ 
and, thanks to $\sum_{z_{t+1}} \frac{ w_{(\mathbf{z}',z_{t+1})}}{w_{\mathbf{z}'}}=n$, belongs to the class $\mathcal{E}[0,1]_n^1$.
\end{proof}

As a consequence of Lemma~\ref{marginals_pw} it follows---from the chain rule---that the joint distribution 
of a histogram distribution $p_{\mathbf{X}}(\mathbf{x}) \in\mathcal{E}[0,1]_n^d$ is fully specified by the conditional distributions
$p_{X_{t+1}|X_1,\dots,X_t}(x_{t+1}|\mathbf{x}_{[1:t]} \in c_{\mathbf{z}}),t \in [1\!:\!(d-1)], \mathbf{z} = (z_1, z_2, \dots, z_t) \in [0\!:\!(n-1)]^t$, and the marginal distribution $p_{X_1}(x_1)$.

We next define the auxiliary functions $F_{r}$ and $Z_r$ needed in the construction of the $d$-dimensional generalization of the $2$-dimensional transport map 
$M: x \rightarrow \Big(f_{X_1}^{\mathbf{z}_1}(x),\sum_{i=0}^{n-1}f_{X_2}^{(i)}(g_s(nf_{X_1}^{\mathbf{z}_1}(x)-i))\Big)$. 
\begin{definition}
\label{Z-functions}
For $\mathbf{k} =[k_1,\dots,k_t] \in [0\!:\!(n-1)]^t$, $t \in \N,$ 
define $c_{\mathbf{k}} = \Big[\frac{k_1}{n}, \frac{k_1+1}{n}\Big] \times \Big[\frac{k_2}{n}, \frac{k_2+1}{n}\Big] \times \dots \times \Big[\frac{k_t}{n}, \frac{k_t+1}{n}\Big]$.
Let $\mathbf{z} = (z_1,z_2,\dots,z_d) \in [0\!:\!(n-1)]^d$, set $\mathbf{z}_i = \mathbf{z}_{[1:(i-1)]}$, for $i \in [2\!:\!d]$, and fix a histogram distribution $p_{\mathbf{X}}(\mathbf{x}) \in\mathcal{E}[0,1]_n^d$ 
specified by $p^{\mathbf{z}_i}_{X_i} = p_{X_i|X_1,\dots,X_{i-1}}\big(x_i| \mathbf{x}_{[1:(i-1)]} \in c_{\mathbf{z}_i}\big)$, for $i \in [2\!:\!d]$, and\footnote{Formally, $\mathbf{z}_1$, albeit not defined, would correspond to a $0$-dimensional quantity. It is used throughout the paper only for notational convenience.}
$p^{\mathbf{z}_1}_{X_1}(x_1) = p_{X_1}(x_1)$.
For $i \in [1\!:\!d]$, let $f^{\mathbf{z}_i}_{X_i}$ be the piecewise linear function that, according to Corollary \ref{1d_theorem_hist}, satisfies
$f^{\mathbf{z}_i}_{X_i}\#U=p^{\mathbf{z}_i}_{X_i}$, and define recursively, for all $s \in \N$,
\begin{equation}
    F_r(x,\mathbf{z}_{r+1},s) := g_s\big(nf^{\mathbf{z}_r}_{X_r}\big(F_{r-1}(x,\mathbf{z}_r,s)\big) - z_{r}\big), \,\, r\,\in\, [1\!:\!(d-1)],  \label{eq:recursive-def-F}
\end{equation} 
with the initialization 
\[F_0(x,\mathbf{z}_1,s) := x.\]
Further, define the functions $Z_r$ according to
\[Z_r(x,s):= \sum_{\mathbf{z}_{r}} f^{\mathbf{z}_r}_{X_r}\big(F_{r-1}(x,\mathbf{z}_r,s)\big), \,\, r\,\in\, [2\!:\!d], \]
and
\[ Z_1(x,s):= f^{\mathbf{z}_1}_{X_1}(x).\]
\end{definition}

With the quantities just defined, we can write
\[
M: x \rightarrow \Big(f_{X_1}^{\mathbf{z}_1}(x),\sum_{i=0}^{n-1}f_{X_2}^{(i)}(g_s(nf_{X_1}^{\mathbf{z}_1}(x)-i))\Big)=(Z_1(x,s),Z_2(x,s)).
\]
In the $d$-dimensional case, the space-filling transport map that takes a $1$-dimensional uniform distribution into 
a given $d$-dimensional histogram distribution, or more precisely a sufficiently accurate approximation thereof, will be seen to be given by
\begin{equation}
M: x \rightarrow (Z_1(x,s), Z_2(x,s), \dots, Z_d(x,s)). \label{eq:m-dim-map}
\end{equation}
Theorem~\ref{general_theorem} below, the central result of this section, makes this formal. The material from here on up to Theorem~\ref{general_theorem} is all
preparatory and technical. We recommend that it be skipped at first reading and suggest to proceed to Theorem~\ref{general_theorem}, in particular the intuition behind the construction of (\ref{eq:m-dim-map}) provided right after the proof of Theorem~\ref{general_theorem}. We do recommend, however, to first visit Figure~\ref{fig:F_functions}, which illustrates the $F_r$-functions and their role in generating the target histogram distribution.

The following lemma establishes support properties of the $F_r$-functions and corresponding consequences for the $Z_r$-functions.

\begin{figure}
  \includegraphics[width=\linewidth]{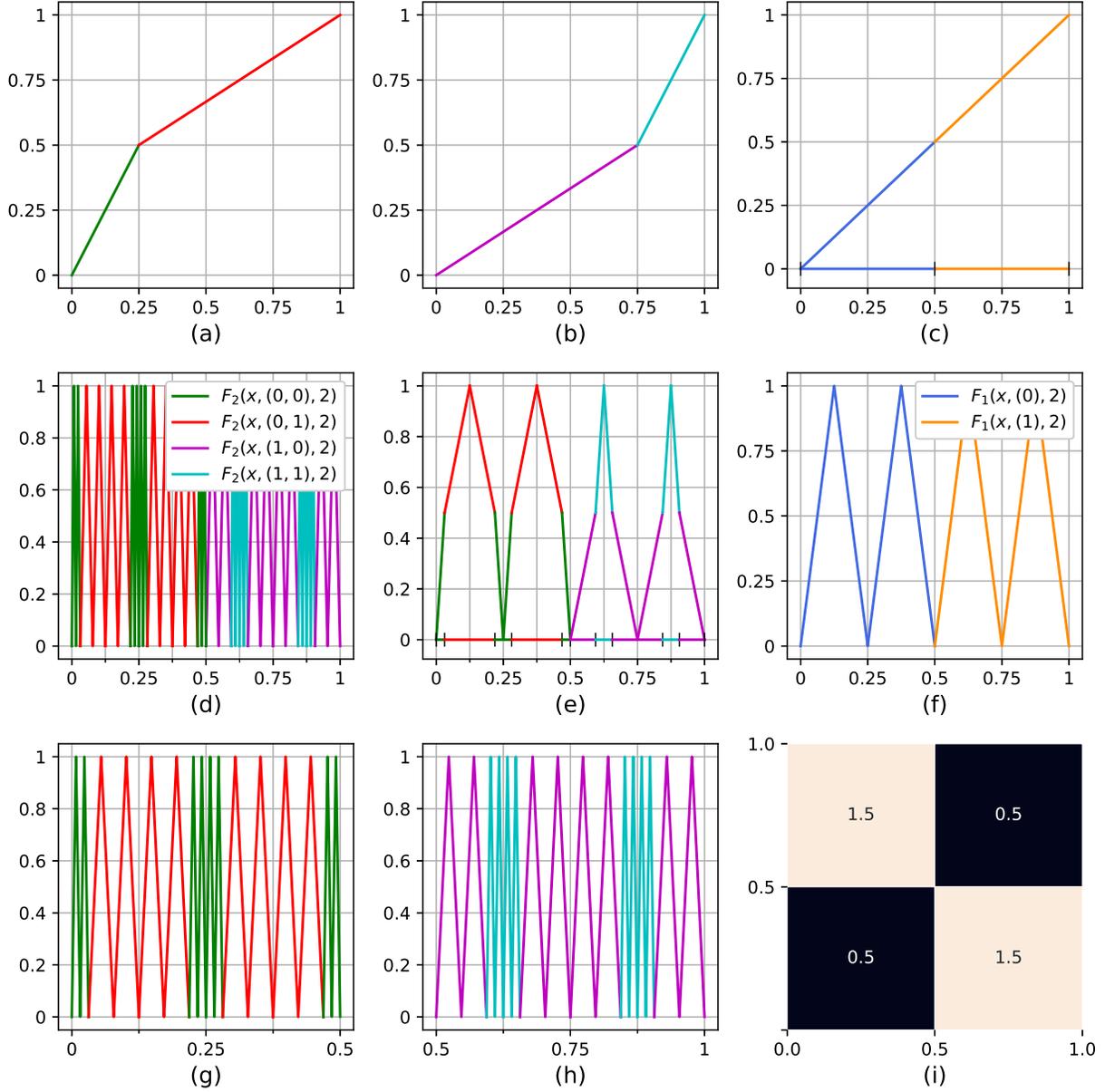}
  \caption{An example illustrating the functions $F_r(x,\mathbf{z}_{r+1},s)$ for $r=1$ and $r=2$. The corresponding target histogram distribution $p_{X_1,X_2}(x_1,x_2)$ is visualized in subplot (i). The function $Z_1(x,s) = f_{X_1}^{\mathbf{z}_1}(x) = x$ characterizing the marginal $p_{X_1} = U$ is shown in subplot (c). The functions $f_{X_2}^{(0)}(x)$ and $f_{X_2}^{(1)}(x)$ characterizing $p^{(0)}_{X_2}$ and $p^{(1)}_{X_2}$, respectively, are
  depicted in subplots (a) and (b). Subplot (e) shows $Z_2(x,2)$ and subplot (f) visualizes $F_1(x,(0),2)$ and $F_1(x,(1),2)$. The functions $F_2(x,\mathbf{z}_{3},2)$ are shown in subplot (d). Subplots (g) and (h) depict zoomed-in versions of subplot (d). The functions $F_2(x,\mathbf{z}_{3},2)$ have disjoint support sets, but, in contrast to $F_1(x,(0),2)$ and $F_1(x,(1),2)$, their support sets are not connected. The support sets of the colored pieces in subplot (d) match those of the respective colored pieces in subplot (e), likewise for subplots (f) and (c).
  }
  \label{fig:F_functions}
\end{figure}

\begin{lemma}
\label{separation_property}
Let $F_i, i \in [0\!:\!(d-1)],  \mathbf{z}_i,  i\in [2\!:\!d],  Z_i, f^{\mathbf{z}_i}_{X_i}, i\in [1\!:\!d]$, be as in Definition~\ref{Z-functions}. Then, for all $r \in [2\!:\!d]$, we have
\[ \bigcap_{\mathbf{z}_r} \supp \big(F_{r-1}(x,\mathbf{z}_{r},s)\big) = \varnothing,\]
and hence for every $r \in [2\!:\!d]$, for all $\mathbf{t}_r \in [0\!:\!(n-1)]^{r-1}$, it holds that
\[ f^{\mathbf{t}_r}_{X_r}\big(F_{r-1}(x,\mathbf{t}_r,s)\big)>0 \, \Longrightarrow \, Z_r(x,s) = 
f^{\mathbf{t}_r}_{X_r}\big(F_{r-1}(x,\mathbf{t}_r,s)\big), \quad x \in [0,1].\]
\end{lemma}
\begin{proof}
The proof is through induction across $r$. 
We start with the base case $r=2$. Since the vector $\mathbf{z}_2$ is
, in fact, a scalar, we can write $\mathbf{z}_2 = z_1$
and note that
\[   F_1(x,\mathbf{z}_{2},s) = F_1(x,z_1,s) = g_s\big(nf^{\mathbf{z}_1}_{X_1}(x) - z_{1}\big)\] and
\[Z_2(x,s)= \sum_{\mathbf{z}_{2}} f^{\mathbf{z}_2}_{X_2}\big(g_s\big(nf^{\mathbf{z}_1}_{X_1}(x) - z_{1}\big)\big) = \sum_{z_1=0}^{n-1} f^{(z_1)}_{X_2}\big(g_s\big(nf^{\mathbf{z}_1}_{X_1}(x) - z_{1}\big)\big).\]
Fix an arbitrary $\hat{x}\in[0,1]$. Since $f^{\mathbf{z}_1}_{X_1}(x)\in [0,1]$, it follows that $f^{\mathbf{z}_1}_{X_1}(\hat{x}) \in [t_{1}/n,(t_{1}+1)/n]$ for some $t_1 \in [0\!:\!(n-1)]$. Then, $\big(nf^{\mathbf{z}_1}_{X_1}(\hat{x}) - z_{1}\big) \in [0,1]$ if and only if $z_1=t_1$. Further, as for $x\notin [0,1]$, $g_s(x) = 0$, we have $F_1(\hat{x},z_1,s) = g_s\big(nf^{\mathbf{z}_1}_{X_1}(\hat{x}) - z_{1}\big)=0$, for all $z_1 \neq t_1$.
Combined with 
the fact that $\hat{x}\in[0,1]$ was chosen arbitrarily, this implies that for every $x \in [0,1]$, there exists a $t_1 \in [0\!:\!(n-1)]$ such that for all $z_1 \neq t_1$, it holds that $F_1(x,z_1,s) = 0$.  This can equivalently be expressed as
\begin{equation}
\label{eq:disj_supp_base}\bigcap_{z_1=0}^{n-1} \supp \big(F_{1}(x,z_1,s)\big)  = \varnothing.    
\end{equation}
Next, since $f^{z_1}_{X_2}(0)=0$, for all $z_1 \in [0\!:\!(n-1)]$, it follows from (\ref{eq:disj_supp_base}) that, for all $t_1 \in [0\!:\!(n-1)]$, 
\[
f^{(t_1)}_{X_2}(F_1(x,t_1,s))>0 \, \Longrightarrow \, Z_2(x,s) = \sum_{z_1=0}^{n-1} f^{(z_1)}_{X_2}\big(F_{1}(x,z_1,s)\big) = f^{(t_1)}_{X_2}( F_1(x,t_1,s) ), \quad x \in [0,1].\]
This establishes the base case. To prove the induction step, we assume that the statement holds for some $r \in [2\!:\!(d-1)]$. Then, by the induction assumption, 
\[ \bigcap_{\mathbf{z}_r} \supp \big(F_{r-1}(x,\mathbf{z}_{r},s)\big) = \varnothing,\]
or, equivalently, for every $x \in [0,1]$, there exists a $\mathbf{t}_r \in [0\!:\!(n-1)]^{r-1}$ such that
for all $\mathbf{z}_{r} \neq \mathbf{t}_r$, we have $F_{r-1}(x,\mathbf{z}_{r},s) = 0$.
Now, fix an arbitrary $\hat{x} \in [0,1]$ together with its corresponding $\mathbf{t}_r \in [0\!:\!(n-1)]^{r-1}$ such that
$F_{r-1}(\hat{x},\mathbf{z}_{r},s) = 0$, for all $\mathbf{z}_{r} \neq \mathbf{t}_r$, and consider
\begin{equation}
    \label{F_decomposition}
    F_r(\hat{x},\mathbf{z}_{r+1},s) = g_s\big(nf^{\mathbf{z}_r}_{X_r}\big(F_{r-1}(\hat{x},\mathbf{z}_r,s)\big) - z_{r}\big)
    = \begin{cases} 0, &\mbox{ if }  \mathbf{z}_r \neq \mathbf{t}_r, \\
    g_s\big(nf^{\mathbf{t}_r}_{X_r}\big(F_{r-1}(\hat{x},\mathbf{t}_r,s)\big) - z_{r}\big), &\mbox{ if } \mathbf{z}_r = \mathbf{t}_r.\end{cases}
\end{equation}
Again, as $f^{\mathbf{t}_r}_{X_r}(x)\in [0,1]$, we can conclude that, for an arbitrarily fixed $\hat{x} \in [0,1]$, there is
a $t_r \in [0\!:\!(n-1)]$ such that $f^{\mathbf{t}_r}_{X_r}\big(F_{r-1}(\hat{x},\mathbf{t}_{r},s)\big) \in [t_{r}/n,(t_{r}+1)/n]$.
Then, $\big(nf^{\mathbf{t}_r}_{X_r}\big(F_{r-1}(\hat{x},\mathbf{t}_r,s)\big) - z_{r}\big) \in [0,1]$ if and only if $z_r=t_r$. Thanks to $g_s(x) = 0$, for $x\notin [0,1]$, (\ref{F_decomposition}) becomes
\begin{equation*}
\label{eq:iter_F}
    F_r(\hat{x},\mathbf{z}_{r+1},s) = \begin{cases} 0, &\mbox{ if }  \mathbf{z}_{r+1} \neq \mathbf{t}_{r+1}, \\
    g_s\big(nf^{\mathbf{t}_r}_{X_r}\big(F_{r-1}(\hat{x},\mathbf{t}_r,s)\big) - t_{r}\big) = F_r(\hat{x},\mathbf{t}_{r+1},s), &\mbox{ if } \mathbf{z}_{r+1} = \mathbf{t}_{r+1}.\end{cases}
\end{equation*}
Combined with the fact that $\hat{x}\in[0,1]$ was chosen arbitrarily, this implies that for every $x \in [0,1]$, there exists a $\mathbf{t}_{r+1} \in [0\!:\!(n-1)]^{r}$ such that for all $\mathbf{z}_{r+1} \neq \mathbf{t}_{r+1}$, it holds that $F_{r}(x,\mathbf{z}_{r+1},s) = 0$. This can equivalently be expressed as
\begin{equation}
\label{eq:disj_supp_iter} \bigcap_{\mathbf{z}_{r+1}} \supp \big(F_{r}(x,\mathbf{z}_{r+1},s)\big) = \varnothing.
\end{equation}
Next, since $f^{\mathbf{z}_{r+1}}_{X_{r+1}}(0)=0$, for all $\mathbf{z}_{r+1}$, it follows from (\ref{eq:disj_supp_iter}) that, for all $\mathbf{t}_{r+1} \in [0\!:\!(n-1)]^{r}$,
\[
f^{\mathbf{t}_{r+1}}_{X_{r+1}}\big(F_{r}(x,\mathbf{t}_{r+1},s)\big)>0 \Rightarrow Z_{r+1}(x,s):= \sum_{\mathbf{z}_{r+1}} f^{\mathbf{z}_{r+1}}_{X_{r+1}}\big(F_{r}(x,\mathbf{z}_{r+1},s)\big) = f^{\mathbf{t}_{r+1}}_{X_{r+1}}\big(F_{r}(x,\mathbf{t}_{r+1},s)\big),\, x \in [0,1].\]
This concludes the proof.
\end{proof}

Before proceeding, we need to introduce further notation. Specifically, let $P_1 = [a,b], P_2 = [c,d] $ be intervals in $[0,1]$, i.e., $0\leq a < b \leq 1$ and $0\leq c < d \leq 1$. Then, we define $P:=P_1 \diamond P_2$ according to $P=[a+c(b-a), a+d(b-a)]$. Note that $|P|=|P_1||P_2|$. The $\diamond$-operation is associative in the sense of $(P_1 \diamond P_2) \diamond P_3=P_1 \diamond (P_2 \diamond P_3)$. We further define the function $N(x,[a,b]) = a + x(b-a),\, x\,\in\,[0,1]$, which rescales $[0,1]$ to the interval $[a,b]$, and note that 
\begin{equation}
\label{N_property}
    N(N(x,[a,b]),[c,d]) = N(x,[c,d]\diamond[a,b]).
\end{equation}
We will also need the function $N^{-}(x,[a,b]) = b - x(b-a)$, $x \in [0,1]$, which, like $N(x,[a,b])$, rescales $[0,1]$ to the interval $[a,b]$, but does so
in reverse manner, i.e., by mapping $x=0$ to $b$ and $x=1$ to $a$. 
Additionally, we define the operator $S([a,b])=[1-b,1-a]$, for all $0\leq a < b \leq 1$, which maps the interval $[a,b] \subseteq [0,1]$ to 
the interval $[1-b,1-a]$. Note that $S$ is cardinality-preserving, i.e., $|S([a,b])|=(b-a) = |[a,b]|$. Moreover, we have the relation
\begin{equation}
\label{inverse_N}
    1 - N(x,S([a,b])) = N^{-}(x,[a,b]).
\end{equation} 

The next lemma establishes that the $Z_r$-functions indeed realize the per-bin histogram distributions constituting the desired target histogram distribution.
\begin{lemma}
\label{general_lemma}
Let $\mathbf{z} = (z_1,z_2,\dots,z_d) \in [0\!:\!(n-1)]^d$ and $\Delta_h = \Big[\frac{h}{2^s},\frac{h+1}{2^s} \Big]$, $h \in [0\!:\!(2^s-1)]$. Set $c_{z_i} = \Big[ \frac{z_i}{n}, \frac{z_i+1}{n} \Big]$, for $i \in [1\!:\!d]$. Let $\mathbf{z}_i = \mathbf{z}_{[1:(i-1)]}$ and fix $p_{\mathbf{X}}(\mathbf{x}) \in\mathcal{E}[0,1]_n^d$ with $p^{\mathbf{z}_1}_{X_1} := p_{X_1}$ and $p^{\mathbf{z}_i}_{X_i} := p_{X_i|X_1,\dots,X_{i-1}}\big(x_i| \mathbf{x}_{[1:(i-1)]} \in c_{\mathbf{z}_i}\big)$, for $i \in [2\!:\!d]$, where $p^{\mathbf{z}_i}_{X_i} \in \mathcal{E}[0,1]^1_n$, $i \in [1\!:\!d]$, has weights $w^{\mathbf{z}_i}_{k}$,
for $k \in [0\!:\!(n-1)]$. Let $f^{\mathbf{z}_i}_{X_i}$ be the piecewise linear function which, according to Corollary \ref{1d_theorem_hist}, satisfies $f^{\mathbf{z}_i}_{X_i}\#U=p^{\mathbf{z}_i}_{X_i}$, $i \in [1\!:\!d]$. Define $P^{\mathbf{z}_i}_{r}=\Big[ \frac{1}{n} \sum^{r-1}_{k=0} w^{\mathbf{z}_i}_{k}, \frac{1}{n} \sum^{r}_{k=0} w^{\mathbf{z}_i}_{k} \Big]$ and $P^{\mathbf{z}_i,h}_{r}=P^{\mathbf{z}_i}_{r} \diamond \Delta_h$, $r\in[1\!:\!(n-1)]$, and $P^{\mathbf{z}_i}_{0}=\Big[ 0, \frac{w^{\mathbf{z}_i}_{0}}{n}\Big]$, $P^{\mathbf{z}_i,h}_{0}=P^{\mathbf{z}_i}_{0} \diamond \Delta_h$.
Then, for every $k \in [2\!:\!d]$, for all $\mathbf{z}_k \in [0\!:\!(n-1)]^{k-1}$ and $\mathbf{h}_k = (h_1, h_2, \dots, h_{k-1}) \in [0\!:\!(2^s-1)]^{k-1}$, it holds that
\[Z_k(N(x,T_{k-1}),s) = f^{\mathbf{z}_k}_{X_k}(F_{k-1}(N(x,T_{k-1}),\mathbf{z}_k,s)) = f^{\mathbf{z}_k}_{X_k}(x), \,\, \text{ for } \, x\in [0,1] \text{ and } \, \sum_{i=1}^{k-1} h_{i} \in 2\mathbb{N}_0,\]
and
\[Z_k(N(x,T_{k-1}),s) = f^{\mathbf{z}_k}_{X_k}(F_{k-1}(N(x,T_{k-1}),\mathbf{z}_k,s)) = f^{\mathbf{z}_k}_{X_k}(1-x), \,\, \text{ for } \, x\in [0,1] \text{ and } \, \sum_{i=1}^{k-1} h_{i} \in 2\mathbb{N}_0+1,\]
where $T_i$, $i \in [1\!:\!d]$, is defined recursively according to
\[T_i = \begin{cases}
    T_{i-1} \diamond P^{\mathbf{z}}_{i}, &\text{if } \, \sum_{\ell=1}^{i-1}h_{\ell} \in 2\mathbb{N}_{0}\\
    T_{i-1} \diamond S(P^{\mathbf{z}}_{i}), &\text{if } \, \sum_{\ell=1}^{i-1}h_{\ell} \in 2\mathbb{N}_{0}+1
\end{cases},\]
for $i\in[2\!:\!d]$, and initialized by $T_1 = P^{\mathbf{z}}_0 \diamond P^{\mathbf{z}}_{1}$, with $P^{\mathbf{z}}_i:=P^{\mathbf{z}_i,h_i}_{z_i}$ and
$P^{\mathbf{z}}_0 = [0,1]$. Moreover, $|T_{k}| = \frac{1}{2^{sk}} \, p_{\mathbf{X}}(\mathbf{x}_{[1:k]} \in c_{\mathbf{z}_{k+1}})$, for all $k \in [1\!:\!(d-1)]$.
\end{lemma}

\begin{proof}
The proof is through induction across $k$. We start with the base case $k=2$. Since the vectors $\mathbf{z}_2$ and $\mathbf{t}_2$ are, in fact, scalars, we can write $\mathbf{z}_2 = z_1$ and $\mathbf{t}_2 = t_1$.
Fix $h_1 \in [0\!:\!(2^s-1)]$, $z_1 \in [0\!:\!(n-1)]$, and note that $T_1 = [0,1] \diamond P^{\mathbf{z}}_1 = P^{\mathbf{z}}_1 = P^{\mathbf{z}_1}_{z_1} \diamond \Delta_{h_1} = \Big[ \frac{1}{n} \sum^{z_1-1}_{k=0} w^{\mathbf{z}_1}_{k}, \frac{1}{n} \sum^{z_1}_{k=0} w^{\mathbf{z}_1}_{k} \Big] \diamond \Big[\frac{h_1}{2^s},\frac{h_1+1}{2^s} \Big] = \Bigg[ \frac{1}{n} \sum^{z_1-1}_{k=0} w^{\mathbf{z}_1}_{k} + \frac{h_1 w^{\mathbf{z}_1}_{z_1}}{2^sn}, \frac{1}{n} \sum^{z_1-1}_{k=0} w^{\mathbf{z}_1}_{k} + \frac{(h_1+1) w^{\mathbf{z}_1}_{z_1}}{2^sn} \Bigg]$. Further, note that $|T_1| = \frac{w^{\mathbf{z}_1}_{z_1}}{2^sn} = \frac{1}{2^s}p_{\mathbf{X}}(x_1 \in c_{z_1})$.
By Corollary \ref{1d_theorem_hist}, $f^{\mathbf{z}_1}_{X_1}(x)$ is linear on $P^{\mathbf{z}_1}_{z_1}$ with slope $\frac{1}{w^{\mathbf{z}_1}_{z_1}}$ and boundary points $f^{\mathbf{z}_1}_{X_1}(\frac{1}{n} \sum^{z_1-1}_{k=0} w^{\mathbf{z}_1}_{k}) = z_1/n$ and $f^{\mathbf{z}_1}_{X_1}(\frac{1}{n} \sum^{z_1}_{k=0} w^{\mathbf{z}_1}_{k}) =  (z_1+1)/n$. The explicit form of $f^{\mathbf{z}_1}_{X_1}(x)$ on $P^{\mathbf{z}_1}_{z_1}$ follows from the remark after Corollary~\ref{1d_theorem_hist} as
\[
f^{\mathbf{z}_1}_{X_1}(x) = \frac{x}{w^{\mathbf{z}_1}_{z_1}} - \frac{\sum_{i=0}^{z_1-1}w^{\mathbf{z}_1}_{i}}{nw^{\mathbf{z}_1}_{z_1}} + \frac{z_1}{n}, \quad x \in P^{\mathbf{z}_1}_{z_1}.
\]
Next, since $N(x,T_1) = \frac{1}{n} \sum^{z_1-1}_{k=0} w^{\mathbf{z}_1}_{k} + \frac{(x+h_1) w^{\mathbf{z}_1}_{z_1}}{2^sn}$, 
noting that $T_1 \subset P^{\mathbf{z}_1}_{z_1}$, we obtain
\begin{equation}
\label{Z1_localization}
f^{\mathbf{z}_1}_{X_1}(N(x,T_1)) = \frac{x+h_1}{2^sn} + \frac{z_1}{n}, \quad x \in [0,1],
\end{equation}
and, hence, for $x \in [0,1]$,
\begin{equation}
 \begin{aligned}
f^{(z_1)}_{X_2}\big(F_{1}(N(x,T_1),z_1,s)\big) &= f^{(z_1)}_{X_2}\big(g_s\big(nf^{\mathbf{z}_1}_{X_1}(N(x,T_1)) - z_{1}\big)\big) \\
&=f^{(z_1)}_{X_{2}}(g_s((x+h_1)2^{-s})) \\
         & \overset{(a)}{=} \sum_{j=0}^{2^{s-1}-1} f^{(z_1)}_{X_{2}}(g(2^{s-1}(x+h_1)2^{-s}-j))\\
         &= \sum_{j=0}^{2^{s-1}-1} f^{(z_1)}_{X_{2}}(g(x/2+h_1/2-j))\\
         & \overset{(b)}{=} f^{(z_1)}_{X_{2}}(g(x/2+h_1/2-\lfloor h_1/2 \rfloor))\\
 &\overset{(c)}{=} \begin{cases}
 f^{(z_1)}_{X_{2}}(x), &\text{ if } h_1 \in 2\mathbb{N}_0,\\
 f^{(z_1)}_{X_{2}}(1-x), &\text{ if } h_1 \in 2\mathbb{N}_0+1,
\end{cases}
    \end{aligned} \label{eq:rel-f-F-symm}
\end{equation}
where we used Lemma \ref{p_saw} in (a), the fact that $g(x/2+h_1/2-j) = 0$, for all $x \in [0,1]$, for $j \neq \lfloor h_1/2 \rfloor$ in (b), and $h_1/2-\lfloor h_1/2 \rfloor = 0$ for $h_1 \in 2\mathbb{N}_0$ and $h_1/2-\lfloor h_1/2 \rfloor = 1/2$ for $h_1 \in 2\mathbb{N}_0+1$ along with $g(x)=g(1-x)$, for $x\in[0,1]$, in (c).
Finally, as by Corollary~\ref{1d_theorem_hist}, $f^{(z_1)}_{X_2}(x)>0$, for all $x \in (0,1]$, it follows from (\ref{eq:rel-f-F-symm}) that 
$f^{(z_1)}_{X_2}\big(F_{1}(N(x,T_1),z_1,s)\big) > 0$, for all $x \in (0,1]$ for $h_1\,\in\,2\mathbb{N}_0$, and for all $x \in [0,1)$ for $h_1 \, \in \, 2\mathbb{N}_0+1$. Application of Lemma~\ref{separation_property}
then yields 
\begin{equation}
\label{base_eq}
    Z_2(N(x,T_1),s) = f^{(z_1)}_{X_2}\big(F_{1}(N(x,T_1),z_1,s)\big),
\end{equation}
for all $x \in (0,1]$ for $h_1\,\in\,2\mathbb{N}_0$, and for all $x \in [0,1)$ for $h_1 \, \in \, 2\mathbb{N}_0+1$. To see that (\ref{base_eq}) holds for
$x=0$ and $h_1\,\in\,2\mathbb{N}_0$, simply note that $Z_2(N(0,T_1),s)=\sum_{z_1}f^{(z_1)}_{X_2}(F_{1}(N(0,T_1),z_1,s))$ and
$F_1(N(0,T_1),z_1,s)=g_s(h_1/2^s)=0$, which thanks to $f^{(z_1)}_{X_2}(0)=0$ implies $Z_2(N(0,T_1),s)=f^{(z_1)}_{X_2}(F_{1}(N(0,T_1),z_1,s))=0$. The case $x=1$ and 
$h_1 \, \in \, 2\mathbb{N}_0+1$ follows along the exact same lines noting that $F_1(N(1,T_1),z_1,s)=g_s((h_1+1)/2^s)=0$.
This finalizes the proof of the base case.

The proof of the induction step largely follows the arguments underlying the proof of the base case. Fix $k\in \mathbb{N}$, with $k \geq 2$, and assume 
that for all $\mathbf{z}_k = (z_1, z_2, \dots, z_{k-1}) \in [0\!:\!(n-1)]^{k-1}$ and $\mathbf{h}_k = (h_1, h_2, \dots, h_{k-1})\in [0\!:\!(2^s-1)]^{k-1}$, it holds that
\[Z_k(N(x,T_{k-1}),s) = f^{\mathbf{z}_k}_{X_k}(F_{k-1}(N(x,T_{k-1}),\mathbf{z}_k,s)) = f^{\mathbf{z}_k}_{X_k}(x), \,\, \text{ for } \, x\in [0,1]\text{ and } \, \sum_{i=1}^{k-1} h_{i} \in 2\mathbb{N}_0,\]
and
\[Z_k(N(x,T_{k-1}),s) = f^{\mathbf{z}_k}_{X_k}(F_{k-1}(N(x,T_{k-1}),\mathbf{z}_k,s)) = f^{\mathbf{z}_k}_{X_k}(1-x), \,\, \text{ for } \, x\in [0,1]\text{ and } \,  \sum_{i=1}^{k-1} h_{i} \in 2\mathbb{N}_0+1,\]
with $|T_{k-1}| = \frac{1}{2^{s(k-1)}} \, p_{\mathbf{X}}(\mathbf{x}_{[1:(k-1)]} \in c_{\mathbf{z}_{k}})$. Fix $\mathbf{h}_{k+1} = (h_1, h_2, \dots, h_k)\in [0\!:\!(2^s-1)]^{k}$ and $\mathbf{z}_{k+1} = (z_1, z_2, \dots, z_k) \in [0\!:\!(n-1)]^{k}$. 
Consider $Z_k(N(x,T_{k-1}),s) = f^{\mathbf{z}_k}_{X_k}(F_{k-1}(N(x,T_{k-1}),\mathbf{z}_k,s))$ on the interval 
\[P^{\mathbf{z}}_{k} = P^{\mathbf{z}_k}_{z_k} \diamond \Delta_{h_k} = \Bigg[ \frac{1}{n} \sum^{z_k-1}_{j=0} w^{\mathbf{z}_k}_{j} + \frac{h_k w^{\mathbf{z}_k}_{z_k}}{2^sn}, \frac{1}{n} \sum^{z_k-1}_{j=0} w^{\mathbf{z}_k}_{j} + \frac{(h_k+1) w^{\mathbf{z}_k}_{z_k}}{2^sn} \Bigg].\] 
We first note that \[T_k = \begin{cases}
    T_{k-1} \diamond P^{\mathbf{z}}_{k}, &\text{if }  \sum_{i=1}^{k-1} h_{i} \in 2\mathbb{N}_{0}\\
    T_{k-1} \diamond S(P^{\mathbf{z}}_{k}), &\text{if }  \sum_{i=1}^{k-1} h_{i} \in 2\mathbb{N}_{0}+1
\end{cases},\]
and 
\begin{align*}
|T_k| &= \begin{cases}
    |T_{k-1}| |P^{\mathbf{z}}_{k}|, &\text{if }  \sum_{i=1}^{k-1} h_{i} \in 2\mathbb{N}_{0}\\
    |T_{k-1}| |S(P^{\mathbf{z}}_{k})|, &\text{if }  \sum_{i=1}^{k-1} h_{i} \in 2\mathbb{N}_{0}+1
\end{cases} \\&= |T_{k-1}| \frac{ w^{\mathbf{z}_k}_{z_k}}{2^sn} \\&= \frac{1}{2^{s(k-1)}} \, p_{\mathbf{X}}(\mathbf{x}_{[1:(k-1)]} \in c_{\mathbf{z}_{k}}) \, \frac{1}{2^s} p_{X_k|X_1,\dots,X_{k-1}}\big(x_k \in c_{z_k}| \mathbf{x}_{[1:(k-1)]} \in c_{\mathbf{z}_k}\big) \\&= \frac{1}{2^{sk}} \, p_{\mathbf{X}}(\mathbf{x}_{[1:k]} \in c_{\mathbf{z}_{k+1}}).\end{align*}
We first provide the proof for the case $ \sum_{i=1}^{k-1} h_{i} \in 2\mathbb{N}_0$. By Corollary~\ref{1d_theorem_hist}, $f^{\mathbf{z}_k}_{X_k}(x)$ is linear on $P^{\mathbf{z}_k}_{z_k}$ with slope $1/{w^{\mathbf{z}_k}_{z_k}}$ and boundary points $f^{\mathbf{z}_k}_{X_k}(\frac{1}{n} \sum^{z_k-1}_{j=0} w^{\mathbf{z}_k}_{j}) = z_k/n$ and $f^{\mathbf{z}_k}_{X_k}(\frac{1}{n} \sum^{z_k}_{j=0} w^{\mathbf{z}_k}_{j}) =  (z_k+1)/n$. The explicit form of $f^{\mathbf{z}_k}_{X_k}$ on $P^{\mathbf{z}_k}_{z_k}$ follows from the remark after Corollary~\ref{1d_theorem_hist} as
\[f^{\mathbf{z}_k}_{X_k}(x) = \frac{x}{w^{\mathbf{z}_k}_{z_k}} - \frac{\sum_{j=0}^{z_k-1}w^{\mathbf{z}_k}_{j}}{nw^{\mathbf{z}_k}_{z_k}} + \frac{z_k}{n}.\]
Using $N(x,T_k)=N(N(x,P^{\mathbf{z}}_{k}),T_{k-1})$, which is thanks to (\ref{N_property}), in the induction assumption (for $ \sum_{i=1}^{k-1} h_{i} \in 2\mathbb{N}_0$), we get
\begin{equation}
    \label{ZoomZ}
    \begin{aligned}
        f^{\mathbf{z}_k}_{X_k}(F_{k-1}(N(x,T_k),\mathbf{z}_k,s)) =
        f^{\mathbf{z}_k}_{X_k}(F_{k-1}(N(N(x,P^{\mathbf{z}}_{k}),T_{k-1}),\mathbf{z}_k,s)) \\
        = f^{\mathbf{z}_k}_{X_k}(N(x,P^{\mathbf{z}}_{k})) = \frac{x+h_k}{2^sn} + \frac{z_k}{n}, \quad x \in [0,1].
    \end{aligned}
\end{equation}
Next, for $x \in [0,1]$, it follows from (\ref{eq:recursive-def-F}) and (\ref{ZoomZ}) that
\begin{equation}
    \begin{aligned}
        f^{\mathbf{z}_{k+1}}_{X_{k+1}}\big(F_{k}(N(x,T_k),\mathbf{z}_{k+1},s)\big) &=
        f^{\mathbf{z}_{k+1}}_{X_{k+1}}(g_s(nf^{\mathbf{z}_k}_{X_k}(F_{k-1}(N(x,T_k),\mathbf{z}_k,s)) - z_{k}))\\ 
        &=f^{\mathbf{z}_{k+1}}_{X_{k+1}}(g_s((x+h_k)2^{-s})) \\
         &\overset{(a)}{=} \sum_{j=0}^{2^{s-1}-1} f^{\mathbf{z}_{k+1}}_{X_{k+1}}(g(2^{s-1}(x+h_k)2^{-s}-j))\\
         &= \sum_{j=0}^{2^{s-1}-1} f^{\mathbf{z}_{k+1}}_{X_{k+1}}(g(x/2+h_k/2-j))\\
 &\overset{(b)}{=} f^{\mathbf{z}_{k+1}}_{X_{k+1}}(g(x/2+h_k/2-\lfloor h_k/2 \rfloor))\\
 &\overset{(c)}{=} \begin{cases}
 f^{\mathbf{z}_{k+1}}_{X_{k+1}}(x), &\text{ if }   h_k \in 2\mathbb{N}_0,\\
 f^{\mathbf{z}_{k+1}}_{X_{k+1}}(1-x), &\text{ if }  h_k \in 2\mathbb{N}_0+1,
\end{cases}\\
&\overset{(d)}{=} \begin{cases}
 f^{\mathbf{z}_{k+1}}_{X_{k+1}}(x), &\text{ if }  \sum_{i=1}^{k} h_{i} \in 2\mathbb{N}_0,\\
 f^{\mathbf{z}_{k+1}}_{X_{k+1}}(1-x), &\text{ if } \sum_{i=1}^{k} h_{i} \in 2\mathbb{N}_0+1,
\end{cases}
    \end{aligned} \label{eq:rel-f-F-inductive}
\end{equation}
where we used Lemma \ref{p_saw} in (a), the fact that $g(x/2+h_k/2-j) = 0$, for all $x \in [0,1]$, for $j \neq \lfloor h_k/2 \rfloor$ in (b), $h_k/2-\lfloor h_k/2 \rfloor = 0$ for $h_k\in 2\mathbb{N}_0$ and $h_k/2-\lfloor h_k/2 \rfloor = 1/2$ for $h_k \in 2\mathbb{N}_0+1$ along with $g(x)=g(1-x)$, for $x\in[0,1]$, in (c), and $\sum_{i=1}^{k-1} h_{i} \in 2\mathbb{N}_0$ in (d).
Finally, as by Corollary~\ref{1d_theorem_hist}, $f^{\mathbf{z}_{k+1}}_{X_{k+1}}(x)>0$, for all $x\in(0,1]$, it follows from (\ref{eq:rel-f-F-inductive}) that
$f^{\mathbf{z}_{k+1}}_{X_{k+1}}\big(F_{k}(N(x,T_k),\mathbf{z}_{k+1},s)\big) > 0$, for all $x \in (0,1]$ for $\sum_{i=1}^{k} h_{i} \in 2\mathbb{N}_0$, and for all $x \in [0,1)$ for $\sum_{i=1}^{k} h_{i} \in 2\mathbb{N}_0+1$.
Application of Lemma \ref{separation_property} then yields
\begin{equation*}
\label{final_eq}
    Z_{k+1}(N(x,T_k),s) = f^{\mathbf{z}_{k+1}}_{X_{k+1}}\big(F_{k}(N(x,T_k),\mathbf{z}_{k+1},s)\big),
\end{equation*}
for all $x \in (0,1]$ for $\sum_{i=1}^{k} h_{i} \in 2\mathbb{N}_0$, and for all $x \in [0,1)$ for $\sum_{i=1}^{k} h_{i} \in 2\mathbb{N}_0+1$. The boundary cases i) $x=0$ and $\sum_{i=1}^{k} h_{i} \in 2\mathbb{N}_0$ and ii) $x=1$ and $\sum_{i=1}^{k} h_{i} \in 2\mathbb{N}_0 +1$ follow along the same lines as in the base case upon noting that $F_{k}(N(0,T_k),\mathbf{z}_{k+1},s)=g_s(h_k/2^s)$ and $F_{k}(N(1,T_k),\mathbf{z}_{k+1},s)=g_s((h_k+1)/2^s)$.

We proceed to the proof for the case $\sum_{i=1}^{k-1} h_{i} \in 2\mathbb{N}_0+1$.
Using (\ref{inverse_N}) and $N(x,T_k)=N(N(x,S(P^{\mathbf{z}}_{k})),T_{k-1})$, which is thanks to (\ref{N_property}), in the induction assumption (for $\sum_{i=1}^{k-1} h_{i} \in 2\mathbb{N}_0+1$), we get
\begin{equation}
    \label{ZoomZ2}
    \begin{aligned}
        f^{\mathbf{z}_k}_{X_k}(F_{k-1}(N(x,T_k),\mathbf{z}_k,s)) &=
        f^{\mathbf{z}_k}_{X_k}(F_{k-1}(N(N(x,S(P^{\mathbf{z}}_{k})),T_{k-1}),\mathbf{z}_k,s)) \\
        &= f^{\mathbf{z}_k}_{X_k}(1-N(x,S(P^{\mathbf{z}}_{k}))) \\&= f^{\mathbf{z}_k}_{X_k}(N^{-}(x,P^{\mathbf{z}}_{k})) \\&= \frac{h_k+1-x}{2^sn} + \frac{z_k}{n}, \quad x \in [0,1].
    \end{aligned}
\end{equation}
Next, for $x \in [0,1]$, it follows from (\ref{eq:recursive-def-F}) and (\ref{ZoomZ2}) that
\begin{equation}
    \begin{aligned}
f^{\mathbf{z}_{k+1}}_{X_{k+1}}\big(F_{k}(N(x,T_k),\mathbf{z}_{k+1},s)\big) &=
f^{\mathbf{z}_{k+1}}_{X_{k+1}}(g_s(nf^{\mathbf{z}_k}_{X_k}(F_{k-1}(N(x,T_k),\mathbf{z}_k,s)) - z_{k}))\\ 
&=f^{\mathbf{z}_{k+1}}_{X_{k+1}}(g_s((h_k+1-x)2^{-s})) \\
&\overset{(a)}{=} \sum_{j=0}^{2^{s-1}-1} f^{\mathbf{z}_{k+1}}_{X_{k+1}}(g(2^{s-1}(h_k+1-x)2^{-s}-j))\\
&= \sum_{j=0}^{2^{s-1}-1} f^{\mathbf{z}_{k+1}}_{X_{k+1}}(g(h_k/2+1/2-x/2-j))\\
&\overset{(b)}{=} f^{\mathbf{z}_{k+1}}_{X_{k+1}}(g(h_k/2+1/2-x/2-\lfloor h_k/2 \rfloor))\\
&\overset{(c)}{=} \begin{cases}
f^{\mathbf{z}_{k+1}}_{X_{k+1}}(1-x), &\text{ if } h_k \in 2\mathbb{N}_0,\\
f^{\mathbf{z}_{k+1}}_{X_{k+1}}(x), &\text{ if } h_k \in 2\mathbb{N}_0+1,
\end{cases}
\\
&\overset{(d)}{=} \begin{cases}
f^{\mathbf{z}_{k+1}}_{X_{k+1}}(x), &\text{ if } \sum_{i=1}^k h_i \in 2\mathbb{N}_0,\\
f^{\mathbf{z}_{k+1}}_{X_{k+1}}(1-x), &\text{ if } \sum_{i=1}^k h_i \in 2\mathbb{N}_0+1,
\end{cases}
    \end{aligned} \label{eq:rel-f-F-inductive2}
\end{equation}
where we used Lemma \ref{p_saw} in (a), the fact that $g(h_k/2+1/2-x/2-j) = 0$, for all $x \in [0,1]$, for $j \neq \lfloor h_k/2 \rfloor$ in (b), $h_k/2-\lfloor h_k/2 \rfloor = 0$ for $h_k\in 2\mathbb{N}_0$ and $h_k/2-\lfloor h_k/2 \rfloor = 1/2$ for $h_k \in 2\mathbb{N}_0+1$ along with $g(x)=g(1-x)$, for $x\in[0,1]$, in (c), and $\sum_{i=1}^{k-1} h_i \in 2\N_0+1$ in (d).
 
Finally, as by Corollary~\ref{1d_theorem_hist}, $f^{\mathbf{z}_{k+1}}_{X_{k+1}}(x)>0$, for all $x\in(0,1]$, it follows from (\ref{eq:rel-f-F-inductive2}) that
$f^{\mathbf{z}_{k+1}}_{X_{k+1}}\big(F_{k}(N(x,T_k),\mathbf{z}_{k+1},s)\big) > 0$, for all $x \in (0,1]$ for $\sum_{i=1}^k h_i \in 2\mathbb{N}_0$, and for all $x \in [0,1)$ for $\sum_{i=1}^k h_i \in 2\mathbb{N}_0+1$.
Application of Lemma \ref{separation_property} then yields
\begin{equation*}
\label{final_eq2}
    Z_{k+1}(N(x,T_k),s) = f^{\mathbf{z}_{k+1}}_{X_{k+1}}\big(F_{k}(N(x,T_k),\mathbf{z}_{k+1},s)\big),
\end{equation*}
for all $x \in (0,1]$ for $\sum_{i=1}^k h_i \in 2\mathbb{N}_0$, and for all $x \in [0,1)$ for $\sum_{i=1}^k h_i \in 2\mathbb{N}_0+1$. The boundary cases i) $x=0$ and $\sum_{i=1}^k h_i \in 2\mathbb{N}_0$ and ii) $x=1$ and $\sum_{i=1}^k h_i \in 2\mathbb{N}_0 +1$ follow along the same lines as in the base case upon noting that $F_{k}(N(0,T_k),\mathbf{z}_{k+1},s)=g_s((h_k+1)/2^s)$ and $F_{k}(N(1,T_k),\mathbf{z}_{k+1},s)=g_s(h_k/2^s)$.

This concludes the proof of the induction step and thereby the overall proof.
\end{proof}

We continue with a corollary to Lemma~\ref{general_lemma} complementing the results on $|T_k|,k\in[1\!:\!(d-1)]$, by the corresponding expression for $|T_d|$ and specifying the range of the $Z_r$-functions on the domain $T_d$.

\begin{corollary}
\label{main_consequence}
Let $\mathbf{z} = (z_1,z_2,\dots,z_d) \in [0\!:\!(n-1)]^d$ and $\mathbf{z}_i = \mathbf{z}_{[1:(i-1)]}$. Fix $p_{\mathbf{X}}(\mathbf{x}) \in\mathcal{E}[0,1]_n^d$, and for all $\mathbf{h}_d = (h_1, h_2, \dots, h_d) \in [0\!:\!(2^s-1)]^{d}$,
let $T_k, k \in [1\!:\!d]$, be defined as in Lemma \ref{general_lemma}.
Then, it holds that $|T_d| = \frac{1}{2^{sd}} \, p_{\mathbf{X}}(\mathbf{x} \in c_{\mathbf{z}})$. Moreover, for every $k \in [1\!:\!d]$, for all $x \in T_d$, $Z_k(x,s) \in \bigg[\frac{z_k}{n}+\frac{h_k}{2^sn},\frac{z_k}{n}+\frac{h_k+1}{2^sn}\bigg]$.
\end{corollary}
\begin{proof}
We first prove the statement on $|T_d|$ and start by noting that, owing to Lemma \ref{general_lemma}, \[|T_{d-1}| = \frac{1}{2^{s(d-1)}} \, p_{\mathbf{X}}(\mathbf{x}_{[1:(d-1)]} \in c_{\mathbf{z}_{d}}).\] 
With 
\[P^{\mathbf{z}_d,h_d}_{z_d} = P^{\mathbf{z}_d}_{z_d} \diamond \Delta_{h_d} = \Bigg[ \frac{1}{n} \sum^{z_d-1}_{j=0} w^{\mathbf{z}_d}_{j} + \frac{h_d w^{\mathbf{z}_d}_{z_d}}{2^sn}, \frac{1}{n} \sum^{z_d-1}_{j=0} w^{\mathbf{z}_d}_{j} + \frac{(h_d+1) w^{\mathbf{z}_d}_{z_d}}{2^sn} \Bigg]\]
and $|P^{\mathbf{z}_d,h_d}_{z_d}|=|S(P^{\mathbf{z}_d,h_d}_{z_d})|$,
we get 
\begin{equation*}
\begin{aligned}
|T_d| &= \begin{cases}
    |T_{d-1}| |P^{\mathbf{z}_d,h_d}_{z_d}|, &\text{if } \, \sum_{i=1}^{d-1}h_{i}  \in 2\mathbb{N}_{0}\\
    |T_{d-1}| |S(P^{\mathbf{z}_d,h_d}_{z_d})|, &\text{if } \, \sum_{i=1}^{d-1}h_{i}  \in 2\mathbb{N}_{0}+1
\end{cases}\\&= \frac{1}{2^{s(d-1)}} \, p_{\mathbf{X}}(\mathbf{x}_{[1:(d-1)]} \in c_{\mathbf{z}_{d}}) \, \frac{ w^{\mathbf{z}_d}_{z_d}}{2^sn}\\ &= \frac{1}{2^{s(d-1)}} \, p_{\mathbf{X}}(\mathbf{x}_{[1:(d-1)]} \in c_{\mathbf{z}_{d}}) \, \frac{1}{2^s} p_{X_d|X_1,\dots,X_{d-1}}\big(x_d \in c_{z_d}| \mathbf{x}_{[1:(d-1)]} \in c_{\mathbf{z}_d}\big) \\&=  \frac{1}{2^{sd}}\, p_{\mathbf{X}}(\mathbf{x} \in c_{\mathbf{z}}).
\end{aligned}
\end{equation*}
This establishes the first statement.

To prove the second statement, we first note that, for $k=1$, by (\ref{Z1_localization}), $Z_1(x,s) \in \bigg[\frac{z_1}{n}+\frac{h_1}{2^sn},\frac{z_1}{n}+\frac{h_1+1}{2^sn}\bigg]$, for all $x \in T_{1}$. Next, for every $k \in [2\!:\!d]$, for all $T_{k-1}$, by Lemma \ref{general_lemma}, it holds that
\[Z_k(N(x,T_{k-1}),s) = \begin{cases}
    f^{\mathbf{z}_k}_{X_k}(x), &  \sum_{i=1}^{k-1} h_{i} \in 2\mathbb{N}_{0}\\
    f^{\mathbf{z}_k}_{X_k}(1-x), & \sum_{i=1}^{k-1} h_{i} \in 2\mathbb{N}_{0}+1
\end{cases},
\]
for all $x\in[0,1]$.
Now, arbitrarily fix $z_k \in [0\!:\!(n-1)], h_k \in [0\!:\!(2^s-1)]$ and consider \[T_k = \begin{cases}
    T_{k-1} \diamond P^{\mathbf{z}}_{k}, &\text{if } \sum_{i=1}^{k-1} h_{i} \in 2\mathbb{N}_{0}\\
    T_{k-1} \diamond S(P^{\mathbf{z}}_{k}), &\text{if } \sum_{i=1}^{k-1} h_{i} \in 2\mathbb{N}_{0}+1
\end{cases}\] with 
$P^{\mathbf{z}}_{k}=P^{\mathbf{z}_{k},h_{k}}_{z_k}$. With (\ref{N_property}), this yields, for all $x\in[0,1]$,
\[Z_k(N(x,T_{k}),s) =  \begin{cases}
    Z_k(N(N(x,P^{\mathbf{z}}_{k}),T_{k-1}),s) = f^{\mathbf{z}_k}_{X_k}(N(x,P^{\mathbf{z}}_{k})), & \sum_{i=1}^{k-1} h_{i} \in 2\mathbb{N}_{0}\\
    Z_k(N(N(x,S(P^{\mathbf{z}}_{k})),T_{k-1}),s) = f^{\mathbf{z}_k}_{X_k}(1-N(x,S(P^{\mathbf{z}}_{k}))), & \sum_{i=1}^{k-1} h_{i} \in 2\mathbb{N}_{0}+1
\end{cases}.
\]
Now, by (\ref{ZoomZ}), it follows for $\sum_{i=1}^{k-1} h_{i} \in 2\N_0$ that
\[ f^{\mathbf{z}_k}_{X_k}(N(x,P^{\mathbf{z}}_{k})) = \frac{x+h_k}{2^sn} + \frac{z_k}{n} \in \bigg[\frac{z_k}{n}+\frac{h_k}{2^sn},\frac{z_k}{n}+\frac{h_k+1}{2^sn}\bigg], \mbox{ for } x \in [0,1],\]
and analogously, for $\sum_{i=1}^{k-1} h_{i} \in 2\N_0+1$,
by (\ref{ZoomZ2}),
\[ f^{\mathbf{z}_k}_{X_k}(1-N(x,S(P^{\mathbf{z}}_{k}))) = f^{\mathbf{z}_k}_{X_k}(N^{-}(x,P^{\mathbf{z}}_{k})) = \frac{h_k+1-x}{2^sn} + \frac{z_k}{n} \in \bigg[\frac{z_k}{n}+\frac{h_k}{2^sn},\frac{z_k}{n}+\frac{h_k+1}{2^sn}\bigg], \mbox{ for } x \in [0,1].\]
We have hence shown that, for all $k\in[1\!:\!d]$, $Z_k(x,s) \in \bigg[\frac{z_k}{n}+\frac{h_k}{2^sn},\frac{z_k}{n}+\frac{h_k+1}{2^sn}\bigg]$, for all $x \in T_{k}$. The proof is completed upon noting that $T_d \subseteq T_k$, for all $k \in [1\!:\!d]$.
\end{proof}

We are now ready to state the main result of this section, namely that the piecewise linear map 
$$
M: x \rightarrow (Z_1(x,s), Z_2(x,s), \dots, Z_d(x,s)) 
$$
transports a $1$-dimensional uniform distribution in a space-filling manner to an arbitrarily close approximation of any high-dimensional histogram distribution.
\begin{theorem}
\label{general_theorem}
For every distribution $p_{\mathbf{X}}(\mathbf{x})\,\in\,\mathcal{E}[0,1]_n^d$, the corresponding transport map
\begin{equation}
    \label{general_map}
    M: x \rightarrow (Z_1(x,s), Z_2(x,s), \dots, Z_d(x,s)) 
    \end{equation}
satisfies
\[
    W(M \#U, p_{\mathbf{X}} ) \leq \frac{\sqrt{d}}{n2^s}.
\]
\end{theorem}

\begin{proof}
Let $\mathbf{z} = (z_1,z_2,\dots,z_d) \in [0\!:\!(n-1)]^d$, $\Delta_h = \Big[\frac{h}{2^s},\frac{h+1}{2^s} \Big]$ with $h \in [0\!:\!(2^s-1)]$, and $\mathbf{h} = (h_1,h_2,\dots,h_d) \in [0\!:\!(2^s-1)]^d$. With $c_{z_i} = [ \frac{z_i}{n}, \frac{z_i+1}{n} ]$, $i \in [1\!:\!d]$, let
$c^\mathbf{h}_\mathbf{z}= \bigtimes_{i=1}^{d}(c_{z_i}\diamond \Delta_{h_i})$. 
Let $T_d$ be defined as in Lemma \ref{general_lemma}. By Corollary \ref{main_consequence}, 
$M:T_d \rightarrow c^{\mathbf{h}}_{\mathbf{z}}$ and $|T_d| = \frac{1}{2^{sd}}\,p_{\mathbf{X}}(\mathbf{x} \in c_\mathbf{z})$.
We hence get $(M\#U)(\mathbf{x} \in c^{\mathbf{h}}_{\mathbf{z}}) = |T_d| = \frac{1}{2^{sd}}\, p_{\mathbf{X}}(\mathbf{x} \in c_\mathbf{z})$.
This establishes that the map $M$ transports probability mass $\frac{1}{2^{sd}}\, p_{\mathbf{X}}(\mathbf{x} \in c_\mathbf{z})$ to the cube $c^{\mathbf{h}}_\mathbf{z}$ of volume $\big(\frac{1}{n2^{s}}\big)^d$, for all $\mathbf{z}$. As $p_{\mathbf{X}}$ is a histogram distribution, it is uniformly distributed on its constituent cubes $c_{\mathbf{z}}$, which, in turn, implies that the amount of probability mass it exhibits on each subcube $c^{\mathbf{h}}_\mathbf{z}$ of $c_{\mathbf{z}}$ is given by $\frac{1}{2^{sd}}\, p_{\mathbf{X}}(\mathbf{x} \in c_\mathbf{z})$. The map $M$, when pushing forward $U$, therefore transports exactly the right amount of probability mass to each cube $c^{\mathbf{h}}_\mathbf{z}$ for a coupling between $p_{\mathbf{X}}$ and $M\#U$ to exist. Combining this with
$\|\mathbf{x}-\mathbf{y}\|\leq \frac{\sqrt{d}}{n2^{s}}$, for all points $\mathbf{x},\mathbf{y}$ in a $d$-dimensional cube of
side length $n^{-1}2^{-s}$, it follows from Definition~\ref{def:wasserstein-distance} that 
  \[W(M \#U, p_{\mathbf{X}}) \leq \frac{\sqrt{d}}{n2^s}. \hspace{1.15cm}\qedhere\]
\end{proof}

Theorem \ref{general_theorem} was proven in \cite{icml2020} for $d=2$.
We remark that a space-filling approach for increasing distribution dimensionality was first described by Bailey and Telgarsky in \cite{Bailey2019}. 
Specifically, the construction in \cite{Bailey2019} generates uniform target distributions of arbitrary dimension based on the transport map
$M: x\, \rightarrow\, (x,g_s(x),g_{2s}(x),\dots)$. The generalization introduced in this paper is capable of producing arbitrary histogram target distributions through space-filling transport maps that build on several key ideas, the first two of which are best illustrated by revisiting the $2$-dimensional case with corresponding transport map
$M: x \rightarrow (f_{X_1}^{\mathbf{z}_1}(x),\sum_{i=0}^{n-1}f_{X_2}^{(i)}(g_s(nf_{X_1}^{\mathbf{z}_1}(x)-i)))$. First, $M$ in its second component composes the function $\sum_{i=0}^{n-1}f_{X_2}^{(i)}(g_s(nx-i))$ with its first component $f_{X_1}^{\mathbf{z}_1}(x)$. Formally, this idea is also present in 
the Bailey-Telgarsky map, where the second component $g_s(x)$ can be interpreted as a trivial composition of $g_s(\cdot)$ with the first component, $x$.
It is, in fact, this composition idea that leads to the space-filling property. Second, $\sum_{i=0}^{n-1}f_{X_2}^{(i)}(g_s(nf_{X_1}^{\mathbf{z}_1}(x)-i))$ yields localization through squeezing and shifting of the $f_{X_2}^{(i)}$. This idea allows to realize different marginal distributions for different horizontal histogram bins (see the rightmost subplot in Fig.~\ref{fig:map_comb}) and is not present in the Bailey-Telgarsky construction as, owing to the target distribution being uniform, there is no concept of histogram distributions.
Taken together the two ideas just described allow to generate arbitrary 
marginal histogram distributions $p_{X_2|X_1}\big(x_2| x_{1})$, which are then combined---through the chain rule---with the histogram distribution $p_{X_1}(x_1)$ to the overall target histogram distribution $p_{X_1,X_2}(x_1,x_2)$.

A further idea underlying our transport map construction becomes transparent in the general $d$-dimensional case. Specifically, in taking the space-filling idea to higher dimensions, we note that in the Bailey-Telgarsky map
$M: x\, \rightarrow\, (x,g_s(x),g_{2s}(x),\dots)$, the third component, $g_{2s}(x)$ can actually be interpreted as a composition of $g_s(\cdot)$ with the second component $g_s(x)$, simply as $g_s(g_s(x))=g_{2s}(x)$. Likewise, as already noted in the previous paragraph, the second component, $g_s(x)$, is a composition of $g_s(\cdot)$ with the first component, $x$. This insight informs the recursive definition of the $F_r$-functions according to (\ref{eq:recursive-def-F}), which, modulo the shaping by the localized $f_{X_r}^{\mathbf{z}_{r}}$-functions, can be seen to exhibit this $g_s$-composition property as well. The $Z_r(x,s)$-functions constituting the components of our transport map (\ref{general_map}) are then obtained by applying the localization idea as described above for the $2$-dimensional case. There is, however, an important difference between localization in the $2$-dimensional case and in the general $d$-dimensional case. This is best seen by inspecting the $3$-dimensional case illustrated in Figure~\ref{fig:F_functions}. Specifically, whereas in the $2$-dimensional case the $F_r$-functions are contiguously supported (see subplot (f)), in the $3$-dimensional case, as illustrated in subplot (d), the support sets are disjointed, but exhibit a periodic pattern. Going to higher dimensions yields a fractal-like support set picture. We emphasize that this support set structure is a consequence of interlacing the self-compositions of the $g_s$-functions with the localized per-bin histogram-distribution shaping functions $f_{X_r}^{\mathbf{z}_{r}}$.

We finally note that the transport map $M$ in Theorem~\ref{general_theorem} can be interpreted as a transport operator in the sense of optimal transport theory \cite{MAL-073,villani2008optimal}, with the source distribution being $1$-dimensional and the target-distribution $d$-dimensional. What is special here is that the transport operator acts between spaces of different dimensions and does so in a space-filling manner \cite{mccann2019optimal}.

\section{Realization of transport map through quantized networks}

This section is concerned with the realization of the transport map $M$ by ReLU networks. In particular, we shall consider networks with quantized weights, for three reasons. First, in practice network weights can not be stored as real numbers on a computer, but rather have to be encoded with a finite number of bits.
Second, we want to convince ourselves that the space-filling property of the transport map, brittle as it seems, is, in fact, not dependent on the network weights being
real numbers. Third, we will be able to develop a relationship, presented in Section~\ref{sec:complexity}, between the complexity of target distributions and the complexity of the ReLU networks realizing the corresponding transport maps. Specifically, complexity will be quantified through the number of bits needed to encode the distribution and the network, respectively, to within a prescribed accuracy.

We will see that ReLU networks with quantized weights generate histogram distributions with quantized weights, referred to as quantized histogram distributions in the following. In Section~\ref{sec:approximation of arbitrary distributions}, we will then study the approximation of general distributions
by quantized histogram distributions. Finally, in Section~\ref{sec:approximation-arbitrary-sets}, we put everything together and characterize the error incurred when approximating arbitrary target distributions by the transportation of a $1$-dimensional uniform distribution
through a ReLU network with quantized weights.

Before proceeding, we need to define quantized histogram distributions and quantized networks. We start with scalar distributions.
\begin{definition}
\label{quantized_dist_1d}
Let $\delta = 1/A$, for some $A \in \mathbb{N}$. A random variable $X$ is said to have a $\delta$-quantized histogram distribution of resolution $n$ on $[0,1]$, denoted as $X \sim \tilde{\mathcal{E}}_{\delta}[0,1]_n^1$, if its pdf is given by
\[
\begin{aligned}
p(x) &= \sum_{k=0}^{n-1} w_{k} \chi_{[k/n, (k+1)/n]}(x), \quad \sum_{k=0}^{n-1} w_{k}  = n, \\ &w_{k}= \delta m_{k}>0, \ \ m_{k} \in \mathbb{N}, \ \ \text{for all} \ \ k \in [0\!:\!(n-1)].
\end{aligned}
\]
\end{definition}
We extend this definition to random vectors by saying that a random vector has a $\delta$-quantized histogram distribution, if all its conditional ($1$-dimensional) distributions $p^{\mathbf{z}_i}_{X_i}$ are $\delta$-quantized histogram distributions.

\begin{definition}
\label{quantized_dist}
Let $\delta = 1/A$, for some $A \in \mathbb{N}$. A random vector $\mathbf{X} = (X_1, X_2, \dots, X_d)^\top$ is said to have a $\delta$-quantized histogram distribution of resolution $n$ on the $d$-dimensional unit cube, denoted as $\mathbf{X} \sim \tilde{\mathcal{E}}_{\delta}[0,1]_n^d$, if  $\mathbf{X} \sim \mathcal{E}[0,1]_n^d$ with $p^{\mathbf{z}_i}_{X_i} \in \tilde{\mathcal{E}}_{\delta}[0,1]_n^1$, for every $i \in [1\!:\!d]$, for all $\mathbf{z}_i$.
\end{definition}

We continue with the definition of quantized ReLU networks.

\begin{definition}
\label{quantized_network}
For $\delta > 0$, we say that a ReLU network is $\delta$-quantized if each of its weights is of one of the following two types. A weight $w$ is of Type $1$ if $w \in (\delta \Z \cap [-1/\delta,1/\delta])$ and of Type $2$ if $\frac{1}{w} \in (\delta \Z \cap [-1/\delta,1/\delta])$.
\end{definition}

Formally, the goal of this section is to find, for fixed $p_{\mathbf{X}} \in \tilde{\mathcal{E}}_{\delta}[0,1]_n^d$, a quantized ReLU network $\Phi$ such that $\Phi \# U$ approximates $p_{\mathbf{X}}$ to within a prescribed accuracy. To this end, we start with an auxiliary lemma, which constructs the building blocks of such networks.
\begin{lemma}
\label{NN_encoding_supportive}
For every $\delta$-quantized $p_{\mathbf{X}} \in \tilde{\mathcal{E}}_{\delta}[0,1]_n^d$ with $d>1$, the map $M^r:\mathbb{R}^{n^r+r}\rightarrow \mathbb{R}^{n^{r+1}+r+1}$, $r \in [0\!:\!(d-1)]$, defined as
\[M^0\!:\! F_0(x,\mathbf{z}_{1},s) \rightarrow \Big(F_{1}(x,\mathbf{z}^1_{2},s), F_{1}(x,\mathbf{z}^2_{2},s), \dots, F_{1}(x,\mathbf{z}^{n}_{2},s), Z_1(x,s)\Big),\]
and, for $r \in [1\!:\!(d-1)]$,
\begin{equation*}
    \begin{aligned}
        M^r: &\Big(F_r(x,\mathbf{z}^1_{r+1},s), F_r(x,\mathbf{z}^2_{r+1},s), \dots, F_r(x,\mathbf{z}^{n^r}_{r+1},s), Z_1(x,s), Z_2(x,s), \dots, Z_r(x,s) \Big) \\ &\rightarrow \Big(F_{r+1}(x,\mathbf{z}^1_{r+2},s), F_{r+1}(x,\mathbf{z}^2_{r+2},s), \dots, F_{r+1}(x,\mathbf{z}^{n^{r+1}}_{r+2},s), Z_1(x,s), Z_2(x,s), \dots, Z_{r+1}(x,s) \Big),
    \end{aligned}
\end{equation*}
is realizable through a $\Delta$-quantized ReLU network $\Psi^{M^r} \in\cNN_{n^r+r,n^{r+1}+r+1}$ with $\mathcal{M}(\Psi^{M^r}) = \mathcal{O}(n^{r+2} + sn^{r+1})$ and $\mathcal{L}(\Psi^{M^r}) = s+3$. Here, $\Delta=\frac{\delta}{n}$ and the vectors $\mathbf{z}^{i}_{r} \in [0\!:\!(n-1)]^{r-1}$, $i \in [1\!:\!n^{r-1}]$, are in natural order\footnote{e.g., for $n=2,r=3$, the order is $\mathbf{z}^{1}_{3} = (0,0), \mathbf{z}^{2}_{3} = (0,1), \mathbf{z}^{3}_{3} = (1,0), \mathbf{z}^{4}_{3} = (1,1)$.} with respect to $i$.
\end{lemma}

\begin{proof}
We start with auxiliary results needed in the proof and then proceed to establish the statement for the cases $r=0$ and $r \geq 1$ separately.
According to Corollary \ref{1d_theorem_hist}, for every $k \in [1\!:\!d]$, for all $\mathbf{z}_k \in [0\!:\!(n-1)]^{k-1}$, $f_{X_k}^{\mathbf{z}_k}(x)$ can be 
realized through a ReLU network $\Phi^{\mathbf{z}_k}:\R\rightarrow\R \in \cNN_{1,1}$ given by
\[\Phi^{\mathbf{z}_k}: x \rightarrow  \frac{1}{w_0}\rho(x) + \sum_{i=1}^{n-1} \Big(\frac{1}{w_i} - \frac{1}{w_{i-1}}\Big) \rho\Big(x-\frac{1}{n}\sum_{j=0}^{i-1}w_j\Big),
\]
and satisfying $\mathcal{M}(\Phi^{\mathbf{z}_k}) \leq 4n-2$, $\mathcal{L}(\Phi^{\mathbf{z}_k}) = 2$. For $\Delta = \frac{\delta}{n}$, the network $\Phi^{\mathbf{z}_k}$ is $\Delta$-quantized with the weights $\frac{1}{w_0}$, $\frac{1}{w_{i}}$, and $\frac{1}{w_{i-1}}$ of Type $2$, and the weights $\frac{1}{n}\sum_{j=0}^{i-1}w_j$ of Type $1$.
The networks $\Phi^{\mathbf{z}_k}_i(x)$ implementing $\big(nf^{\mathbf{z}_k}_{X_k}(x) - i\big)$ are in $\cNN_{1,1}$ and have $\mathcal{M}(\Phi^{\mathbf{z}_k}_i) \leq 4n-1$, $\mathcal{L}(\Phi^{\mathbf{z}_k}_i) = 2$, with their weights all of either Type $1$ or Type $2$ w.r.t. $\Delta$-quantization.
The network $\Psi^s_g(x)$ realizing $g_s(x)$ (see Section~\ref{sec:sawtooth-functions}) is in $\cNN_{1,1}$ with $\mathcal{M}(\Psi^s_g) = 11s-3$, $\mathcal{L}(\Psi^s_g) = s+1$, and with all its weights in $\{-4,-2,-1,1,2,4\}$, which are, again, of Type $1$ w.r.t. $\Delta$-quantization. It follows from \cite[Lemma II.3]{deep-it-2019} that the networks $\Psi^{\mathbf{z}_{k}}_{i,s} = \Psi^s_g(\Phi^{\mathbf{z}_k}_i)$ are in $\cNN_{1,1}$ with $\mathcal{M}(\Psi^{\mathbf{z}_{k}}_{i,s}) \leq 8n+22s-8$ and $\mathcal{L}(\Psi^{\mathbf{z}_{k}}_{i,s}) = s+3$. 
 
We are now ready to prove the statement for $r=0$. Here, $M^0\!:\!\mathbb{R}\rightarrow \mathbb{R}^{n+1}$ with
\[M^0\!:\! F_0(x,\mathbf{z}_{1},s) \rightarrow \Big(F_{1}(x,\mathbf{z}^1_{2},s), F_{1}(x,\mathbf{z}^2_{2},s), \dots, F_{1}(x,\mathbf{z}^{n}_{2},s), Z_1(x,s)\Big),\]
or equivalently
\[ M^0\!:\! x \rightarrow \Big(g_s\big(nf^{\mathbf{z}_1}_{X_1}(x)\big), g_s\big(nf^{\mathbf{z}_1}_{X_1}(x) - 1\big), \dots, g_s\big(nf^{\mathbf{z}_1}_{X_1}(x) - (n-1)\big), f^{\mathbf{z}_1}_{X_1}(x)\Big).\]
The networks $\Psi^{\mathbf{z}_{1}}_{i,s}$ realizing the components $g_s\big(nf^{\mathbf{z}_1}_{X_1}(x) - i\big), i\in[0\!:\!(n-1)]$, of the mapping $M^0$ all have depth $s+3$, whereas the network $\Phi^{\mathbf{z}_1}$ implementing the last component of $M^0$, $f^{\mathbf{z}_1}_{X_1}(x)$, has depth $2$. As we want to apply \cite[Lemma II.5]{deep-it-2019}, we hence need to augment $\Phi^{\mathbf{z}_1}$ to depth $s+3$. This is effected by exploiting that
$\Phi^{\mathbf{z}_1}(x)\geq 0, \ \forall x \in \R$, which allows us to retain the input-output relation realized by the network while amending it by multiplications by $1$ (acting as affine transformations) interlaced by applications of $\rho$ for an overall depth of $s+3$. This leads to
the augmented network $\tilde{\Phi}^{\mathbf{z}_1} = \rho \, \circ \, \dots \, \circ \, \rho \, \circ \, \Phi^{\mathbf{z}_1}$, with $\mathcal{M}(\tilde{\Phi}^{\mathbf{z}_1}) \leq 4n+s-1$, $\mathcal{L}(\tilde{\Phi}^{\mathbf{z}_1}) = s+3$. Application of \cite[Lemma II.5]{deep-it-2019} now allows us to conclude that the network $\Psi^{M^0}= \Big(\Psi^{\mathbf{z}_{1}}_{0,s}, \Psi^{\mathbf{z}_{1}}_{1,s}, \dots, \Psi^{\mathbf{z}_{1}}_{n-1,s},\tilde{\Phi}^{\mathbf{z}_1}\Big)$ realizing the map $M^0$ is in $\cNN_{1,1}$ and satisfies $\mathcal{M}(\Psi^{M^0}) = \mathcal{O}(n^2+sn)$,

$\mathcal{L}(\Psi^{M^0}) = s+3$. This proves the statement for $r=0$.

We proceed to the proof for the case $r\geq 1$. To this end, we use (\ref{eq:recursive-def-F}) to write the map $M^r:\mathbb{R}^{n^r+r}\rightarrow \mathbb{R}^{n^{r+1}+r+1}$, for $r \in [1\!:\!(d-1)]$,
as follows
\[M^r: (y_1, y_2, \dots, y_{n^r+r}) \rightarrow \Big( \Big[g_s\Big(nf^{\mathbf{z}^{i}_{r+1}}_{X_{r+1}}(y_i) - k\Big) \Big]_{(i,k) \in ([1:n^r],[0:(n-1)])}, y_{n^r+1}, \dots, y_{n^r+r},  \sum_{i \in [1:n^r]} f^{\mathbf{z}^{i}_{r+1}}_{X_{r+1}}(y_i) \Big), \]
where the notation $[h(i,k)]_{(i,k) \in ([1:n^r],[0:(n-1)])}$ designates the sequence $h(i,k)$ with $(i,k)$ ranging over $([1\!:\!n^r],[0\!:\!(n-1)])$, with ordering according to $\Big( [h(1,k)]_{k \in [0:(n-1)]}, [h(2,k)]_{k \in [0:(n-1)]}, \dots, [h(n^r,k)]_{k \in [0:(n-1)]}\Big)$.
As discussed above, each $g_s\big(nf^{\mathbf{z}^{i}_{r+1}}_{X_{r+1}}(y_i) - k\big)$ can be realized by a network $\Psi^{\mathbf{z}^{i}_{r+1}}_{k,s} \in \cNN_{1,1}$ with $\mathcal{M}(\Psi^{\mathbf{z}^{i}_{r+1}}_{k,s})\leq 8n+22s-8$, $\mathcal{L}(\Psi^{\mathbf{z}^{i}_{r+1}}_{k,s}) = s+3$. We will also need the identity networks $\Phi^{s+3}_{id}(x) = (\rho \circ \dots \circ \rho)(x)=x$, for all $x \geq 0$, with $\mathcal{M}(\Phi^{s+3}_{id}) = s+3$, $\mathcal{L}(\Phi^{s+3}_{id}) = s+3$. Finally, by \cite[Lemma II.6]{deep-it-2019}, there exists a network $\Psi^{\Sigma}$ realizing
the function $\sum_{i \in [1:n^r]} f^{\mathbf{z}^{i}_{r+1}}_{X_{r+1}}(y_i)$, and with
$\Psi^{\Sigma} \in \cNN_{n^r,1}$, $\mathcal{M}(\Psi^{\Sigma})\leq 4n^{r+1}$, $\mathcal{L}(\Psi^{\Sigma}) = 2$. We shall also need the extension of $\Psi^{\Sigma}$ to a network of depth $s+3$ according to $\tilde{\Psi}^{\Sigma} = \rho \, \circ \, \dots \, \circ \, \rho \, \circ \, \Psi^{\Sigma}$ with $\mathcal{M}(\tilde{\Psi}^{\Sigma})\leq 4n^{r+1}+s+1$, $\mathcal{L}(\tilde{\Psi}^{\Sigma}) = s+3$.
The proof is now concluded by realizing the map $M^r$ as a ReLU network $\Psi^{M^r}$ according to
\begin{equation*}
\begin{aligned}
    &\Psi^{M^r}(y_1, y_2, \dots, y_{n^r+r}) \\ &= \Big(\Big[\Psi^{\mathbf{z}^{i}_{r+1}}_{k,s}(y_i)\Big]_{(i,k) \in ([1:n^r],[0:(n-1)])},  \Phi^{s+3}_{id}(y_{n^r+1}), \dots, \Phi^{s+3}_{id}(y_{n^r+r}), \tilde{\Psi}^{\Sigma}(y_1, \dots, y_{n^r}) \Big).
\end{aligned}
\end{equation*}
Application of \cite[Lemma II.5]{deep-it-2019} now yields $\Psi^{M^r} \in \cNN_{n^r+r,n^{r+1}+r+1}$ with $\mathcal{M}(\Psi^{M^r}) = \mathcal{O}(n^{r+2} + sn^{r+1})$, $\mathcal{L}(\Psi^{M^r}) = s+3$.
\end{proof}

The next result characterizes the ReLU networks realizing the transport map and quantifies their size in terms of connectivity and depth.

\begin{lemma}
\label{NN_encoding}
For every $p_{\mathbf{X}} \in \tilde{\mathcal{E}}_{\delta}[0,1]_n^d$ with $d>1$, the corresponding transport map
\begin{equation*}
    M: x \rightarrow (Z_1(x,s), Z_2(x,s), \dots, Z_d(x,s))
\end{equation*} 
can be realized through a $\Delta$-quantized ReLU network $\Psi^{M} \in \cNN_{1,d}$ with $\mathcal{M}(\Psi^{M}) = \mathcal{O}(n^{d} + sn^{d-1})$, $\mathcal{L}(\Psi^{M})=(s+3)d - s-1$, and $\Delta=\frac{\delta}{n}$.
\end{lemma}
\begin{proof}
Consider the map $M':= M^{d-2} \circ M^{d-1} \circ \dots \circ M^0$,
\[ M': x \rightarrow \Big(F_{d-1}(x,\mathbf{z}^1_{d},s), F_{d-1}(x,\mathbf{z}^2_{d},s), \dots, F_{d-1}(x,\mathbf{z}^{n^{d-1}}_{d},s), Z_1(x,s), Z_2(x,s), \dots, Z_{d-1}(x,s) \Big),\]
where the $M^r$, $r \in [0\!:\!(d-2)]$, are as defined in Lemma \ref{NN_encoding_supportive}. Since by Lemma \ref{NN_encoding_supportive}, $M^r, r \in [0\!:\!(d-2)]$, can be realized by a network with connectivity $\mathcal{O}(n^{r+2}+sn^{r+1})$ and depth $s+3$, it follows from \cite[Lemma II.3]{deep-it-2019} that the map $M'$ can be implemented by a network $\Psi' \in \cNN_{1,n^{d-1}+d-1}$, with $\mathcal{M}(\Psi') = \mathcal{O}(n^{d} + sn^{d-1})$, $\mathcal{L}(\Psi')=(s+3)(d-1)$; here, we used $\sum_{k=0}^{d-2} \mathcal{O}(n^{k+2} + sn^{k+1}) = \mathcal{O}(n^{d} + sn^{d-1})$.
Next, consider the map
\[S: \Big(y_1, \dots, y_{n^{d-1} + d-1} \Big) \rightarrow \Big( \rho(y_{n^{d-1} + 1}), \rho(y_{n^{d-1} + 2}), \dots, \rho(y_{n^{d-1} + d-1}),  \rho(\sum_{i \in [1:n^{d-1}]} y_i) \Big), \]
and note that by \cite[Lemma II.5]{deep-it-2019}, there exists a network $\Psi^S \in \cNN_{n^{d-1}+d-1,d}$ with $\mathcal{M}(\Psi^S) \leq n^{d-1} + 2d - 1$, and $\mathcal{L}(\Psi^S)=2$ realizing $S$. The proof is concluded by noting that, thanks to \cite[Lemma II.3]{deep-it-2019}, the desired map $M = S \circ M'$ is realized by
the network $\Psi^{M}:= \Psi^S(\Psi'(x))$, $\Psi^{M} \in \cNN_{1,d}$ with $\mathcal{M}(\Psi^{M}) = \mathcal{O}(n^{d} + sn^{d-1})$, $\mathcal{L}(\Psi')=(s+3)d - s - 1$. Moreover, the weights of $\Psi^{M}$ are either of Type 1 or Type 2 w.r.t. $\Delta$-quantization.
\end{proof}

We are now ready to state the main result of this section, namely that for every quantized histogram distribution $p_{\mathbf{X}}$ and every $\epsilon>0$, there exists a quantized ReLU network $\Psi$ satisfying $W(\Psi \# U,p_{\mathbf{X}}) \leq \epsilon$. In particular, we also quantify the dependence of $\epsilon$ on the resolution $n$ and the dimension $d$ of $p_{\mathbf{X}}$ as well as the depth of the network $\Psi$.

\begin{theorem}
\label{main_histNN_theorem}
For every $\delta$-quantized $p_{\mathbf{X}} \in \tilde{\mathcal{E}}_{\delta}[0,1]_n^d$ with $d>1$, there exists a $\Delta$-quantized ReLU network $\Psi \in \cNN_{1,d}$ with $\mathcal{M}(\Psi) = \mathcal{O}(n^{d} + sn^{d-1})$, $\mathcal{L}(\Psi)=(s+3)d - s - 1$, and $\Delta=\frac{\delta}{n}$, such that
\begin{equation}
  W(\Psi \#U, p_{\mathbf{X}} ) \leq \frac{\sqrt{d}}{n2^s}.   \label{approximation-error-wasserstein-network}
\end{equation}
\end{theorem}
\begin{proof}
By Lemma \ref{NN_encoding}, for every $p_{\mathbf{X}} \in \tilde{\mathcal{E}}_{\delta}[0,1]_n^d$ with $d>1$, the corresponding transport map
\begin{equation*}
    M: x \rightarrow (Z_1(x,s), Z_2(x,s), \dots, Z_d(x,s))
\end{equation*} 
can be realized through a $\Delta$-quantized ReLU network $\Psi^{M} \in \cNN_{1,d}$ with $\mathcal{M}(\Psi^{M}) = \mathcal{O}(n^{d} + sn^{d-1})$, $\mathcal{L}(\Psi^{M})=(s+3)d - s-1$, and $\Delta=\frac{\delta}{n}$. Moreover, as $\tilde{\mathcal{E}}_{\delta}[0,1]_n^d\subset \mathcal{E}[0,1]_n^d$, it follows from Theorem \ref{general_theorem} that
\[
    W(\Psi^{M} \#U, p_{\mathbf{X}} )  \leq \frac{\sqrt{d}}{n2^s}. \hspace{1.15cm}\qedhere
\]
\end{proof}

We note that for fixed histogram resolution $n$, the upper bound on the approximation error (\ref{approximation-error-wasserstein-network}) decays exponentially in $s$ and hence in network depth $\mathcal{L}(\Psi)$. In particular, choosing $s \sim n$, guarantees that the error in Theorem \ref{main_histNN_theorem} decays exponentially in $n$ while the connectivity of the network is in $\mathcal{O}(n^d)$; this behavior is asymptotically optimal as the number of parameters in $\tilde{\mathcal{E}}_{\delta}[0,1]_n^d$ is of the same order.

\section{Approximation of arbitrary distributions on \texorpdfstring{\([0,1]^d\)}{d-dimensional unit cube} by quantized histogram distributions}
\label{sec:approximation of arbitrary distributions}

This section is concerned with the approximation of 
arbitrary distributions $\nu$ supported on $[0,1]^d$ by $\delta$-quantized histogram distributions of resolution $n$ as defined in the previous section.

Define the $k$-dimensional subcube $c_{\mathbf{i}_k} = [i_1/n, (i_1+1)/n] \times [i_2/n, (i_2+1)/n] \times \dots \times [i_k/n, (i_k+1)/n],$ where $\mathbf{i}_k = (i_1,i_2,\dots,i_k) \in [0\!:\!(n-1)]^k$,
and its corner point
\[
	p_{\mathbf{i}_k} = \paren*{\frac{i_1}{n}, \frac{i_2}{n}, \dotsb, \frac{i_k}{n}}.
\]
Next, we discretize the domain $[0,1]^d$ into the subcubes $c_{\mathbf{i}_d}$
and characterize the amount of probability mass $\nu$ assigns to the individual subcubes.
First, set 
\[
	m_{\mathbf{i}_d} := \nu(c_{\mathbf{i}_d}).
\]
Then, for $k \in [1\!:\!(d-1)]$, we define the projections $P_k : \mathbb{R}^d\rightarrow \mathbb{R}^k, (x_1,\dots, x_k, \dots, x_d) \mapsto (x_1,\dots, x_k)$ and the corresponding $k$-dimensional marginals $\nu_k := P_k \# \nu$ with weights
\[
m_{\mathbf{i}_k} := \nu_k(c_{\mathbf{i}_k}).
\]
It will also be useful to define conditional masses according to $n_{i_1}=m_{i_1}$ and, for $k \in [2\!:\!d]$, for all\footnote{Throughout, we use the symbols $\mathbf{i}_{1}$ and $i_1$ interchangeably.} $\mathbf{i}_{k-1}$ with $m_{\mathbf{i}_{k-1}} \neq 0$,
\[
	n_{\mathbf{i}_{k}} := \frac{m_{\mathbf{i}_{k}}}{m_{\mathbf{i}_{k-1}}}.
\]
For $m_{\mathbf{i}_{k-1}} = 0$, we can, in principle, set the conditional masses arbitrarily, but, for concreteness, we choose 
\[
	n_{\mathbf{i}_{k}} := \frac{1}{n}.
\]

Now that we have defined the masses $m_{\mathbf{i}_k}$ and the conditional masses $n_{\mathbf{i}_k}$ for the distribution $\nu$, we can proceed to derive the masses
$\tilde{m}_{\mathbf{i}_k}$ and $\tilde{n}_{\mathbf{i}_k}$ of the corresponding $\delta$-quantized histogram distribution.
Denote the index of the subcube with the highest (original) mass in the first coordinate as\footnote{Formally, $\mathbf{i}_0$, albeit not defined, would correspond to a $0$-dimensional quantity. It is used throughout the paper only for notational consistency.}
\begin{equation}
	i_1^{*(\mathbf{i}_0)} : = \argmax_{i_1 \in [0:(n-1)]} m_{i_1}. \label{eq:first-coordinate}
\end{equation}
If there are multiple subcubes with the same maximal mass, simply pick one of them (it does not matter which one).
Now, for $k = 1$ and $i_1 \neq i_1^{*(\mathbf{i}_0)}$, we choose the quantized masses as follows,
\[
	\tilde{m}_{i_1} := \tilde{n}_{i_1}:=
	\begin{cases}
		\delta \lceil\frac{1}{\delta} m_{i_1}\rceil, &\text{ if } m_{i_1} > 0 \\
		\delta, &\text{ if } m_{i_1} = 0
	\end{cases},
\]
and for $i_1 = i_1^{*(\mathbf{i}_0)}$,
\[
	\tilde{m}_{i_1^{*(\mathbf{i}_0)}} := \tilde{n}_{i_1^{*(\mathbf{i}_0)}} := 1 - \sum_{i_1 \neq i_1^{*(\mathbf{i}_0)}} \tilde{m}_{i_1}.\\
\]
Note that with this definition, the quantized masses $\tilde{m}_{i_1}$ are always nonzero for $i_1 \neq i_1^{*(\mathbf{i}_0)}$, even in subcubes where the original masses $m_{i_1}$ are equal to zero. We will later verify that this is also the case for $i_1 = i_1^{*(\mathbf{i}_0)}$ whenever $\delta < \frac{1}{n(n-1)}$.
For $k \geq 2$, we similarly borrow mass from the subcube with maximum mass, and we do so in each coordinate individually. 
To this end, for each $k \in [2\!:\!d]$, we set for all $\mathbf{i}_{k-1}$,
\[
	i_{k}^{*(\mathbf{i}_{k-1})} : = \argmax_{i_k \in [0:(n-1)]} m_{\mathbf{i}_{k}}.
\]
As in the assignment (\ref{eq:first-coordinate}) for the first coordinate, if there are multiple such values, any of them will do. To define the quantized conditional masses, we set for each $\mathbf{i}_{k-1} \in [0\!:\!(n-1)]^{k-1}$ and each $i_k \neq i_k^{*(\mathbf{i}_{k-1})}$,
\[
	\tilde{n}_{\mathbf{i}_k} := 
	\begin{cases}
		\delta \lceil\frac{1}{\delta} n_{\mathbf{i}_k}\rceil = \delta \Big\lceil\frac{1}{\delta} \frac{m_{\mathbf{i}_k}}{m_{\mathbf{i}_{k-1}}} \Big\rceil, & \text{ if } m_{\mathbf{i}_k} > 0 \\
		\delta, &\text{ if } m_{\mathbf{i}_k} = 0
	\end{cases},
\]
as long as $m_{\mathbf{i}_{k-1}} > 0$.
If $m_{\mathbf{i}_{k-1}} = 0$, we let 
\[
	\tilde{n}_{\mathbf{i}_k} := \delta \Big\lceil\frac{1}{\delta} n_{\mathbf{i}_k} \Big\rceil = \delta \Big\lceil \frac{1}{\delta} \frac{1}{n} \Big\rceil.
\]
We can then define the quantized weights according to
\[
	\tilde{m}_{\mathbf{i}_k} := \tilde{m}_{\mathbf{i}_{k-1}} \tilde{n}_{\mathbf{i}_k} = \tilde{n}_{\mathbf{i}_k} \dotsm \, \tilde{n}_{\mathbf{i}_1}.
\]
Finally, for $i_k = i_k^{*(\mathbf{i}_{k-1})}$, we set
\[
	\tilde{n}_{\paren*{\mathbf{i}_{k-1}, i_k^{*(\mathbf{i}_{k-1})}}} := 1 - \sum_{i_k \neq i_k^{*(\mathbf{i}_{k-1})}} \tilde{n}_{(\mathbf{i}_{k-1}, i_k)}
\]
and correspondingly
\[
	\tilde{m}_{\paren*{\mathbf{i}_{k-1}, i_k^{*(\mathbf{i}_{k-1})}}} := \tilde{m}_{\mathbf{i}_{k-1}} \tilde{n}_{\paren*{\mathbf{i}_{k-1},i_k^{*(\mathbf{i}_{k-1})}}} = \tilde{n}_{\paren*{\mathbf{i}_{k-1},i_k^{*(\mathbf{i}_{k-1})}}} \dotsm \, \tilde{n}_{\mathbf{i}_1}.\\
\]

We now check that the quantized weights verify the following properties:
\begin{enumerate}
\item Correct marginals:
\begin{align*}
	\sum_{i_k=1}^{n} \tilde{m}_{(\mathbf{i}_{k-1}, i_k)} & = \sum_{i_k \neq i_k^{*(\mathbf{i}_{k-1})}} \tilde{m}_{(\mathbf{i}_{k-1}, i_k)} + \tilde{m}_{{\paren*{\mathbf{i}_{k-1},i_k^{*(\mathbf{i}_{k-1})}}}}\\
		& = \sum_{i_k \neq i_k^{*(\mathbf{i}_{k-1})}} \tilde{m}_{\mathbf{i}_{k-1}} \tilde{n}_{(\mathbf{i}_{k-1}, i_k)} + \tilde{m}_{\mathbf{i}_{k-1}} \tilde{n}_{{\paren*{\mathbf{i}_{k-1},i_k^{*(\mathbf{i}_{k-1})}}}}\\
		& = \tilde{m}_{\mathbf{i}_{k-1}} \paren*{\sum_{i_k \neq i_k^{*(\mathbf{i}_{k-1})}} \tilde{n}_{(\mathbf{i}_{k-1}, i_k)} + \paren*{ 1 - \sum_{i_k \neq i_k^{*(\mathbf{i}_{k-1})}} \tilde{n}_{{\paren*{\mathbf{i}_{k-1},i_k}}}} } \\
		& = \tilde{m}_{\mathbf{i}_{k-1}}.
\end{align*}

\item If $\delta < \frac{1}{n(n-1)}$, then all quantized masses are positive. To this end, we first note that
\[
	n_{\paren*{\mathbf{i}_{k-1},i_k^{*(\mathbf{i}_{k-1})}}} = \frac{m_{\paren*{\mathbf{i}_{k-1}, i_k^{*(\mathbf{i}_{k-1})}}}}{m_{\mathbf{i}_{k-1}}} \geq \frac{1}{n}.
\]
Since for $i_k \neq i_{k}^{*(\mathbf{i}_{k-1})}$, we have by definition
\[
	\tilde{n}_{\mathbf{i}_k} - n_{\mathbf{i}_k} \leq \delta,
\]
it follows that
\begin{align*}
	\tilde{n}_{\paren*{\mathbf{i}_{k-1},i_k^{*(\mathbf{i}_{k-1})}}} & = 1 - \sum_{i_k \neq i_k^{*(\mathbf{i}_{k-1})}} \tilde{n}_{(\mathbf{i}_{k-1}, i_k)} \\
		& \geq 1 - \sum_{i_k \neq i_k^{*(\mathbf{i}_{k-1})}} \paren*{n_{(\mathbf{i}_{k-1}, i_k)} + \delta} \\
		& = n_{\paren*{\mathbf{i}_{k-1},i_k^{*(\mathbf{i}_{k-1})}}} - (n-1) \delta \\
		& > \frac{1}{n} - \frac{n-1}{n(n-1)} = 0.
\end{align*}\\
\end{enumerate}

We next formalize the procedure for going from the original masses $m_{\mathbf{i}_k}$ to the quantized masses $\tilde{m}_{\mathbf{i}_k}$ by characterizing a transport map effecting this transition.

\begin{lemma} \label{lemma:move masses lemma}
Let $k \in [1\!:\!d]$, $\nu$ a distribution supported on $[0,1]^k$ and with masses $m_{\mathbf{i}_k}$ in the subcubes $c_{\mathbf{i}_k}$ and conditional masses $n_{\mathbf{i}_k}$, all as specified above. Let the quantized masses $\tilde{m}_{\mathbf{i}_k}$ and the conditional quantized masses $\tilde{n}_{\mathbf{i}_k}$ also be given as above. Then, for all $\mathbf{i}_{k}$, we have
\begin{equation}
\label{eq:sum-lemma8}
\begin{aligned}
	\tilde{m}_{\mathbf{i}_k}  =  m_{\mathbf{i}_k} & + \sum_{k'=1}^k \chi_{[0:(n-1)] \setminus \bracket*{i_{k'}^{*(\mathbf{i}_{k'-1})}}}(i_{k'}) \ \tilde{\Upsilon}(\mathbf{i},k,k'+1) \paren*{\tilde{n}_{\mathbf{i}_{k'}} - n_{\mathbf{i}_{k'}}}  \Upsilon^{\eta}(\mathbf{i},k'-1,1)\\
		& - \sum_{k'=1}^k \chi_{\bracket*{i_{k'}^{*(\mathbf{i}_{k'-1})}}}(i_{k'})  \Upsilon(\mathbf{i},k,k'+1) \paren*{n_{\mathbf{i}_{k'}} - \tilde{n}_{\mathbf{i}_{k'}} } \Upsilon^{\eta}(\mathbf{i},k'-1,1),
\end{aligned}
\end{equation}
where
\[\Upsilon(\mathbf{i},b,a) = \begin{cases} n_{\mathbf{i}_{b}}  \dotsm \, n_{\mathbf{i}_{a}}, &\textnormal{if } b\geq a\\
1, &\textnormal{else}
\end{cases},\]
\[\tilde{\Upsilon}(\mathbf{i},b,a) = \begin{cases} \tilde{n}_{\mathbf{i}_{b}} \dotsm \, \tilde{n}_{\mathbf{i}_{a}}, &\textnormal{if } b\geq a\\
1, &\textnormal{else}
\end{cases},\]
and
\[\Upsilon^{\eta}(\mathbf{i},b,a) = \begin{cases} \eta_{\mathbf{i}_{b}} (i_{b}) \dotsm \, \eta_{\mathbf{i}_{a}} (i_{a}), &\textnormal{if } b\geq a\\
1, &\textnormal{else}
\end{cases},\]
with
\[
	\eta_{\mathbf{i}_k}(i_k) := \begin{cases}
		n_{\mathbf{i}_k}, & \textnormal{if } i_k \neq {i_{k}^{*(\mathbf{i}_{k-1})}} \\
		\tilde{n}_{\mathbf{i}_k}, & \textnormal{if } i_k = {i_{k}^{*(\mathbf{i}_{k-1})}}
	\end{cases}.
\]\\
\end{lemma}

\noindent The proof of Lemma~\ref{lemma:move masses lemma} is provided in the appendix.\\

We are now ready to state the main result of this section. Specifically, we establish an upper bound on the Wasserstein distance between a given (arbitrary) distribution $\nu$ supported on $[0,1]^d$, for any $d\in \mathbb{N}$, and the corresponding $\delta$-quantized histogram distribution of resolution $n$ obtained based on the procedure described above.

\begin{theorem} \label{thm:quantization bound}
Let $d \in \N$. For every distribution $\nu$ supported on $[0,1]^d$, there exists
a $\delta$-quantized histogram distribution $\mu$ of resolution $n$ such that
\[
	W(\mu, \nu) \leq \frac{2\sqrt{d}}{n} + \frac{d(d+1)}{2} (n-1) \delta.
\]
\end{theorem}

\begin{proof}
The proof proceeds in three steps as follows. 
\begin{enumerate}

\item For each $\mathbf{i}_d\in [0\!:\!(n-1)]^d$, we redistribute the mass $m_{\mathbf{i}_d}=\nu(c_{\mathbf{i}_d})$ to a point mass concentrated in the corner point $p_{\mathbf{i}_d}$.
\vspace*{1mm}

\item We transport the masses $m_{\mathbf{i}_d}$ according to the procedure described above to result in the masses
$\tilde{m}_{\mathbf{i}_d}$, still located at $p_{\mathbf{i}_d}$.
\vspace*{1mm}

\item For each $\mathbf{i}_d\in [0\!:\!(n-1)]^d$, we spread out the mass $\tilde{m}_{\mathbf{i}_d}$ uniformly across the subcube indexed by $\mathbf{i}_{d}$.
\end{enumerate}

\noindent\textit{Step 1.} 
We define the distribution 
\[
	\nu' = \sum_{\mathbf{i}_d\in [0:(n-1)]^d} m_{\mathbf{i}_d} \delta_{p_{\mathbf{i}_d}}
\]
and note that transporting $\nu$ to $\nu'$ incurs transportation cost
\begin{equation}
    \label{subdistance}
    W(\nu,\nu') \leq \sum_{\mathbf{i}_d} m_{\mathbf{i}_d} \underbrace{\sqrt{\paren*{\frac{1}{n}}^2 + \dotsb + \paren*{\frac{1}{n}}^2}}_{\text{maximum distance in each subcube}} = \sum_{\mathbf{i}_d} m_{\mathbf{i}_d} \frac{\sqrt{d}}{n} = \frac{\sqrt{d}}{n}.
\end{equation}

\noindent\textit{Step 2.} To redistribute the masses from the original values $m_{\mathbf{i}_d}$ to the quantized values $\tilde{m}_{\mathbf{i}_d}$, we proceed coordinate by coordinate.
Specifically, in the $k$-th coordinate, we carry out two (sets of) transportations. 
The first one moves, for fixed $i_1,  \dotsc  ,i_{k-1}, i_{k+1},  \dotsc  , i_d$, mass from the point $p_{(i_1,  \dotsc  ,i_{k-1}, i_{k}^{*(\mathbf{i}_{k-1})}, i_{k+1},  \dotsc  , i_d)}$ to the points $p_{(i_1,  \dotsc  ,i_{k-1}, i_{k}, i_{k+1},  \dotsc  , i_d)}$, for all $i_{k} \neq i_{k}^{*(\mathbf{i}_{k-1})}$, and does this for all tuples $i_1,  \dotsc  ,i_{k-1}, i_{k+1},  \dotsc  , i_d$. The second set of transportations reconfigures masses in the coordinates $[1 \! : \! (k-1)]$ so as to obtain the correct marginals in coordinate $k$. These reconfigurations moreover preserve the marginals in coordinates $[1 \! : \! (k-1)]$. We make all this precise through the following claim, proved below after Step~3 has been presented.\\

\noindent \textbf{Claim:} 
Reconfiguring the masses between the corner points such that the 
mass in the point $p_{\mathbf{i}_{d}}$ is given by
\begin{align}
	m_{\mathbf{i}_d} & + \sum_{k'=1}^k \chi_{[0:(n-1)] \setminus \bracket*{i_{k'}^{*(\mathbf{i}_{k'-1})}}}(i_{k'}) \ \paren*{\frac{m_{\paren*{i_1, \dotsc, i_{k'}^{*(\mathbf{i}_{k'-1})}, i_{k+1}, \dotsc, i_d}}}{m_{\paren*{i_1, \dotsc, i_{k'}^{*(\mathbf{i}_{k'-1})}}}}} \nonumber\\
	& \phantom{===========} \tilde{\Upsilon}(\mathbf{i},k,k'+1) \paren*{\tilde{n}_{\mathbf{i}_{k'}} - n_{\mathbf{i}_{k'}}} \Upsilon^{\eta}(\mathbf{i},k'-1,1) \label{transportation-kth-dimension}\\
		& - \sum_{k'=1}^k \chi_{\bracket*{i_{k'}^{*(\mathbf{i}_{k'-1})}}}(i_{k'}) \ \Upsilon(\mathbf{i},d,k'+1) \paren*{n_{\mathbf{i}_{k'}} - \tilde{n}_{\mathbf{i}_{k'}} } \Upsilon^{\eta}(\mathbf{i},k'-1,1), \nonumber
\end{align}
where 
\[
	m_{\paren*{i_1, \dotsc, i_{k'}^{*(\mathbf{i}_{k'-1})}, i_{k+1}, \dotsc, i_d}} := \sum_{i_{k'+1},\dotsc,i_k} m_{\paren*{i_1, \dotsc, i_{k'}^{*(\mathbf{i}_{k'-1})}, i_{k'+1}, \dotsc, i_d}},
\]
yields the correct marginal masses $\tilde{m}_{\mathbf{i}_{k'}}$ in all coordinates $k' \in [1 \! : \! k]$ and comes at a Wasserstein cost of 
at most $k(n-1) \delta$, i.e., the Wasserstein distance between the configuration of masses before the moves and the configuration after the moves is at most $k(n-1) \delta$.
There is a slight complication when $m_{\paren*{i_1, \dotsc, i_{k'}^{*(\mathbf{i}_{k'-1})}}} = 0$ as in this case the fraction in (\ref{transportation-kth-dimension}) is technically undefined. However, analogously to the definition of the $n_{\mathbf{i}_k}$ in the case of zero-masses in the discussion preceding this theorem, we take 
\[
	\paren*{\frac{m_{\paren*{i_1, \dotsc, i_{k'}^{*(\mathbf{i}_{k'-1})}, i_{k+1}, \dotsc, i_d}}}{m_{\paren*{i_1, \dotsc, i_{k'}^{*(\mathbf{i}_{k'-1})}}}}} := \paren*{\frac{1}{n}}^{d-k}
\]
when $m_{\paren*{i_1, \dotsc, i_{k'}^{*(\mathbf{i}_{k'-1})}}} = 0$.
In either case, we have 
\[
	 \sum_{i_{k+1},\dotsc,i_d} \paren*{\frac{m_{\paren*{i_1, \dotsc, i_{k'}^{*(\mathbf{i}_{k'-1})}, i_{k+1}, \dotsc, i_d}}}{m_{\paren*{i_1, \dotsc, i_{k'}^{*(\mathbf{i}_{k'-1})}}}}} = 1.
\]

We note that the transport map in the Claim characterizes, at a high level, the state of the masses at an intermediate step in the transportation, while (\ref{eq:sum-lemma8}) describes the ``final state'' after all the moves have been completed in coordinate $k$.

If we accept the claim and apply it for $k=d$ in combination with Lemma \ref{lemma:move masses lemma}, it follows that the masses $m_{\mathbf{i}_d}$ are, indeed, redistributed to the masses $\tilde{m}_{\mathbf{i}_d}$. Moreover, we get that the total cost of the transportations in Step 2 effecting this redistribution is upper-bounded by
\[
	(n-1) \, \delta + 2 \, (n-1) \, \delta + \dotsm + d\, (n-1) \, \delta = \frac{d(d+1)}{2} (n-1) \, \delta.
\]

\noindent\textit{Step 3.} The Wasserstein cost associated with spreading out the masses $\tilde{m}_{\mathbf{i}_d}$ uniformly across their associated subcubes follows 
from (\ref{subdistance}) as
\[
	W \paren*{ \sum_{\mathbf{i}_d} \tilde{m}_{\mathbf{i}_d} \delta_{p_{\mathbf{i}_d}} , \mu } \leq \frac{\sqrt{d}}{n}.
\]

Using the fact that Wasserstein distance is a metric, we can put the costs incurred in the individual steps together according to
\begin{align*}
	W\paren*{\mu,\nu} & \leq \underbrace{W\paren*{\nu,\sum_{\mathbf{i}_d} m_{\mathbf{i}_d} \delta_{p_{\mathbf{i}_d}}}}_{\text{Step 1}} + \underbrace{W\paren*{\sum_{\mathbf{i}_d} m_{\mathbf{i}_d} \delta_{p_{\mathbf{i}_d}}, \sum_{\textbf{i}_{d}} \tilde{m}_{\textbf{i}_{d}} \delta_{p_{\textbf{i}_{d}}}}}_{\text{Step 2}} + \underbrace{W \paren*{ \sum_{\textbf{i}_d} \tilde{m}_{\textbf{i}_d} \delta_{p_{\textbf{i}_d}}, \mu}}_{\text{Step 3}} \\
	& \leq \frac{\sqrt{d}}{n} + \frac{d(d+1)}{2} \, (n-1) \, \delta + \frac{\sqrt{d}}{n} \\
	& = \frac{2\sqrt{d}}{n} + \frac{d(d+1)}{2} (n-1) \, \delta.
\end{align*}
It remains to prove the claim.\\

\noindent \textit{Proof of the Claim.} We proceed by induction on $k$ and start with the base case $k=1$. The statement on the transportation cost associated with (\ref{transportation-kth-dimension}) does not need an induction argument, rather it follows as a byproduct of the proof by induction. Evaluating the transport map for $k=1$ yields
\begin{align*}
	m_{\mathbf{i}_d} & +  \chi_{[0:(n-1)] \setminus \bracket*{i_{1}^{*(\mathbf{i}_{0})}}}(i_{1})  \paren*{\frac{m_{\paren*{i_{1}^{*(\mathbf{i}_{0})}, i_{2}, \dotsc, i_d}}}{m_{i_{1}^{*(\mathbf{i}_{0})}}}} (\tilde{m}_{i_1} - m_{i_1})
		 \\&-  \chi_{\bracket*{i_{1}^{*(\mathbf{i}_{0})}}}(i_{1}) \ \paren*{\frac{m_{\paren*{i_{1}^{*(\mathbf{i}_{0})}, i_{2}, \dotsc, i_d}}}{m_{i_{1}^{*(\mathbf{i}_{0})}}}}   (m_{i_1}-\tilde{m}_{i_1}).
\end{align*}
Since masses are moved in the first coordinate only and $\paren*{\tilde{m}_{i_1} - m_{i_1}} \leq \delta$, for $i_1 \neq i_1^{*(\mathbf{i}_0)}$, the Wasserstein cost of the overall transportation satisfies
\[
	\sum_{i_1 \neq i_1^{*(\mathbf{i}_0)}} \sum_{i_2, \dotsc, i_d} \paren*{\frac{m_{\paren*{i_1^{*(\mathbf{i}_0)}, i_2, \dotsc, i_d}}}{m_{i_1^{*(\mathbf{i}_0)}}}} \paren*{\tilde{m}_{i_1} - m_{i_1}} \leq (n-1) \,\delta.
\]
Furthermore, we obtain the desired marginal masses in $i_1$ as a consequence of
\[
	\sum_{i_2, \dotsc, i_d} \left( m_{\mathbf{i}_d} + \paren*{\frac{m_{\paren*{i_1^{*(\mathbf{i}_0)}, i_2, \dotsc, i_d}}}{m_{i_1^{*(\mathbf{i}_0)}}}} \paren*{\tilde{m}_{i_1} - m_{i_1}} \right)  = m_{i_1} + \paren*{\tilde{m}_{i_1} - m_{i_1}} = \tilde{m}_{i_1}.
\]
This completes the proof of the base case. 

We proceed to establish the induction step. Assume that transportations were conducted in coordinate $k$ according to (\ref{transportation-kth-dimension}) and that all marginal masses up to and including coordinate $k$ are as desired.
We consider the transport equation (\ref{transportation-kth-dimension}) in coordinate $k+1$, i.e., the sums in (\ref{transportation-kth-dimension}) range from $1$ to $k+1$ and
start by pointing out that
\[
	n_{(i_1, \dotsc, i_{k+1}^{*(\mathbf{i}_{k})}, \dotsc, i_d)} n_{(i_1, \dotsc, i_{k+1}^{*(\mathbf{i}_{k})} ,\dotsc, i_{d-1})} \dotsm \, n_{(i_1, \dotsc, i_{k+1}^{*(\mathbf{i}_{k})},i_{k+2})} = \paren*{\frac{m_{\paren*{i_1, \dotsc, i_{k}, i_{k+1}^{*(\mathbf{i}_{k})}, i_{k+2}, \dotsc, i_d}}}{m_{\paren*{i_1, \dotsc, i_{k}, i_{k+1}^{*(\mathbf{i}_{k})}}}}}.
\]
The first set of transportations (corresponding to the index $k'=k+1$ in the transport equation (\ref{transportation-kth-dimension}) evaluated for coordinate $k+1$) hence amounts to moving,
for fixed $i_1,\dotsc,i_{k}, i_{k+2}, \dotsc,i_d$, the mass
\[
	\paren*{\frac{m_{\paren*{i_1, \dotsc, i_{k}, i_{k+1}^{*(\mathbf{i}_{k})}, i_{k+2}, \dotsc, i_d}}}{m_{\paren*{i_1, \dotsc, i_{k}, i_{k+1}^{*(\mathbf{i}_{k})}}}}} \paren*{n_{{\mathbf{i}}_{k+1}} - \tilde{n}_{{\mathbf{i}}_{k+1}}} \Upsilon^{\eta}(\mathbf{i},k,1)
\]
out of the point $p_{\paren*{i_1, \dotsc,i_k, i_{k+1}^{*(\mathbf{i}_{k})}, i_{k+2}, \dotsc, i_d}}$
and redistributing it across the points
$p_{\paren*{i_1, \dotsc, i_k, i_{k+1}, i_{k+2}, \dotsc, i_d}}$, for $i_{k+1} \neq i_{k+1}^{*(\mathbf{i}_{k})}$. Note that
for $i_{k+1}=i_{k+1}^{*(\mathbf{i}_{k})}$, the quantity 
$\paren*{n_{{\mathbf{i}}_{k+1}} - \tilde{n}_{{\mathbf{i}}_{k+1}}}$ is positive by definition of $\tilde{n}$. These transportations are conducted for all possible tuples $i_1,\dotsc,i_{k}, i_{k+2}, \dotsc,i_d$. The Wasserstein cost associated with the collection of these transportations satisfies
\begin{align*}
	\sum_{i_1, \dotsc ,i_k} & \sum_{i_{k+1} \neq i_{k+1}^{*(\mathbf{i}_{k})}} \sum_{i_{k+2}, \dotsc, i_d} \paren*{\frac{m_{\paren*{i_1, \dotsc, i_{k}, i_{k+1}^{*(\mathbf{i}_{k})}, i_{k+2}, \dotsc, i_d}}}{m_{\paren*{i_1, \dotsc, i_{k}, i_{k+1}^{*(\mathbf{i}_{k})}}}}} \paren*{\tilde{n}_{{\mathbf{i}}_{k+1}} - n_{{\mathbf{i}}_{k+1}}} \Upsilon^{\eta}(\mathbf{i},k,1) \\
	& = \sum_{i_1, \dotsc ,i_k} \sum_{i_{k+1} \neq i_{k+1}^{*(\mathbf{i}_{k})}} \paren*{\tilde{n}_{{\mathbf{i}}_{k+1}} - n_{{\mathbf{i}}_{k+1}}} \Upsilon^{\eta}(\mathbf{i},k,1) \\
	& \leq (n-1) \, \delta \sum_{i_1, \dotsc ,i_k} \Upsilon^{\eta}(\mathbf{i},k,1) \\
	& \leq (n-1) \, \delta \sum_{i_1, \dotsc ,i_k} \Upsilon(\mathbf{i},k,1) \\
	& = (n-1) \, \delta,
\end{align*}
where the last inequality follows because $\eta_{\mathbf{i}_{k}} (i_{k}) = \tilde{n}_{\mathbf{i}_{k}}$ exactly when $i_k = i_k^{*(\mathbf{i}_{k-1})}$, in which case we have
$\tilde{n}_{\mathbf{i}_{k}} \leq n_{\mathbf{i}_{k}}$. The second set of transportations reconfigures the masses in the coordinates $k' \leq k$ in order to obtain correct marginals in the $(k+1)$-th coordinate. To this end, we first note that, for each $k' \leq k$, for all $i_{k'}$, the following identity holds
\[
	\paren*{\frac{m_{\paren*{i_1, \dotsc, i_{k'}^{*(\mathbf{i}_{k'-1})}, i_{k+1}, \dotsc , i_d}}} {m_{\paren*{i_1, \dotsc, i_{k'}^{*(\mathbf{i}_{k'-1})}}}}} = 
	\underbrace{\paren*{\frac{m_{\paren*{i_1, \dotsc, i_{k'}^{*(\mathbf{i}_{k'-1})}, i_{k+2}, \dotsc , i_d}}} {m_{\paren*{i_1, \dotsc, i_{k'}^{*(\mathbf{i}_{k'-1})}}}}}}_{\text{independent of } i_{k+1}}
	\underbrace{\paren*{\frac{m_{\paren*{i_1, \dotsc, i_{k'}^{*(\mathbf{i}_{k'-1})}, i_{k+1}, \dotsc , i_d}}} {m_{\paren*{i_1, \dotsc, i_{k'}^{*(\mathbf{i}_{k'-1})}, i_{k+2}, \dotsc , i_d}}}}}_{\text{sums to 1 in } i_{k+1}},
\]
as long as $m_{\paren*{i_1, \dotsc, i_{k'}^{*(\mathbf{i}_{k'-1})}, i_{k+2}, \dotsc , i_d}} \neq 0$. If $m_{\paren*{i_1, \dotsc, i_{k'}^{*(\mathbf{i}_{k'-1})}, i_{k+2}, \dotsc , i_d}} = 0$, then, as usual, we set 
\[
	\paren*{\frac{m_{\paren*{i_1, \dotsc, i_{k'}^{*(\mathbf{i}_{k'-1})}, i_{k+1}, \dotsc , i_d}}} {m_{\paren*{i_1, \dotsc, i_{k'}^{*(\mathbf{i}_{k'-1})}, i_{k+2}, \dotsc , i_d}}}} := \frac{1}{n}.
\]
Specifically, noting that by the induction assumption transportation according to (\ref{transportation-kth-dimension}) was carried out in coordinate $k$, we would like to reconfigure, for each $k' \le k$, the masses
\begin{align}
	\paren*{\frac{m_{\paren*{i_1, \dotsc, i_{k'}^{*(\mathbf{i}_{k'-1})}, i_{k+2}, \dotsc , i_d}}} {m_{\paren*{i_1, \dotsc, i_{k'}^{*(\mathbf{i}_{k'-1})}}}}}
	\paren*{\frac{m_{\paren*{i_1, \dotsc, i_{k'}^{*(\mathbf{i}_{k'-1})}, i_{k+1}, \dotsc , i_d}}} {m_{\paren*{i_1, \dotsc, i_{k'}^{*(\mathbf{i}_{k'-1})}, i_{k+2}, \dotsc , i_d}}}}
&\tilde{\Upsilon}(\mathbf{i},k,k'+1) \nonumber \\
	& \paren*{\tilde{n}_{\mathbf{i}_{k'}} - n_{\mathbf{i}_{k'}} } \Upsilon^{\eta}(\mathbf{i},k'-1,1) \label{source-mass}
\end{align}
into
\begin{equation}
	\paren*{\frac{m_{\paren*{i_1, \dotsc, i_{k'}^{*(\mathbf{i}_{k'-1})}, i_{k+2}, \dotsc , i_d}}} {m_{\paren*{i_1, \dotsc, i_{k'}^{*(\mathbf{i}_{k'-1})}}}}}
	\tilde{n}_{\mathbf{i}_{k+1}}
\tilde{\Upsilon}(\mathbf{i},k,k'+1) \paren*{\tilde{n}_{\mathbf{i}_{k'}} - n_{\mathbf{i}_{k'}} } \Upsilon^{\eta}(\mathbf{i},k'-1,1). \label{target-mass}
\end{equation}
This reconfiguration is possible as only one term in each (\ref{source-mass}) and (\ref{target-mass}) depends on $i_{k+1}$ and 
\[
	\sum_{i_{k+1}} \tilde{n}_{\mathbf{i}_{k+1}} = 1 = \sum_{i_{k+1}} \paren*{\frac{m_{\paren*{i_1, \dotsc, i_{k'}^{*(\mathbf{i}_{k'-1})}, i_{k+1}, \dotsc , i_d}}} {m_{\paren*{i_1, \dotsc, i_{k'}^{*(\mathbf{i}_{k'-1})}, i_{k+2}, \dotsc , i_d}}}}.
\]
Masses to be moved in this manner appear for
all $i_1,...,i_{k'-1},i_{k'},i_{k'+1}, \dotsc, i_{k}, i_{k+2}, \dotsc, i_d$ with $i_{k'} \neq i_{k'}^{*(\mathbf{i}_{k'-1})}$. It follows by inspection of the transport map (\ref{transportation-kth-dimension}) that these transportations do not alter the marginals for $k'\leq k$ as, for given $k'$, the mass moved out of the point $p_{\paren*{i_1, \dotsc,i_{k'-1}, i_{k'}^{*(\mathbf{i}_{k'-1})}, i_{k'+1}, \dotsc, i_d}}$, accounted for by the sum with negative sign in (\ref{transportation-kth-dimension}), equals the total mass moved into the points $p_{\paren*{i_1, \dotsc, i_{k'-1}, i_{k'}, i_{k'+1}, \dotsc, i_d}}$, for $i_{k'} \neq i_{k'}^{*(\mathbf{i}_{k'}-1)}$, accounted for by the sum with positive sign. These moves hence retain the marginals for $k' \leq k$, which are correct by the induction assumption. Before establishing that the desired marginals in coordinate $k+1$ are obtained, we compute the Wasserstein cost associated with the mass reconfiguration moves according to
\begin{align*}
	\sum_{i_1,...,i_{k'-1}} & \sum_{i_{k'} \neq i_{k'}^{*(\mathbf{i}_{k'-1})}} \sum_{i_{k'+1}, \dotsc, i_{k}} \sum_{i_{k+2}, \dotsc, i_d} \paren*{\frac{m_{\paren*{i_1, \dotsc, i_{k'}^{*(\mathbf{i}_{k'-1})}, i_{k+2}, \dotsc , i_d}}} {m_{\paren*{i_1, \dotsc, i_{k'}^{*(\mathbf{i}_{k'-1})}}}}}
\tilde{\Upsilon}(\mathbf{i},k,k'+1) \\
	& \phantom{========================} 
\paren*{\tilde{n}_{\mathbf{i}_{k'}} - n_{\mathbf{i}_{k'}} } \Upsilon^{\eta}(\mathbf{i},k'-1,1)\\
	& = \sum_{i_1,...,i_{k'-1}} \sum_{i_{k'} \neq i_{k'}^{*(\mathbf{i}_{k'-1})}} \sum_{i_{k'+1}, \dotsc, i_{k}} \tilde{\Upsilon}(\mathbf{i},k,k'+1) \paren*{\tilde{n}_{\mathbf{i}_{k'}} - n_{\mathbf{i}_{k'}} } \Upsilon^{\eta}(\mathbf{i},k'-1,1)\\
	& = \sum_{i_1,...,i_{k'-1}} \sum_{i_{k'} \neq i_{k'}^{*(\mathbf{i}_{k'-1})}} \underbrace{\paren*{\tilde{n}_{\mathbf{i}_{k'}} - n_{\mathbf{i}_{k'}}}}_{\leq \delta} \Upsilon^{\eta}(\mathbf{i},k'-1,1)\\
	& \leq (n-1) \delta \sum_{i_1,...,i_{k'-1}} \Upsilon^{\eta}(\mathbf{i},k'-1,1)\\
	& \leq (n-1) \delta \sum_{i_1,...,i_{k'-1}} \Upsilon(\mathbf{i},k'-1,1) \\
	& = (n-1) \delta.
\end{align*}
We must carry this out for all $k' \in [1\! : \!k]$, which results in a Wasserstein cost of $k (n-1) \delta$ for the reconfigurations. Altogether, we have a Wasserstein cost of
\[
	(k+1)(n-1) \delta
\]
incurred by the moves corresponding to coordinate $k+1$.

Next, using $ \tilde{n}_{\mathbf{i}_{k+1}} \tilde{\Upsilon}(\mathbf{i},k,k'+1) = \tilde{\Upsilon}(\mathbf{i},k+1,k'+1)$, it follows that the updated mass in the point $p_{\mathbf{i}_d}$ is given by
\begin{align*}
	m_{\mathbf{i}_d} & + \sum_{k'=1}^{k+1} \chi_{[0:(n-1)] \setminus \bracket*{i_{k'}^{*(\mathbf{i}_{k'-1})}}}(i_{k'}) \ \paren*{\frac{m_{\paren*{i_1, \dotsc, i_{k'}^{*(\mathbf{i}_{k'-1})}, i_{k+2}, \dotsc, i_d}}}{m_{\paren*{i_1, \dotsc, i_{k'}^{*(\mathbf{i}_{k'-1})}}}}}\\
	& \phantom{========} \tilde{\Upsilon}(\mathbf{i},k+1,k'+1) \paren*{\tilde{n}_{\mathbf{i}_{k'}} - n_{\mathbf{i}_{k'}}} \Upsilon^{\eta}(\mathbf{i},k'-1,1) \\
		& - \sum_{k'=1}^{k+1} \chi_{\bracket*{i_{k'}^{*(\mathbf{i}_{k'-1})}}}(i_{k'}) \ \Upsilon(\mathbf{i},d,k'+1) \paren*{n_{\mathbf{i}_{k'}} - \tilde{n}_{\mathbf{i}_{k'}} } \Upsilon^{\eta}(\mathbf{i},k'-1,1).
\end{align*}
Finally, we need to check that the marginals in the $(k+1)$-th coordinate are, indeed, given by $\tilde{m}_{\mathbf{i}_{k+1}}$.
This is accomplished by noting that owing to Lemma \ref{lemma:move masses lemma}, we have
\begin{align*}
\sum_{i_{k+2} ,\dotsc, i_d} &\left[ m_{\mathbf{i}_d} + \sum_{k'=1}^{k+1} \chi_{[0:(n-1)] \setminus \bracket*{i_{k'}^{*(\mathbf{i}_{k'-1})}}}(i_{k'}) \ \paren*{\frac{m_{\paren*{i_1, \dotsc, i_{k'}^{*(\mathbf{i}_{k'-1})}, i_{k+2}, \dotsc, i_d}}}{m_{\paren*{i_1, \dotsc, i_{k'}^{*(\mathbf{i}_{k'-1})}}}}} \right.\\
	& \phantom{============} \tilde{\Upsilon}(\mathbf{i},k+1,k'+1) \paren*{\tilde{n}_{\mathbf{i}_{k'}} - n_{\mathbf{i}_{k'}}} \Upsilon^{\eta}(\mathbf{i},k'-1,1) \\
	&\left. - \sum_{k'=1}^{k+1} \chi_{\bracket*{i_{k'}^{*(\mathbf{i}_{k'-1})}}}(i_{k'}) \ \Upsilon(\mathbf{i},d,k'+1) \paren*{n_{\mathbf{i}_{k'}} - \tilde{n}_{\mathbf{i}_{k'}} } \Upsilon^{\eta}(\mathbf{i},k'-1,1) \right] \\
	& = m_{\mathbf{i}_{k+1}} \\
	& + \sum_{k'=1}^{k+1} \chi_{[0:(n-1)] \setminus \bracket*{i_{k'}^{*(\mathbf{i}_{k'-1})}}}(i_{k'}) \ \tilde{\Upsilon}(\mathbf{i},k+1,k'+1) \paren*{\tilde{n}_{\mathbf{i}_{k'}} - n_{\mathbf{i}_{k'}}} \Upsilon^{\eta}(\mathbf{i},k'-1,1) \\
	& - \sum_{k'=1}^{k+1} \chi_{\bracket*{i_{k'}^{*(\mathbf{i}_{k'-1})}}}(i_{k'}) \ \Upsilon(\mathbf{i},k+1,k'+1) \paren*{n_{\mathbf{i}_{k'}} - \tilde{n}_{\mathbf{i}_{k'}} } \Upsilon^{\eta}(\mathbf{i},k'-1,1)\\
	& = \tilde{m}_{\mathbf{i}_{k+1}}.
\end{align*}
This concludes the proof of the induction step and hence establishes the claim.
\end{proof}

\section{Approximation of arbitrary distributions on bounded subsets of \texorpdfstring{$\R^d$}{Rd} with generative ReLU networks}\label{sec:approximation-arbitrary-sets}

In this section, we put all the pieces developed together and state the main result of this paper.

\begin{theorem} \label{thm:main result on [0,1]}
For every distribution $\nu$ supported on $[0,1]^d$, there exists a $\frac{\delta}{n}$-quantized ReLU network $\Phi \in \cNN_{1,d}$ with $\mathcal{M}(\Phi) = \mathcal{O}(n^{d} + sn^{d-1})$ and $\mathcal{L}(\Phi)=(s+3)d - s -1$ such that 
\begin{equation}
	W(\Phi \# U, \nu) \leq \frac{\sqrt{d}}{n2^s} + \frac{2\sqrt{d}}{n} + \frac{d(d+1)}{2} (n-1) \delta. \label{main-result-w}
\end{equation}
\end{theorem}
\begin{proof}
The proof follows by application of the triangle inequality for Wasserstein distance in combination with Theorems~\ref{main_histNN_theorem} and \ref{thm:quantization bound} according to
\begin{equation*}
W(\Phi \# U, \nu) \leq W(\Phi \# U, \mu) + W(\mu, \nu) \leq \frac{\sqrt{d}}{n2^s} + \frac{2\sqrt{d}}{n} + \frac{d(d+1)}{2} (n-1) \delta,
\end{equation*}
where $\mu$ denotes the $\delta$-quantized histogram distribution of resolution $n$ per Theorem~\ref{thm:quantization bound}.
\end{proof}

When the target distribution is uniform, Theorem~\ref{thm:main result on [0,1]} recovers \cite[Theorem 2.1]{Bailey2019}. We can simplify the bound (\ref{main-result-w}) by setting $\delta = \frac{1}{\lceil \sqrt{d} (d+1) n (n-1) \rceil}$ (which satisfies the requirement $\delta < \frac{1}{n(n-1)}$ guaranteeing that all quantized weights are positive) to obtain, for $n \geq 2$,
\begin{equation}
    \label{main_bound}
\begin{aligned}
	W(\Phi \# U, \nu) \leq \frac{\sqrt{d}}{n2^s} + \frac{2 \sqrt{d}}{n} + \frac{\sqrt{d}}{2n} = \frac{\sqrt{d}}{n2^s} + \frac{5 \sqrt{d}}{2n}.
\end{aligned}
\end{equation}
The error bound in (\ref{main_bound}) illustrates the main conceptual insight of this paper, namely that generating arbitrary $d$-dimensional distributions from a $1$-dimensional uniform distribution by pushforward through a deep ReLU network does not come at a cost---in terms of Wasserstein-distance error---relative to generating the target distribution from $d$ independent random variables. 
Specifically, if we let the depth $s$ of the generating network go to infinity, the first term in the rightmost expression of (\ref{main_bound}) will go to zero exponentially fast in $s$---thanks to the space-filling property of the transport map realized by the generating network---leaving us only with the second term, which reflects the error stemming from the histogram approximation of the distribution. Moreover, this second term is inversely proportional to the histogram resolution $n$ and can thus be made arbitrarily small by letting the histogram resolution $n$ approach infinity. The width of the corresponding generating network will grow according to $\mathcal{O}(n^d)$. 

Theorem \ref{thm:main result on [0,1]} applies to distributions supported on the unit cube $[0,1]^d$. The extension to distributions supported on bounded subsets of $\R^d$ is, however, fairly straightforward. Before stating this extension, we provide a lemma that will help us deal with the scaling and shifting of distributions.

\begin{lemma}\label{lemma:lipschitz}
Let $\mu,\nu$ be distributions on $\R^d$ and let $f: \mathbb{R}^d \rightarrow \mathbb{R}^d$ be a Lipschitz-continuous mapping with Lipschitz constant $\textnormal{Lip} (f) < \infty$.
Then,
\[
	W(f \# \mu, f \# \nu) \leq \textnormal{Lip}(f) \ W(\mu, \nu).
\]
\end{lemma}

\begin{proof}
Let $\pi$ be a coupling between $\mu$ and $\nu$ and let $g: \mathbb{R}^{2d} \rightarrow \mathbb{R}^{2d}; (\mathbf{y}_1, \mathbf{y}_2) \mapsto (f(\mathbf{y}_1), f(\mathbf{y}_2))$. Then $g \# \pi$ is a coupling between $\mu$ and $\nu$ and 
\begin{align*}
	W(f \# \mu, f \# \nu) & \leq \int_{\mathbb{R}^{2d}} \|\mathbf{y}_1 - \mathbf{y}_2\| \ d(g \# \pi)(\mathbf{y}_1,\mathbf{y}_2)\\
		& = \int_{\mathbb{R}^{2d}} \|f(\mathbf{y}_1) - f(\mathbf{y}_2)\| \ d\pi (\mathbf{y}_1,\mathbf{y}_2)\\
		& \le \text{Lip}(f) \int_{\mathbb{R}^{2d}} \|\mathbf{y}_1 - \mathbf{y}_2\| \ d \pi (\mathbf{y}_1,\mathbf{y}_2) =  \textnormal{Lip}(f) \ W(\mu, \nu).\hspace{1.15cm}\qedhere
\end{align*}
\end{proof}

We are now ready to state the extension announced above.

\begin{theorem} \label{thm:main result}
Let $\nu$ be a distribution on $\R^d$ supported on 
$S = \alpha \, [0,1]^d + \boldsymbol{\beta}$ for $\alpha >0 $ and $\boldsymbol{\beta} \in \R^d$. Let $g(\mathbf{x}) = \frac{1}{\alpha}(\mathbf{x} - \boldsymbol{\beta})$ and $\delta = \frac{1}{\lceil \sqrt{d} (d+1) n (n-1)\rceil}$. Then, there exists a $\frac{\delta}{n}$-quantized ReLU network $\Phi \in \cNN_{1,d}$ with $\mathcal{M}(\Phi) = \mathcal{O}(n^{d} + sn^{d-1})$ and $\mathcal{L}(\Phi)=(s+3)d - s -1$ such that
\[
	W(g^{-1} \# \Phi \# U, \nu) \leq \alpha \paren*{\frac{\sqrt{d}}{n2^s} + \frac{5 \sqrt{d}}{2n}}.
\]
\end{theorem}

\begin{proof}
We first note that $g^{-1}$ is Lipschitz with $\text{Lip}(g^{-1}) = \alpha$. The result then follows immediately from Lemma \ref{lemma:lipschitz} combined with (\ref{main_bound}) by taking
$\Phi \in \cNN_{1,d}$ to approximate the distribution $g\#\nu$ according to Theorem \ref{thm:main result on [0,1]}.
\end{proof}

We finally remark that $g^{-1}$ by virtue of being an affine map can easily be realized by a ReLU network.

\section{Complexity of generative networks}\label{sec:complexity}

In this section, we compare the complexity of ReLU networks generating a given class of probability distributions to fundamental
bounds on the complexity of encoding classes of probability distributions through discrete approximations, a process commonly referred to as quantization \cite{10.5555/555805}. Specifically, complexity
will be measured in terms of the number of bits needed to describe the generative networks and, respectively,
distributions.
We begin by reviewing a fundamental result on the approximation of (non-singular) distributions.

\begin{definition}[\cite{10.5555/555805}]
For $n \in \N$ and the non-singular distribution $\nu$ supported on $[0,1]^d$, we define the minimal $n$-term quantization error as
\[ V_{n}(\nu):= \inf \{ W(\nu,\mu) : |\supp(\mu)| \leq n\}. \]
\end{definition}
The quantity $V_{n}(\nu)$ characterizes the approximation error---in Wasserstein distance---incurred by the best discrete $n$-point approximation of $\nu$.

The next result, taken from \cite{10.5555/555805}, states that this approximation error exhibits the same asymptotics for all (non-singular) distributions satisfying a mild moment constraint.

\begin{theorem}[{\cite[Theorem 6.2]{10.5555/555805}}]
\label{fundamental_limits_distributions}
Let $\mathbf{X}$ be a random vector in $\R^d$ with $\mathbf{X} \sim \nu$, where $\nu$ is non-singular and supported on $[0,1]^d$, and $\mathbb{E}\|{\mathbf{X}}\|^{1+\delta} < \infty$ for some $\delta>0$, where $\|\cdot\|$ is
any norm on $\R^d$. Then, 
\[ \lim_{n \rightarrow \infty} n^{1/d} V_{n}(\nu) = C, \]
where $C>0$ is a constant depending on $d$ only. 
\end{theorem}
Theorem~\ref{fundamental_limits_distributions} allows us to conclude that the best-approximating discrete distribution must have at least $n = \mathcal{O}(\eps^{-d})$ points for $V_{n}(\nu) \leq \eps$ to hold. As Wasserstein distance is a metric, we hence have a covering argument which
says that the class of (non-singular) distributions $\nu$ supported on $[0,1]^d$ (and satisfying the moment constraint in Theorem~\ref{fundamental_limits_distributions}) 
has metric entropy lower-bounded by
$d \log(\eps^{-1})$ bits. Although this lower bound is very generous, we demonstrate next that it
is achieved for quantized histogram target distributions encoded by their generating ReLU networks.

\begin{lemma} \label{lem:encoding-histograms}
Consider the class of quantized histogram distributions $\tilde{\mathcal{E}}_\delta[0,1]^d_n$ and let $\epsilon \in (0,1/2)$. Then, there exists a set of $\frac{\delta}{n}$-quantized ReLU networks $\Phi(\epsilon,\cdot)$ of cardinality $2^{\ell(\epsilon)}$, where 
$\ell(\epsilon) \leq C \log(\eps^{-1})$, with $C$ a constant depending on $d,\delta,n$, such that
\[
  \sup_{\nu \in \tilde{\mathcal{E}}_\delta[0,1]^d_n}  W(\Phi(\epsilon,\nu) \#U, \nu )  \leq \epsilon.
\]
\end{lemma}

\begin{proof}
By Theorem \ref{main_histNN_theorem}, for every distribution $\nu \in \tilde{\mathcal{E}}_\delta[0,1]^d_n$, there exists a
$\frac{\delta}{n}$-quantized ReLU network $\Phi$, with 
$\mathcal{M}(\Phi) = \mathcal{O}(n^{d} + sn^{d-1})$ and $\mathcal{L}(\Phi) = (s+3)d-s-1$, such that
\begin{equation*}
	W(\Phi \# U, \nu) \leq \frac{\sqrt{d}}{n2^s} = : \epsilon.
\end{equation*}
Note that $d,n,\delta$ are fixed and $\epsilon$, as a function of $s$, can be made arbitrarily small by taking $s$ and hence network depth to be sufficiently large. In particular, network depth needs to scale according to $\mathcal{O}(\log(\epsilon^{-1}))$. The resulting network $\Phi$ will hence depend on $\nu$ and 
$\epsilon$, indicated by the notation $\Phi(\epsilon,\nu)$ used henceforth. Next, using $\frac{\delta}{n} \leq 1/2$, which is by assumption,
it follows from \cite[Proposition VI.7]{deep-it-2019} that the number of bits needed to encode $\Phi(\epsilon,\nu)$ 
in a uniquely decodable fashion satisfies
\begin{equation}
\ell(\epsilon) \leq C_0 \left(\mathcal{M}(\Phi(\epsilon,\nu))\log(\mathcal{M}(\Phi(\epsilon,\nu)))+1\right)\log(n/\delta) \leq C(d,\delta,n) \log(\epsilon^{-1}). \label{bound-on-l-eps}
\hspace{1.15cm}\qedhere
\end{equation}
\end{proof}

{\em Remark.} We note that the quantized networks considered in the present paper differ slightly from those in \cite{deep-it-2019} as here we employ two types of quantization, namely Type 1 and Type 2 (see Definition~\ref{quantized_network}), while in \cite{deep-it-2019} all weights are encoded using Type-1 quantization. This does, however, not have
an impact on the bound on $\ell(\epsilon)$ in (\ref{bound-on-l-eps}), in fact, only the constant $C_0$ changes relative to \cite[Proposition VI.7]{deep-it-2019}. More specifically, to encode the quantized weights in the generative networks considered here, we only need one additional bit per weight signifying whether the weight is quantized according to Type 1 or Type 2.

Lemma~\ref{lem:encoding-histograms} tells us that encoding (or quantizing in the sense of \cite{10.5555/555805}) the class of quantized histogram distributions by pushing forward a scalar uniform distribution through generative ReLU networks achieves the metric entropy limit of
$\mathcal{O}(\log(\eps^{-1}))$ as identified in Theorem~\ref{fundamental_limits_distributions}.
We hasten to add that a metric entropy scaling of $\mathcal{O}(\log(\epsilon^{-1}))$ for (quantized) histogram distributions of dimension $d$, resolution $n$, and quantization level $\delta$, all fixed, is what one would expect as we essentially have to encode polynomially (in $n$) 
many (quantized) real numbers.
For general (non-singular) distributions, which constitute a much richer class than (quantized) histogram distributions, we can establish an $\mathcal{O}(\eps^{-d})$ (up to a multiplicative $\log$-term) complexity scaling for their corresponding generative networks, formalized as follows.
\begin{lemma}
Consider the class of non-singular distributions supported on $[0,1]^d$, denoted by $\mathcal{F}([0,1]^d)$, and let $\epsilon \in (0,1/2)$. 
Then, there exists a set of quantized ReLU networks $\Phi(\epsilon,\cdot)$ of cardinality $2^{\ell(\epsilon)}$, where
$\ell(\epsilon) \leq C\eps^{-d} \log^2(\eps^{-1})$, with $C$ a constant depending on $d$, such that
\[
  \sup_{\nu \in \mathcal{F}([0,1]^d)}  W(\Phi(\epsilon,\nu) \#U, \nu )  \leq \epsilon.
\]
\end{lemma}
\begin{proof}
By Theorem \ref{thm:main result on [0,1]}, for every distribution $\nu$ supported on $[0,1]^d$, there exists a $\frac{\delta}{n}$-quantized ReLU network $\Phi$, with $\mathcal{M}(\Phi) = \mathcal{O}(n^{d} + sn^{d-1})$ and $\mathcal{L}(\Phi)=(s+3)d-s-1$, such that
\begin{equation*}
	W(\Phi \# U, \nu) \leq \frac{\sqrt{d}}{n2^s} + \frac{2\sqrt{d}}{n} + \frac{d(d+1)}{2} (n-1) \delta.
\end{equation*}
Setting $\delta = \frac{1}{n^2d^2}$ and $s=\log(n)$, we hence get
\[ 
W(\Phi \# U, \nu) \leq \frac{3\sqrt{d}+1}{n} = : \epsilon,
\]
for a $\frac{1}{n^3d^2}$-quantized network $\Phi(\epsilon,\nu)$, with $\mathcal{M}(\Phi (\epsilon,\nu)) = \mathcal{O}(n^{d})$. Application of
\cite[Proposition VI.7]{deep-it-2019} allows us to conclude that the number of bits needed to encode $\Phi(\epsilon,\nu)$ 
in a uniquely decodable fashion satisfies
$
\ell(\epsilon) \leq C_{0} \left(\mathcal{M}(\Phi)\log(\mathcal{M}(\Phi))+1\right)\log(n^3 d^2) \leq C(d) \epsilon^{-d}\log^{2}(\epsilon^{-1})$. We note that $C(d)$ scales very unfavorably in $d$, namely according to $d^{d/2}$. Finally, we remark that the application of
\cite[Proposition VI.7]{deep-it-2019} requires that $\frac{1}{n^3d^2} \leq 1/2$, which is satisfied if at least one of $n, d$ is strictly larger than $1$.
\end{proof}

\appendix
\section{Proof of Lemma \ref{lemma:move masses lemma}}

\begin{proof}
Note first that
\[
	\chi_{[0:(n-1)] \setminus \bracket*{i_{k'}^{*(\mathbf{i}_{k'-1})}}}(i_{k'}) = \begin{cases} 
		0, & \text{if } i_{k'} = i_{k'}^{*(\mathbf{i}_{k'-1})}\\
		1, & \text{if } i_{k'} \neq i_{k'}^{*(\mathbf{i}_{k'-1})}
	\end{cases}
\]
and
\[
	\chi_{\bracket*{i_{k'}^{*(\mathbf{i}_{k'-1})}}}(i_{k'}) = \begin{cases} 
		0, & \text{if } i_{k'} \neq i_{k'}^{*(\mathbf{i}_{k'-1})} \\
		1, & \text{if } i_{k'} = i_{k'}^{*(\mathbf{i}_{k'-1})}
	\end{cases},
\]
so for a given $i_{k'}$ only one of the two $\chi$-terms above is active. Terms with $i_{k'} \neq i_{k'}^{*(\mathbf{i}_{k'-1})}$ correspond to subcubes to which we add mass to get the quantized masses in the $k'$-th coordinate, while terms with $i_{k'} = i_{k'}^{*(\mathbf{i}_{k'-1})}$ correspond to the subcube from which we take this extra mass. Correspondingly, we refer to terms with $i_{k'} \neq i_{k'}^{*(\mathbf{i}_{k'-1})}$ as ``+ terms'', while we designate terms with $i_{k'} = i_{k'}^{*(\mathbf{i}_{k'-1})}$ as ``-- terms''. By construction, $\paren*{\tilde{n}_{\mathbf{i}_{k'}} - n_{\mathbf{i}_{k'}}} \geq 0$ for + terms, while $\paren*{n_{\mathbf{i}_{k'}} - \tilde{n}_{\mathbf{i}_{k'}}} \geq 0$ for -- terms. In evaluating the sum (\ref{eq:sum-lemma8}), we consider three different cases.\\

\noindent \textit{Case 1}: All terms are + terms. In this case, the sum becomes 
\begin{align*}
	&m_{\mathbf{i}_k} + \sum_{k'=1}^k \tilde{\Upsilon}(\mathbf{i},k,k'+1) \paren*{\tilde{n}_{\mathbf{i}_{k'}} - n_{\mathbf{i}_{k'}}} \Upsilon(\mathbf{i},k'-1,1) \\
	&= m_{\mathbf{i}_k} + \sum_{k'=1}^k \tilde{\Upsilon}(\mathbf{i},k,k') \Upsilon(\mathbf{i},k'-1,1) - \sum_{k'=1}^k \tilde{\Upsilon}(\mathbf{i},k,k'+1)\Upsilon(\mathbf{i},k',1) \\
	&= m_{\mathbf{i}_k} - \underbrace{\Upsilon(\mathbf{i},k,1)}_{=\,m_{\mathbf{i}_k}} + \underbrace{\tilde{\Upsilon}(\mathbf{i},k,1)}_{=\,\tilde{m}_{\mathbf{i}_k}} + \sum_{k'=2}^k \tilde{\Upsilon}(\mathbf{i},k,k') \Upsilon(\mathbf{i},k'-1,1) - \sum_{k'=1}^{k-1} \tilde{\Upsilon}(\mathbf{i},k,k'+1)\Upsilon(\mathbf{i},k',1) \\ 
	&= \tilde{m}_{\mathbf{i}_k} + \sum_{k'=1}^{k-1} \tilde{\Upsilon}(\mathbf{i},k,k'+1)\Upsilon(\mathbf{i},k',1) - \sum_{k'=1}^{k-1}\tilde{\Upsilon}(\mathbf{i},k,k'+1)\Upsilon(\mathbf{i},k',1) \\
	&= \tilde{m}_{\mathbf{i}_k}.
\end{align*}

\noindent \textit{Case 2:} All terms are -- terms. In this case, the sum is
\begin{align*}
	&m_{\mathbf{i}_k} - \sum_{k'=1}^k \Upsilon(\mathbf{i},k,k'+1) \paren*{n_{\mathbf{i}_{k'}} - \tilde{n}_{\mathbf{i}_{k'}} } \tilde{\Upsilon}(\mathbf{i},k'-1,1)\\
	&= m_{\mathbf{i}_k}  + \sum_{k'=1}^k \Upsilon(\mathbf{i},k,k'+1) \tilde{\Upsilon}(\mathbf{i},k',1) - \sum_{k'=1}^k \Upsilon(\mathbf{i},k,k') \tilde{\Upsilon}(\mathbf{i},k'-1,1) \\
	&= m_{\mathbf{i}_k}  - \Upsilon(\mathbf{i},k,1) + \tilde{\Upsilon}(\mathbf{i},k,1) + \sum_{k'=1}^{k-1} \Upsilon(\mathbf{i},k,k'+1) \tilde{\Upsilon}(\mathbf{i},k',1) - \sum_{k'=2}^k \Upsilon(\mathbf{i},k,k') \tilde{\Upsilon}(\mathbf{i},k'-1,1) \\
	&= \tilde{m}_{\mathbf{i}_k} .
\end{align*}

\noindent \textit{Case 3}: There is at least one + term and one -- term. Let the indices of the + terms be given by
\[
	\bracket*{k^+_1, \dotsc, k^+_{\ell_1}}
\]
and those of the -- terms by
\[
	\bracket*{k^-_1, \dotsc, k^-_{\ell_2}},
\]
with both sets arranged in increasing order and $\ell_1 + \ell_2 = k$. 

We first consider the sum of the -- terms given by
\begin{equation}
	\sum_{\ell = 1}^{\ell_2} \, \Upsilon(\mathbf{i},k,k^-_\ell+1) \paren*{n_{\mathbf{i}_{k^-_\ell}} - \tilde{n}_{\mathbf{i}_{k^-_\ell}} } \Upsilon^{\eta}(\mathbf{i},k^-_\ell-1,1) \label{sum-of-all-negative-terms}
\end{equation}
and establish a cancelation property of successive terms in this sum, leaving only the border terms to be considered. Indeed, take $\ell$ such that $1 < \ell < \ell_2$, with the corresponding term given by
\begin{equation}
	\Upsilon(\mathbf{i},k,k^-_\ell+1) \, n_{\mathbf{i}_{k^-_\ell}} \, \Upsilon^{\eta}(\mathbf{i},k^-_\ell-1,1)\,-\,\Upsilon(\mathbf{i},k,k^-_\ell+1) \, \tilde{n}_{\mathbf{i}_{k^-_\ell}} \, \Upsilon^{\eta}(\mathbf{i},k^-_\ell-1,1). \label{l-term}
\end{equation}
Next, note that the positive part of the term corresponding to the index $\ell+1$ is given by
\begin{align}
	&\Upsilon(\mathbf{i},k,k^-_{\ell+1}+1)\,  n_{\mathbf{i}_{k^-_{\ell+1}}} \Upsilon^{\eta}(\mathbf{i},k^-_{\ell+1}-1,1) \nonumber \\
	&= \Upsilon(\mathbf{i},k,k^-_{\ell+1}+1)\,  n_{\mathbf{i}_{k^-_{\ell+1}}} n_{\mathbf{i}_{k^-_{\ell+1}-1}}\, \dotsm \, n_{\mathbf{i}_{k^-_\ell+1}} \tilde{n}_{\mathbf{i}_{k^-_\ell}}\Upsilon^{\eta}(\mathbf{i},k^-_\ell-1,1), \label{l-cancellation}
\end{align}
since all indices that lie strictly between $k^-_{\ell}$ and $k^-_{\ell+1}$, if there are any, correspond to + terms. Comparing (\ref{l-cancellation}) with (\ref{l-term}) reveals that
the positive part of the term corresponding to $\ell+1$ cancels out the negative part of the term for $\ell$. Similarly, the negative part of the term corresponding to $\ell-1$, given by 
\begin{align*}
	-\Upsilon(\mathbf{i},k,k^-_{\ell-1}+1)\, & \tilde{n}_{\mathbf{i}_{k^-_{\ell-1}}} \Upsilon^{\eta}(\mathbf{i},k^-_{\ell-1}-1,1),
\end{align*}
cancels out the positive part of the term for index $\ell$, which is given by
\begin{align*}
	&\Upsilon(\mathbf{i},k,k^-_\ell+1) \, n_{\mathbf{i}_{k^-_\ell}} \Upsilon^{\eta}(\mathbf{i},k^-_\ell-1,1) \\
	&= \Upsilon(\mathbf{i},k, k^-_{\ell}+1) \, n_{\mathbf{i}_{k^-_{\ell}}} n_{\mathbf{i}_{k^-_{\ell}-1}} \dotsm \, n_{\mathbf{i}_{k^-_{\ell-1}+1}} \tilde{n}_{\mathbf{i}_{k^-_{\ell-1}}}\Upsilon^{\eta}(\mathbf{i},k^-_{\ell-1}-1,1).
\end{align*}
Hence, the only contributions remaining in the sum (\ref{sum-of-all-negative-terms}) over all the -- terms are the negative part of the term corresponding to the index $\ell_2$ and the positive part of the term for the index $1$, i.e.,
\begin{align*}
	&\sum_{\ell = 1}^{\ell_2} \,\, \Upsilon(\mathbf{i},k,k^-_\ell+1) \paren*{n_{\mathbf{i}_{k^-_\ell}} - \tilde{n}_{\mathbf{i}_{k^-_\ell}} } \Upsilon^{\eta}(\mathbf{i},k^-_\ell-1,1)\\
	& = \Upsilon(\mathbf{i},k, k^-_{1}+1) \, n_{\mathbf{i}_{k^-_1}} \, \Upsilon^{\eta}(\mathbf{i},k^-_1-1,1) \,-\, \Upsilon(\mathbf{i},k, k^-_{\ell_2}+1) \, \tilde{n}_{\mathbf{i}_{k^-_{\ell_2}}} \, \Upsilon^{\eta}(\mathbf{i},k^-_{\ell_2}-1,1) \\
	& = \Upsilon(\mathbf{i},k, k^-_{1}+1) \, n_{\mathbf{i}_{k^-_1}} \, \Upsilon(\mathbf{i},k^-_1-1,1) \,-\, \Upsilon(\mathbf{i},k, k^-_{\ell_2}+1) \, \tilde{n}_{\mathbf{i}_{k^-_{\ell_2}}} \, \Upsilon^{\eta}(\mathbf{i},k^-_{\ell_2}-1,1) \\	
	& = \ m_{\mathbf{i}_k} - \Upsilon(\mathbf{i},k, k^-_{\ell_2}+1) \, \tilde{n}_{\mathbf{i}_{k^-_{\ell_2}}} \, \Upsilon^{\eta}(\mathbf{i},k^-_{\ell_2}-1,1),
\end{align*}
since indices smaller than $k_1^-$ necessarily correspond to + terms. 

We proceed to the sum over the + terms. A similar cancelation property between consecutive terms in the sum can be established so that we are again left with contributions from the first and the last term only. Indeed, take $\ell$ such that $1 < \ell < \ell_1$, with the corresponding term given by
\begin{equation*}
	\tilde{\Upsilon}(\mathbf{i},k,k^+_\ell+1)\, \tilde{n}_{\mathbf{i}_{k^+_\ell}} \Upsilon^{\eta}(\mathbf{i},k^+_\ell-1,1) \,
	- \, \tilde{\Upsilon}(\mathbf{i},k,k^+_\ell+1)\, n_{\mathbf{i}_{k^+_\ell}} \Upsilon^{\eta}(\mathbf{i},k^+_\ell-1,1).
\end{equation*}
The positive part of the term corresponding to the index $\ell+1$ is given by
\begin{align*}
	&\tilde{\Upsilon}(\mathbf{i},k,k^+_{\ell+1}+1) \,  \tilde{n}_{\mathbf{i}_{k^+_{\ell+1}}} \Upsilon^{\eta}(\mathbf{i},k^+_{\ell+1}-1,1) \\
	&= \tilde{\Upsilon}(\mathbf{i},k,k^+_{\ell+1}+1) \, \tilde{n}_{\mathbf{i}_{k^+_{\ell+1}}} \tilde{n}_{\mathbf{i}_{k^+_{\ell+1}-1}} \dotsm \, \tilde{n}_{\mathbf{i}_{k^+_\ell+1}} n_{\mathbf{i}_{k^+_\ell}} \Upsilon^{\eta}(\mathbf{i},k^+_\ell-1,1),
\end{align*}
since all indices that lie strictly between $k^+_\ell$ and $k^+_{\ell+1}$, if there are any, correspond to -- terms. We can hence conclude, as above, that the positive part of the term corresponding to $\ell+1$ cancels out the negative part of the term for $\ell$. By the same argument, the positive part of the term for $\ell$ is cancelled out by the negative part of the term corresponding to $\ell-1$. Overall, the only remaining contributions are the negative part of the term corresponding to $\ell_1$ and the positive part of the term for the index $1$, i.e.,
\begin{align*}
	&\sum_{\ell = 1}^{\ell_1}\,\, \tilde{\Upsilon}(\mathbf{i},k,k^+_\ell+1) \paren*{\tilde{n}_{\mathbf{i}_{k^+_\ell}} - n_{\mathbf{i}_{k^+_\ell}}} \Upsilon^{\eta}(\mathbf{i},k^+_\ell-1,1) \\
	& = \tilde{\Upsilon}(\mathbf{i},k,1) -\tilde{\Upsilon}(\mathbf{i},k,k^+_{\ell_1}+1)\, n_{\mathbf{i}_{k^+_{\ell_1}}} \Upsilon^{\eta}(\mathbf{i},k^+_{\ell_1}-1,1) \\
	& = \tilde{m}_{\mathbf{i}_{k}} -\tilde{\Upsilon}(\mathbf{i},k,k^+_{\ell_1}+1) \, n_{\mathbf{i}_{k^+_{\ell_1}}} \Upsilon^{\eta}(\mathbf{i},k^+_{\ell_1}-1,1).
\end{align*}

Putting pieces together, (\ref{eq:sum-lemma8}) reduces to
\begin{align}
	& \ m_{\mathbf{i}_k} + \tilde{m}_{\mathbf{i}_{k}} - m_{\mathbf{i}_k} \nonumber \\
	& + \Upsilon(\mathbf{i},k, k^-_{\ell_2}+1) \, \tilde{n}_{\mathbf{i}_{k^-_{\ell_2}}} \Upsilon^{\eta}(\mathbf{i},k^-_{\ell_2}-1,1) \label{eq:overall-sum-lemma8-proof} \\
	& -\tilde{\Upsilon}(\mathbf{i},k,k^+_{\ell_1}+1) \, n_{\mathbf{i}_{k^+_{\ell_1}}} \Upsilon^{\eta}(\mathbf{i},k^+_{\ell_1}-1,1). \nonumber
\end{align}
There are two possibilities to consider now, either $k_{\ell_1}^+ = k$ or $k_{\ell_2}^- = k$. If $k_{\ell_1}^+ = k$, then (\ref{eq:overall-sum-lemma8-proof}) becomes
\begin{align*}
	& \tilde{m}_{\mathbf{i}_{k}} \\
	& + \Upsilon(\mathbf{i},k^+_{\ell_1},k^-_{\ell_2}+1) \, \tilde{n}_{\mathbf{i}_{k^-_{\ell_2}}} \Upsilon^{\eta}(\mathbf{i},k^-_{\ell_2}-1,1) \\
	& - n_{\mathbf{i}_{k^+_{\ell_1}}} \dotsm \, n_{\mathbf{i}_{k^-_{\ell_2}+1}} \, \tilde{n}_{\mathbf{i}_{k^-_{\ell_2}}} \Upsilon^{\eta}(\mathbf{i},k^-_{\ell_2}-1,1)\\
	& = \tilde{m}_{\mathbf{i}_{k}},
\end{align*}
since, by definition, ${k^-_{\ell_2}}$ is the largest index corresponding to -- terms. On the other hand, if $k_{\ell_2}^- = k$, then (\ref{eq:overall-sum-lemma8-proof}) reduces to
\begin{align*}
	& \tilde{m}_{\mathbf{i}_{k}} \\
	& + \tilde{n}_{\mathbf{i}_{k^-_{\ell_2}}} \dotsm \, \tilde{n}_{\mathbf{i}_{k^+_{\ell_1}+1}} n_{\mathbf{i}_{k^+_{\ell_1}}} \Upsilon^{\eta}(\mathbf{i},k^+_{\ell_1}-1,1) \\
	& - \tilde{\Upsilon}(\mathbf{i},k^-_{\ell_2},k^+_{\ell_1}+1) \, n_{\mathbf{i}_{k^+_{\ell_1}}} \Upsilon^{\eta}(\mathbf{i},k^+_{\ell_1}-1,1)\\
	& = \tilde{m}_{\mathbf{i}_{k}},
\end{align*}
since, again by definition, $k^+_{\ell_1}$ is the largest index corresponding to + terms. This concludes the proof.
\end{proof}

\bibliographystyle{spmpsci}
\bibliography{references}
\end{document}